\documentclass[twoside,11pt]{article}
\usepackage{blindtext}
\usepackage{bm}
\usepackage{graphicx}
\usepackage{tabularx}
\usepackage{mathrsfs}%
\usepackage{enumitem}
\usepackage{cite}
\usepackage{algorithm}%
\usepackage{algorithmicx}%
\usepackage{algpseudocode}
\usepackage[colorlinks=true,linkcolor=blue]{hyperref}
\usepackage{float}
\usepackage{makecell}
\usepackage{amsmath,amsthm, amssymb, graphicx, caption, subcaption}  

\usepackage{amsfonts}
\usepackage{booktabs}  
\usepackage{geometry}
\usepackage{hyperref}
\usepackage[nameinlink,capitalise,noabbrev]{cleveref}
\newtheorem{definition}{Definition}
\newtheorem{theorem}{Theorem}[section]
\newtheorem{lemma}{Lemma}[section]
\newtheorem{proposition}{Proposition}[section]
\newtheorem{corollary}{Corollary}[section]
\numberwithin{definition}{section}
\numberwithin{equation}{section}
\numberwithin{proposition}{section}

\usepackage{jmlr2e}

\usepackage{lastpage}
\jmlrheading{}{}{1-\pageref{LastPage}}{; Revised }{}{21-0000}{Author One and Author Two}

\ShortHeadings{Dynamics and Implicit Bias of Quadratic Networks}{One}
\firstpageno{1}

\begin{document}

\title{Quotient Dynamics, Effective Curvature, and Implicit Bias in Positive Quadratic Networks}

\author{
  \name Pengcheng Cheng\thanks{Corresponding author.} \email{3192618014@qq.com} \\
  \addr School of Mathematics, Jilin University \\
        Changchun, 130012, China
}

\editor{My editor}

\maketitle
\begin{abstract}
Positive quadratic networks admit the exact low-rank representation
\[
f_U(x)=\|U^\top x\|_2^2=x^\top UU^\top x,
\]
where the factor $U\in\mathbb{R}^{d\times r}$ is identifiable only up to right multiplication by an orthogonal matrix, while the represented predictor is the rank-$r$ positive semidefinite matrix $Q=UU^\top$. We study how this quotient structure governs factor-training dynamics, intrinsic curvature, recovery, and interpolation bias.
On the full-column-rank stratum, we characterize the fibers of $U\mapsto UU^\top$ and identify $\mathbb{R}^{d\times r}_*/O(r)$ with the manifold of rank-$r$ positive semidefinite matrices. For every smooth objective $L(U)=\ell(UU^\top)$, the Euclidean factor gradient is horizontal. Hence factor gradient flow projects exactly to quotient Riemannian gradient flow, while finite-step gradient descent induces an exact congruence recursion for the predictor.
For quadratic regression, we derive the effective Hessian at an interpolating predictor as the empirical measurement Gram form restricted to the tangent space and evaluated relative to the quotient metric. Under Gaussian rank-one measurements, we compute the population curvature, prove a uniform full-space deviation bound for the empirical normal operator, construct a moment-based spectral initializer, and establish local exponential convergence of gradient flow and local linear convergence of sufficiently small-step gradient descent. The resulting recovery guarantee is explicit but conservative because it relies on full-space second-moment control.
In an underdetermined commuting regime, factor gradient flow becomes an exact entropy mirror flow in joint spectral coordinates. Every strictly positive initialization converges to the corresponding Bregman projection onto the interpolation set. Under isotropic initialization $q(0)=\varepsilon^2\mathbf{1}$, the selected predictors approach the minimum-trace solution set as $\varepsilon\downarrow0$, with nonuniqueness resolved by a weighted entropy criterion within the invariant joint spectral algebra. For fixed positive initialization, finite-step gradient descent selects an interpolant differing from the continuous-time Bregman projection by $O(\eta)$.
Numerical experiments verify the quotient identities, effective-curvature prediction, recovery behavior, and Bregman, minimum-trace, entropy, and finite-step selection laws.
\end{abstract}


\section{Introduction}
\label{sec:introduction}

Overparameterized models are often optimized in a parameter space that
is substantially larger than the space of predictors they represent.
This distinction is especially transparent in factorized models.  For
a positive quadratic two-layer network,
\begin{equation}
f_U(x)
=
\sum_{j=1}^{r}(u_j^\top x)^2
=
x^\top UU^\top x,
\qquad
U=[u_1,\ldots,u_r]\in\mathbb{R}^{d\times r},
\label{eq:intro-positive-quadratic-network}
\end{equation}
the realized predictor is determined by the positive semidefinite
matrix
\begin{equation}
Q
=
UU^\top
\in
\mathbb{S}_+^d.
\label{eq:intro-predictor-matrix}
\end{equation}
The factor $U$, however, is not identifiable: for every
$R\in O(r)$,
\[
UR(UR)^\top
=
UU^\top.
\]
When $U$ has full column rank and $r\ge2$, the orthogonal orbit
\[
\{UR:R\in O(r)\}
\]
has positive dimension.  At a critical point of an
$O(r)$-invariant factor loss, its tangent directions are
representational directions and lie in the nullspace of the parameter
Hessian.  Thus, even when the predictor is locally identifiable from
the observations, the ambient parameter Hessian cannot distinguish
predictor motion from infinitesimal gauge motion.  The case $r=1$
is different: $O(1)=\{-1,1\}$ is discrete and therefore creates no
nonzero infinitesimal gauge direction.  In either case, an analysis
conducted entirely in $\mathbb{R}^{d\times r}$ must distinguish
changes in the represented predictor from redundancy of the chosen
factorization.

The geometric structure underlying this distinction is classical.
Fixed-rank positive semidefinite matrices can be represented as a
quotient of full-column-rank factors by the orthogonal group.  This
viewpoint has been used extensively in Riemannian optimization,
low-rank approximation, and matrix estimation
\cite{absil2008optimization,meyer2011regression,vandereycken2013lowrank,mishra2014fixedrank,massart2020quotient,boumal2023introduction}.
Recent work has further emphasized that smooth parametrizations may
alter optimization landscapes in ways determined by the
parametrization itself, and that singular or rank-deficient boundaries
require separate geometric treatment \cite{levin2023stationary,
levin2025parametrizations}.
These results provide a mature language for quotient manifolds and
factorized optimization.  They do not, by themselves, determine the
trajectory produced when $U$ is trained as an ordinary Euclidean
neural-network parameter.  In particular, they leave open several
algorithmic questions:

\begin{quote}
Does the Euclidean factor trajectory project to a geometrically
natural predictor trajectory?  Which curvature governs its local
observable rate after gauge directions have been removed?  When the
training equations admit multiple positive semidefinite interpolants,
which interpolant is selected by the factor dynamics?
\end{quote}

These questions also connect positive quadratic networks to low-rank
matrix estimation.  Learning a model of the form
\cref{eq:intro-positive-quadratic-network} from samples
$(x_i,y_i)$ amounts to recovering $Q$ from rank-one measurements
\[
y_i
=
x_i^\top Qx_i
=
\langle x_ix_i^\top,Q\rangle.
\]
Quadratic networks have therefore served as tractable models for
studying overparameterized optimization, benign nonconvexity, spectral
initialization, and algorithmic regularization
\cite{du2018power, soltani2018polynomial, li2018algorithmic, li2019rankone}.
More broadly, low-rank matrix factorization has led to recovery,
landscape, and convergence guarantees for matrix sensing and related
problems \cite{burer2003nonlinear, ge2017nospurious, tong2021scaledgd, jin2023incremental}.
Much of that literature focuses on global recovery, absence of
spurious local minima, or acceleration under poor conditioning.  Our
focus is complementary.  We ask how the exact quotient structure of
the map
\[
U\longmapsto UU^\top
\]
organizes the dynamics of the ordinary Euclidean factor
gradient method, and how the intrinsic curvature induced by this
quotient is reflected in local convergence.

A second motivation comes from implicit regularization.  In an
underdetermined problem, minimizing the training loss does not specify
a unique predictor, yet gradient-based algorithms often converge to
structured interpolants.  Matrix factorization is a central test bed
for this phenomenon.  Small initialization has been associated with
low-nuclear-norm or low-rank behavior
\cite{gunasekar2017implicit, li2018algorithmic, arora2019deep, razin2020norms,
jin2023incremental},
while analyses of homogeneous and diagonal networks show that the
selected solution may depend strongly on the parametrization,
initialization scale, initialization shape, and step size
\cite{chizat2020implicit, woodworth2020kernel, azulay2021initialization, even2023diagonal}.
A recurring mechanism is that gradient flow in a reparametrized model
can become mirror flow in predictor coordinates.  This connection has
been formulated abstractly for commuting parametrizations and used in
matrix-sensing analyses based on Bregman divergences
\cite{wu2021mirror, li2022reparametrized}.
An explicit selection theorem, however, requires more than asserting
that some complexity measure is implicitly minimized.  One must derive
the predictor dynamics from the actual trained parametrization, prove
convergence of the trajectory, identify the finite-initialization
variational problem, and then analyze both the small-initialization
limit and the discrepancy caused by a nonzero step size.

This paper carries out that program for positive quadratic networks.
We deliberately restrict attention to the model
\cref{eq:intro-positive-quadratic-network}.  This restriction makes it
possible to close the following chain:
\begin{equation}
\begin{aligned}
\text{exact quotient}
&\Longrightarrow
\text{actual factor dynamics}
\\
&\Longrightarrow
\text{effective quotient curvature}
\\
&\Longrightarrow
\text{local training rates and recovery}
\\
&\Longrightarrow
\text{explicit interpolation bias}.
\end{aligned}
\label{eq:intro-main-chain}
\end{equation}

The purpose is not to present quotient geometry itself as a new
formalism.  Rather, we use the exact quotient representation to derive
consequences for the training algorithm: an intrinsic Hessian that
removes the orthogonal gauge and describes the local linearized
predictor dynamics, an explicit but conservative recovery guarantee
under Gaussian rank-one measurements, and a trace-and-entropy
selection law in an underdetermined commuting regime.

Let
\[
\mathbb{R}^{d\times r}_*
:=
\left\{
U\in\mathbb{R}^{d\times r}:
\operatorname{rank}(U)=r
\right\}
\]
and
\[
\mathcal{M}_r^+
:=
\left\{
Q\in\mathbb{S}_+^d:
\operatorname{rank}(Q)=r
\right\}.
\]
We first prove directly that, for
$U,V\in\mathbb{R}^{d\times r}_*$,
\begin{equation}
UU^\top=VV^\top
\quad\Longleftrightarrow\quad
V=UR
\quad
\text{for some }R\in O(r).
\label{eq:intro-exact-fiber}
\end{equation}
Consequently,
\[
\mathbb{R}^{d\times r}_*/O(r)
\simeq
\mathcal{M}_r^+.
\]
All smooth quotient statements are made on this full-rank stratum.
Rank-deficient matrices belong to its singular boundary and are not
silently included in the manifold theory.

For an objective
\[
L(U)
=
\ell(UU^\top),
\qquad
G(Q)
:=
\nabla_Q\ell(Q),
\]
the Euclidean factor gradient is
\[
\nabla_U L(U)
=
2G(Q)U.
\]
Because $G(Q)$ is symmetric,
\[
U^\top\nabla_U L(U)
=
2U^\top G(Q)U
\]
is symmetric.  The factor gradient is therefore horizontal for the
quotient metric induced by the Frobenius metric.  Hence the actual
Euclidean parameter gradient flow contains no infinitesimal
orthogonal-gauge component and projects exactly to
\begin{equation}
\dot Q
=
-2\bigl(G(Q)Q+QG(Q)\bigr).
\label{eq:intro-predictor-flow}
\end{equation}
This is the Riemannian gradient flow of $\ell$ on
$\mathcal{M}_r^+$ for the induced quotient metric.

For a nonzero step size, ordinary factor gradient descent induces the
exact congruence recursion
\begin{equation}
Q_{k+1}
=
\bigl(I_d-2\eta G_k\bigr)
Q_k
\bigl(I_d-2\eta G_k\bigr),
\label{eq:intro-discrete-congruence}
\end{equation}
where $G_k=G(Q_k)$.  Equivalently,
\[
Q_{k+1}
=
Q_k
-
2\eta(G_kQ_k+Q_kG_k)
+
4\eta^2G_kQ_kG_k.
\]
Thus the discrete factor algorithm is not identified with an
unspecified Riemannian retraction step.  Its predictor recursion is
exact, and its first-order deviation from the continuous quotient flow
is explicit.

For the empirical square loss
\begin{equation}
\ell_n(Q)
=
\frac{1}{2n}
\left\|
\mathcal{A}_X(Q)-y
\right\|_2^2,
\qquad
[\mathcal{A}_X(Q)]_i
=
x_i^\top Qx_i,
\label{eq:intro-empirical-loss}
\end{equation}
we derive the quotient Hessian at an interpolating predictor
$Q_\star=U_\star U_\star^\top$.  If
\[
H
=
\Delta U_\star^\top
+
U_\star\Delta^\top
\in
T_{Q_\star}\mathcal{M}_r^+,
\]
where $\Delta$ is its horizontal lift, then
\begin{equation}
\operatorname{Hess}_g
\ell_n(Q_\star)[H,H]
=
\frac{1}{n}
\left\|\mathcal{A}_X(H)\right\|_2^2.
\label{eq:intro-effective-hessian}
\end{equation}
The associated effective eigenvalues are generalized eigenvalues
relative to the quotient metric $g_{Q_\star}$, rather than ordinary
eigenvalues defined using the ambient Frobenius norm on
$\mathbb{S}^d$.  Their spectrum describes the linearized quotient
dynamics after representational gauge directions have been removed.
In particular, the smallest effective eigenvalue determines the
slowest intrinsic mode of the linearized flow.

Formula \cref{eq:intro-effective-hessian} also separates two distinct
sources of degeneracy.  The orthogonal factor gauge has already been
removed by passing to the quotient, whereas genuine sample-level
non-identifiability remains through
\begin{equation}
T_{Q_\star}\mathcal{M}_r^+
\cap
\ker\mathcal{A}_X.
\label{eq:intro-effective-kernel}
\end{equation}

For Gaussian inputs, we compute the population normal operator
explicitly:
\[
\mathcal{T}(H)
=
\mathbb{E}
\bigl[
(x^\top Hx)xx^\top
\bigr]
=
2H+\operatorname{tr}(H)I_d.
\]
We then prove a uniform deviation bound for the empirical normal
operator on the entire ambient space $\mathbb{S}^d$.  This
full-space estimate yields explicit lower and upper bounds on the
effective quotient curvature and supplies a deterministic coercivity
condition for the nonlinear factor trajectory.

The same Gaussian moment identity gives the data-dependent matrix
\[
M_n
=
\frac{1}{2n}
\sum_{i=1}^{n}
y_i(x_ix_i^\top-I_d),
\qquad
\mathbb{E}M_n
=
Q_\star.
\]
Truncating $M_n$ to its $r$ largest positive eigenvalues produces
a spectral initializer.  Under the full-space deviation event, this
initializer enters an explicit Procrustes neighborhood in which a
secant regularity inequality holds.  We consequently obtain local
exponential convergence of ordinary Euclidean factor gradient flow
and local linear convergence of sufficiently small-step factor
gradient descent.

The gradient-descent step-size condition depends on the unknown
extreme eigenvalues of $Q_\star$ and is therefore an oracle
condition.  Moreover, the statistical bound is deliberately
conservative.  It controls the empirical operator on all of
$\mathbb{S}^d$ by a second-moment argument, rather than establishing
a restricted concentration inequality on a fixed tangent space or a
low-rank secant set.  The resulting sufficient sample requirement has
leading-order scale
\[
n
=
O
\left(
\frac{r^2d^8}{\delta}
\left(
\frac{\|Q_\star\|_F}
{\lambda_r(Q_\star)}
\right)^2
\right).
\]
We do not claim that this dependence is statistically optimal or that
it scales with the intrinsic dimension
\[
\dim\mathcal{M}_r^+
=
dr-\frac{r(r-1)}2.
\]
The theorem should instead be read as an explicit end-to-end
sufficient guarantee connecting the Gaussian measurement operator,
spectral initialization, quotient curvature, and the actual factor
training trajectory.

We finally study an underdetermined regime in which all measurement
matrices preserve a common orthogonal decomposition.  Let
$P_1,\ldots,P_m$ denote the joint spectral projectors and let
\[
d_a
:=
\operatorname{rank}(P_a).
\]
Predictors in the invariant algebra have the form
\[
Q(q)
=
\sum_{a=1}^{m}q_aP_a,
\qquad
q_a>0.
\]
Within this invariant algebra, the factor gradient flow reduces
exactly to
\begin{equation}
\dot q_a
=
-
\frac{4q_a}{d_a}
\left[
\frac{1}{n}
B^\top(Bq-y)
\right]_a.
\label{eq:intro-reduced-flow}
\end{equation}
Define
\begin{equation}
h(q)
=
\frac{1}{4}
\sum_{a=1}^{m}
d_a(q_a\log q_a-q_a).
\label{eq:intro-entropy-potential}
\end{equation}
Then
\[
\nabla^2 h(q)
=
\frac{1}{4}
\operatorname{diag}
\left(
\frac{d_1}{q_1},
\ldots,
\frac{d_m}{q_m}
\right),
\]
and \cref{eq:intro-reduced-flow} is the corresponding mirror flow.

For every fixed strictly positive initialization $q(0)$, we prove
global existence and convergence to the Bregman projection
\begin{equation}
q_\infty
=
\arg\min_{\substack{q\ge0,\,Bq=y}}
D_h(q,q(0)).
\label{eq:intro-bregman-selection}
\end{equation}
Thus the selected interpolant is derived from the trained
parametrization itself.  Convergence is not assumed in advance, and
the variational characterization is obtained from the limiting
dual relation of the actual flow.

For isotropic initialization
\[
q(0)
=
\varepsilon^2\mathbf{1},
\]
the Bregman objective separates into a dominant weighted-trace term
and a lower-order entropy term.  We show that
\[
\operatorname{dist}
\left(
q_\varepsilon,
\arg\min_{\substack{q\ge0,\,Bq=y}}
\sum_{a=1}^{m}d_aq_a
\right)
\longrightarrow0
\qquad
\text{as }\varepsilon\downarrow0.
\]
Since
\[
\operatorname{tr}(Q(q))
=
\sum_{a=1}^{m}d_aq_a,
\]
this yields a minimum-trace bias.  Under the commuting measurement
assumptions, the minimum value agrees with the minimum trace over the
full positive semidefinite interpolation set.  The subsequent entropy
tie-breaking statement is more restricted: when the minimum-trace
solution is nonunique, the small-initialization limit selects the
unique minimizer of
\[
\sum_{a=1}^{m}d_aq_a\log q_a
\]
on the minimum-trace face within the invariant joint spectral
algebra.  We do not claim an entropy characterization over arbitrary
noncommuting positive semidefinite interpolants.

We also derive an explicit
\[
O
\left(
\frac{1}{\log(1/\varepsilon^2)}
\right)
\]
upper bound for the trace gap.  For finite-step factor gradient
descent and fixed $\varepsilon>0$, we prove that the selected
discrete interpolant differs from the continuous Bregman projection
by $O(\eta)$.  More precisely, all constants in the finite-step
comparison are uniform in $\eta$ for fixed $\varepsilon$, but may
depend on the initialization scale.  The order of limits is therefore
essential: the discrete-to-continuous result first sends
$\eta\downarrow0$ with $\varepsilon$ fixed, and only afterward
considers the limit $\varepsilon\downarrow0$.

The main contributions are as follows. Exact quotient and projected training dynamics.
We characterize the complete fibers of
$U\mapsto UU^\top$ on the full-column-rank stratum and construct
the quotient metric induced by the Euclidean factor metric.  We
prove that ordinary Euclidean factor gradient flow is horizontal
and projects exactly to the quotient Riemannian gradient flow.  We
also derive the exact congruence recursion induced by finite-step
factor gradient descent.
An effective curvature tied to the local factor
dynamics.
At an interpolating predictor, the quotient Hessian is the
measurement Gram form restricted to
$T_{Q_\star}\mathcal{M}_r^+$ and evaluated relative to the
quotient metric.  This removes the orthogonal gauge while retaining
genuine predictor non-identifiability.  Its generalized spectrum
describes the linearized quotient modes, and the accompanying
neighborhood-level secant inequality yields local rates for the
actual Euclidean factor trajectory. An explicit but conservative Gaussian recovery
guarantee.
For Gaussian rank-one measurements, we compute the population
curvature, prove a full-space empirical-operator deviation bound,
construct a moment-based spectral initializer, and establish
local recovery by factor gradient flow and oracle-step-size
gradient descent.  The result is an end-to-end sufficient theorem,
not an intrinsic-dimension or statistically optimal recovery
guarantee. An explicit interpolation principle.
In the commuting underdetermined regime, the factor flow becomes an
exact entropy mirror flow.  Its limit is the Bregman projection
determined by the initialization.  Isotropic small initialization
selects the minimum-trace solution set, while nonuniqueness is
resolved by a weighted entropy criterion within the invariant
joint spectral algebra. A controlled continuous-to-discrete selection
comparison.
For fixed positive initialization, finite-step factor gradient
descent converges to an interpolant whose distance from the
continuous Bregman projection is first order in the step size.  The
theorem makes explicit the dependence of the comparison constant
on the initialization-dependent compactness bounds. Experiments matched to the theoretical mechanisms.
The numerical experiments verify the quotient identities and the
exact predictor recursions, compare the smallest effective
eigenvalue with the observed slow local mode, examine spectral
initialization and conditioning, confirm the Bregman projection
law, and illustrate minimum-trace selection, entropy tie-breaking,
and the $O(\eta)$ finite-step discrepancy.  They also demonstrate
that the explicit recovery basin, sample-size condition, and
oracle step-size bound are substantially more conservative than
the corresponding empirical behavior.

The conclusions are intentionally model-specific.  The smooth
quotient analysis is restricted to full-column-rank factors.  The
Gaussian recovery theorem is rank matched, local after a
data-dependent initialization, based on full-space concentration, and
uses an oracle learning-rate condition for gradient descent.  The
explicit Bregman and entropy selection theorems require a commuting
measurement algebra.  We do not claim that the same entropy law holds
for arbitrary noncommuting matrix sensing, that the Gaussian sample
bound is near optimal, or that a fixed nonzero step size shares the
same small-initialization limit as gradient flow.  These restrictions
allow the quotient geometry, the actual training dynamics, the
effective curvature, and the selected interpolant to be derived
rather than inserted as assumptions.

Section~\ref{sec:quotient-dynamics} develops the exact quotient
description and derives the continuous and discrete predictor
dynamics.  Section~\ref{sec:gaussian-recovery} introduces the effective
Hessian, proves the full-space Gaussian operator bound, and establishes
spectral initialization and local recovery.  Section~\ref{sec:commuting-trace-bias} analyzes the commuting underdetermined
regime, including Bregman selection, the minimum-trace limit, entropy
tie-breaking, and finite-step selection error.  Section~\ref{sec:numerical-experiments} reports the numerical experiments.
The final section discusses limitations and directions for extending
the analysis beyond positive quadratic networks.

\section{Exact Quotient Geometry and Training Dynamics}
\label{sec:quotient-dynamics}

We consider the positive-semidefinite factorization
\begin{equation}
    \pi(U) := UU^\top,
    \qquad
    U \in \mathbb{R}^{d\times r},
    \qquad
    1 \le r \le d,
    \label{eq:factorization-map}
\end{equation}
and the factorized empirical objective
\begin{equation}
    L(U) := \ell(\pi(U)),
    \label{eq:factorized-objective}
\end{equation}
where
\begin{equation}
    \ell(Q)
    :=
    \frac{1}{2n}
    \sum_{i=1}^{n}
    \bigl(
        \langle A_i,Q\rangle-y_i
    \bigr)^2,
    \qquad
    Q\in\mathbb{S}^d.
    \label{eq:predictor-objective}
\end{equation}
Here
\(
A_1,\ldots,A_n\in\mathbb{S}^d
\),
\(y\in\mathbb{R}^n\), and
\[
    \langle X,Y\rangle
    :=
    \operatorname{tr}(X^\top Y)
\]
denotes the Frobenius inner product.

The map in \cref{eq:factorization-map} is invariant under right
multiplication by an orthogonal matrix:
\begin{equation}
    \pi(UR)
    =
    URR^\top U^\top
    =
    UU^\top
    =
    \pi(U),
    \qquad
    R\in O(r).
    \label{eq:orthogonal-invariance}
\end{equation}
We show below that, on the full-column-rank stratum, this orthogonal
action accounts for the entire non-identifiability of the
factorization.

Define the full-column-rank factor space
\begin{equation}
    \mathcal{F}_{d,r}
    :=
    \left\{
        U\in\mathbb{R}^{d\times r}
        :
        \operatorname{rank}(U)=r
    \right\},
    \label{eq:factor-space}
\end{equation}
and the fixed-rank positive-semidefinite manifold
\begin{equation}
    \mathcal{M}_{r}^{+}
    :=
    \left\{
        Q\in\mathbb{S}^{d}
        :
        Q\succeq 0,
        \ \operatorname{rank}(Q)=r
    \right\}.
    \label{eq:fixed-rank-psd-manifold}
\end{equation}

\subsection{Exact identification of the quotient}
\label{subsec:exact-quotient}

We begin with the manifold structure of the predictor space.

\begin{proposition}[Fixed-rank PSD manifold]
\label{prop:fixed-rank-psd-manifold}
The set \(\mathcal{M}_{r}^{+}\) is a smooth embedded submanifold of
\(\mathbb{S}^{d}\) of dimension
\begin{equation}
    \dim \mathcal{M}_{r}^{+}
    =
    dr-\frac{r(r-1)}{2}.
    \label{eq:psd-manifold-dimension}
\end{equation}
For every
\(
Q=UU^\top\in\mathcal{M}_{r}^{+}
\)
with \(U\in\mathcal{F}_{d,r}\),
\begin{equation}
    T_Q\mathcal{M}_{r}^{+}
    =
    \left\{
        \Delta U^\top+U\Delta^\top
        :
        \Delta\in\mathbb{R}^{d\times r}
    \right\}.
    \label{eq:psd-tangent-space}
\end{equation}
\end{proposition}

\begin{proof}
Fix \(Q\in\mathcal{M}_{r}^{+}\). Choose an orthogonal matrix
\(
[V\ \ V_\perp]\in O(d)
\)
such that
\begin{equation}
    Q=V\Lambda V^\top,
    \qquad
    \Lambda\in\mathbb{S}_{++}^{r}.
    \label{eq:psd-spectral-decomposition}
\end{equation}
When \(r=d\), all blocks involving \(V_\perp\) are omitted.

For \(S\in\mathbb{S}_{++}^{r}\) sufficiently close to \(\Lambda\) and
\(B\in\mathbb{R}^{(d-r)\times r}\) sufficiently close to zero, define
\begin{equation}
    \Psi(S,B)
    :=
    [V\ \ V_\perp]
    \begin{pmatrix}
        S & B^\top \\
        B & BS^{-1}B^\top
    \end{pmatrix}
    [V\ \ V_\perp]^\top.
    \label{eq:psd-local-chart}
\end{equation}
The middle matrix factors as
\begin{equation}
    \begin{pmatrix}
        S & B^\top \\
        B & BS^{-1}B^\top
    \end{pmatrix}
    =
    \begin{pmatrix}
        I_r \\
        BS^{-1}
    \end{pmatrix}
    S
    \begin{pmatrix}
        I_r & S^{-1}B^\top
    \end{pmatrix},
    \label{eq:psd-chart-factorization}
\end{equation}
and is therefore positive semidefinite of rank \(r\).

Conversely, every \(Q'\in\mathcal{M}_{r}^{+}\) sufficiently close to
\(Q\) has, in the basis \([V\ \ V_\perp]\), a representation
\begin{equation}
    Q'
    =
    [V\ \ V_\perp]
    \begin{pmatrix}
        S & B^\top \\
        B & C
    \end{pmatrix}
    [V\ \ V_\perp]^\top
    \label{eq:nearby-psd-block-form}
\end{equation}
with \(S\in\mathbb{S}_{++}^{r}\). Since \(Q'\succeq0\) and
\(
\operatorname{rank}(Q')=\operatorname{rank}(S)=r
\),
its Schur complement vanishes:
\begin{equation}
    C-BS^{-1}B^\top=0.
    \label{eq:vanishing-schur-complement}
\end{equation}
Thus \(\Psi\) is a smooth local parametrization of
\(\mathcal{M}_{r}^{+}\).

The number of free parameters is
\[
    \frac{r(r+1)}{2}+(d-r)r
    =
    dr-\frac{r(r-1)}{2},
\]
which proves \cref{eq:psd-manifold-dimension}.

Differentiating \cref{eq:psd-local-chart} at
\((S,B)=(\Lambda,0)\) yields tangent matrices of the form
\begin{equation}
    [V\ \ V_\perp]
    \begin{pmatrix}
        H & K^\top \\
        K & 0
    \end{pmatrix}
    [V\ \ V_\perp]^\top,
    \label{eq:psd-chart-tangent}
\end{equation}
where
\(
H\in\mathbb{S}^{r}
\)
and
\(
K\in\mathbb{R}^{(d-r)\times r}
\)
are arbitrary.

Set \(U=V\Lambda^{1/2}\). For
\(
\Delta=VC+V_\perp D
\),
\begin{equation}
\begin{aligned}
    \Delta U^\top+U\Delta^\top
    &=
    [V\ \ V_\perp]
    \begin{pmatrix}
        C\Lambda^{1/2}+\Lambda^{1/2}C^\top
        &
        \Lambda^{1/2}D^\top
        \\
        D\Lambda^{1/2}
        &
        0
    \end{pmatrix}
    [V\ \ V_\perp]^\top.
    \label{eq:tangent-factor-form}
\end{aligned}
\end{equation}
Given \(H\) and \(K\) in \cref{eq:psd-chart-tangent}, choose
\[
    C=\frac{1}{2}H\Lambda^{-1/2},
    \qquad
    D=K\Lambda^{-1/2}.
\]
This produces every tangent vector in
\cref{eq:psd-chart-tangent}, proving
\cref{eq:psd-tangent-space}.
\end{proof}

The orthogonal group acts smoothly on \(\mathcal{F}_{d,r}\) by
\begin{equation}
    U\cdot R:=UR,
    \qquad
    R\in O(r).
    \label{eq:orthogonal-group-action}
\end{equation}

\begin{theorem}[Exact factor fibers]
\label{thm:exact-factor-fibers}
For \(U,V\in\mathcal{F}_{d,r}\),
\begin{equation}
    UU^\top=VV^\top
    \quad\Longleftrightarrow\quad
    V=UR
    \quad
    \text{for a unique }R\in O(r).
    \label{eq:exact-factor-fiber}
\end{equation}
Consequently, the action in \cref{eq:orthogonal-group-action} is
free and proper, and the induced map
\begin{equation}
    \overline{\pi}:
    \mathcal{F}_{d,r}/O(r)
    \longrightarrow
    \mathcal{M}_{r}^{+},
    \qquad
    \overline{\pi}([U])=UU^\top,
    \label{eq:quotient-to-psd-map}
\end{equation}
is a diffeomorphism.
\end{theorem}

\begin{proof}
If \(V=UR\) with \(R\in O(r)\), then
\[
    VV^\top
    =
    URR^\top U^\top
    =
    UU^\top.
\]

Conversely, suppose
\begin{equation}
    UU^\top=VV^\top.
    \label{eq:equal-gram-matrices}
\end{equation}
For every real matrix \(W\),
\(
\operatorname{range}(WW^\top)=\operatorname{range}(W)
\).
Therefore,
\[
    \operatorname{range}(U)
    =
    \operatorname{range}(UU^\top)
    =
    \operatorname{range}(VV^\top)
    =
    \operatorname{range}(V).
\]
Since both factors have full column rank, there exists an invertible
\(C\in\mathbb{R}^{r\times r}\) such that
\begin{equation}
    V=UC.
    \label{eq:factors-related-by-C}
\end{equation}
Substitution into \cref{eq:equal-gram-matrices} gives
\[
    U(I_r-CC^\top)U^\top=0.
\]
Let
\[
    U^\dagger
    :=
    (U^\top U)^{-1}U^\top
\]
be the left pseudoinverse of \(U\). Multiplying from the left by
\(U^\dagger\) and from the right by \((U^\dagger)^\top\) yields
\[
    CC^\top=I_r.
\]
Thus \(C\in O(r)\). If \(UR_1=UR_2\), then the full column rank of
\(U\) implies \(R_1=R_2\), proving uniqueness.

The action is free because
\[
    UR=U
    \quad\Longrightarrow\quad
    R=I_r.
\]
It is proper because \(O(r)\) is compact. Hence
\(\mathcal{F}_{d,r}/O(r)\) is a smooth manifold and the quotient map
\[
    \mathfrak{q}:
    \mathcal{F}_{d,r}
    \longrightarrow
    \mathcal{F}_{d,r}/O(r)
    \label{eq:quotient-projection}
\]
is a smooth submersion.

The differential of \(\pi\) is
\begin{equation}
    D\pi_U[\Delta]
    =
    \Delta U^\top+U\Delta^\top.
    \label{eq:differential-pi}
\end{equation}
We claim that
\begin{equation}
    \ker D\pi_U
    =
    \left\{
        U\Omega:
        \Omega^\top=-\Omega
    \right\}.
    \label{eq:kernel-differential-pi}
\end{equation}
The inclusion from right to left follows from
\[
    D\pi_U[U\Omega]
    =
    U(\Omega+\Omega^\top)U^\top
    =
    0.
\]

For the converse, suppose \(D\pi_U[\Delta]=0\). Decompose
\begin{equation}
    \Delta=UK+\Delta_\perp,
    \qquad
    U^\top\Delta_\perp=0.
    \label{eq:delta-orthogonal-decomposition}
\end{equation}
Let
\[
    P_\perp
    :=
    I_d-U(U^\top U)^{-1}U^\top.
\]
Multiplying
\(
\Delta U^\top+U\Delta^\top=0
\)
from the left by \(P_\perp\) and from the right by
\(U(U^\top U)^{-1}\) gives
\[
    \Delta_\perp=0.
\]
Hence \(\Delta=UK\). Substitution into the kernel equation gives
\[
    U(K+K^\top)U^\top=0.
\]
Full column rank of \(U\) implies
\(
K+K^\top=0
\),
proving \cref{eq:kernel-differential-pi}.

The right-hand side of \cref{eq:kernel-differential-pi} is precisely
the tangent space to the orbit \(UO(r)\). By
\cref{prop:fixed-rank-psd-manifold}, the image of \(D\pi_U\) is
\(T_{UU^\top}\mathcal{M}_{r}^{+}\). Therefore the differential of
\(\overline{\pi}\) is a linear isomorphism at every orbit.
Consequently, \(\overline{\pi}\) is a local diffeomorphism. Since it
is also bijective by \cref{eq:exact-factor-fiber}, it is a global
diffeomorphism.
\end{proof}

By \cref{thm:exact-factor-fibers}, the effective parameter space is
exactly
\begin{equation}
    \mathcal{F}_{d,r}/O(r)
    \cong
    \mathcal{M}_{r}^{+}.
    \label{eq:quotient-identification}
\end{equation}

\subsection{Vertical and horizontal spaces}
\label{subsec:horizontal-spaces}

The tangent space to the orbit through \(U\) is the vertical space
\begin{equation}
    \mathcal{V}_U
    :=
    \left\{
        U\Omega:
        \Omega^\top=-\Omega
    \right\}.
    \label{eq:vertical-space}
\end{equation}
We equip \(\mathcal{F}_{d,r}\), viewed as an open submanifold of
\(\mathbb{R}^{d\times r}\), with the Frobenius metric.

\begin{proposition}[Horizontal space and orthogonal projection]
\label{prop:horizontal-projection}
The Frobenius-orthogonal complement of \(\mathcal{V}_U\) is
\begin{equation}
    \mathcal{H}_U
    =
    \left\{
        \Delta\in\mathbb{R}^{d\times r}
        :
        U^\top\Delta
        \text{ is symmetric}
    \right\}.
    \label{eq:horizontal-space}
\end{equation}
For every \(Z\in\mathbb{R}^{d\times r}\), there exists a unique
skew-symmetric matrix \(\Omega\) satisfying
\begin{equation}
    S\Omega+\Omega S
    =
    U^\top Z-Z^\top U,
    \qquad
    S:=U^\top U.
    \label{eq:sylvester-horizontal-projection}
\end{equation}
The horizontal projection of \(Z\) is
\begin{equation}
    P_U^{\mathrm{hor}}(Z)
    =
    Z-U\Omega.
    \label{eq:horizontal-projection}
\end{equation}
\end{proposition}

\begin{proof}
For \(\Delta\in\mathbb{R}^{d\times r}\) and
\(\Omega^\top=-\Omega\),
\[
    \langle \Delta,U\Omega\rangle
    =
    \operatorname{tr}(\Delta^\top U\Omega).
\]
This vanishes for every skew-symmetric \(\Omega\) if and only if
\(\Delta^\top U\), equivalently \(U^\top\Delta\), is symmetric.
This proves \cref{eq:horizontal-space}.

Write
\[
    Z=(Z-U\Omega)+U\Omega.
\]
The first term is horizontal if and only if
\(
U^\top Z-S\Omega
\)
is symmetric. Since \(\Omega^\top=-\Omega\), this condition is
equivalent to
\[
    U^\top Z-S\Omega
    =
    Z^\top U+\Omega S,
\]
which is \cref{eq:sylvester-horizontal-projection}.

Since \(S\in\mathbb{S}_{++}^{r}\), the Sylvester operator
\[
    \mathscr{S}_S:
    \Omega\longmapsto S\Omega+\Omega S
\]
is invertible on the space of skew-symmetric matrices. Indeed, if
\[
    S
    =
    P\operatorname{diag}(\lambda_1,\ldots,\lambda_r)P^\top,
    \qquad
    \lambda_j>0,
\]
then, in this eigenbasis,
\[
    (\lambda_i+\lambda_j)\Omega_{ij}
    =
    (U^\top Z-Z^\top U)_{ij}.
\]
Every coefficient \(\lambda_i+\lambda_j\) is strictly positive, so
the skew-symmetric solution is unique. 
\cref{eq:horizontal-projection} follows.
\end{proof}

Since
\(
\ker D\pi_U=\mathcal{V}_U
\),
the restriction
\begin{equation}
    D\pi_U\big|_{\mathcal{H}_U}:
    \mathcal{H}_U
    \longrightarrow
    T_{UU^\top}\mathcal{M}_{r}^{+}
    \label{eq:horizontal-differential-isomorphism}
\end{equation}
is a linear isomorphism.

For
\(
\xi\in T_Q\mathcal{M}_{r}^{+}
\)
with \(Q=UU^\top\), we denote by
\(
\xi_U^\uparrow\in\mathcal{H}_U
\)
the unique horizontal lift satisfying
\begin{equation}
    D\pi_U[\xi_U^\uparrow]=\xi.
    \label{eq:horizontal-lift}
\end{equation}
Every other lift differs from \(\xi_U^\uparrow\) by a vertical
vector. Consequently,
\begin{equation}
    \|\xi_U^\uparrow\|_F
    =
    \min_{\Delta:\,D\pi_U[\Delta]=\xi}
    \|\Delta\|_F.
    \label{eq:minimum-norm-horizontal-lift}
\end{equation}

\subsection{The quotient metric}
\label{subsec:quotient-metric}

\begin{definition}[Quotient metric]
\label{def:quotient-metric}
For \(Q=UU^\top\in\mathcal{M}_{r}^{+}\) and
\(
\xi,\zeta\in T_Q\mathcal{M}_{r}^{+}
\),
define
\begin{equation}
    g_Q(\xi,\zeta)
    :=
    \left\langle
        \xi_U^\uparrow,
        \zeta_U^\uparrow
    \right\rangle.
    \label{eq:quotient-metric}
\end{equation}
\end{definition}

\begin{proposition}[Well-definedness of the quotient metric]
\label{prop:quotient-metric-well-defined}
The metric in \cref{def:quotient-metric} is independent of the
choice of factor \(U\). It defines a smooth Riemannian metric on
\(\mathcal{M}_{r}^{+}\), and
\begin{equation}
    \pi:
    \bigl(
        \mathcal{F}_{d,r},
        \langle\cdot,\cdot\rangle_F
    \bigr)
    \longrightarrow
    \bigl(
        \mathcal{M}_{r}^{+},
        g
    \bigr)
    \label{eq:riemannian-submersion-map}
\end{equation}
is a Riemannian submersion.
\end{proposition}

\begin{proof}
Let \(U'=UR\) with \(R\in O(r)\). For
\(\Delta\in\mathbb{R}^{d\times r}\),
\[
    (U')^\top(\Delta R)
    =
    R^\top U^\top\Delta R.
\]
Therefore,
\begin{equation}
    \Delta\in\mathcal{H}_U
    \quad\Longleftrightarrow\quad
    \Delta R\in\mathcal{H}_{U'}.
    \label{eq:horizontal-equivariance}
\end{equation}
Moreover,
\begin{equation}
    D\pi_{U'}[\Delta R]
    =
    \Delta RR^\top U^\top
    +
    URR^\top\Delta^\top
    =
    D\pi_U[\Delta].
    \label{eq:differential-equivariance}
\end{equation}
Hence, if \(\xi_U^\uparrow\) is the horizontal lift of \(\xi\) at
\(U\), then
\[
    \xi_{U'}^\uparrow
    =
    \xi_U^\uparrow R.
\]
It follows that
\[
\begin{aligned}
    \left\langle
        \xi_{U'}^\uparrow,
        \zeta_{U'}^\uparrow
    \right\rangle
    &=
    \left\langle
        \xi_U^\uparrow R,
        \zeta_U^\uparrow R
    \right\rangle
    \\
    &=
    \left\langle
        \xi_U^\uparrow,
        \zeta_U^\uparrow
    \right\rangle.
\end{aligned}
\]
Thus \(g\) is independent of the representative.

Positive definiteness follows from
\[
    g_Q(\xi,\xi)
    =
    \|\xi_U^\uparrow\|_F^2,
\]
which vanishes only when \(\xi=0\).

The horizontal projection depends smoothly on \(U\). Indeed,
\(S=U^\top U\) remains positive definite on \(\mathcal{F}_{d,r}\),
and the inverse of the Sylvester operator in
\cref{eq:sylvester-horizontal-projection} varies smoothly with
\(S\). Hence the horizontal distribution and the induced metric are
smooth.

Finally, for \(\Delta,\Gamma\in\mathcal{H}_U\),
\[
    g_{UU^\top}
    \bigl(
        D\pi_U[\Delta],
        D\pi_U[\Gamma]
    \bigr)
    =
    \langle\Delta,\Gamma\rangle.
\]
Thus \(\pi\) is a Riemannian submersion.
\end{proof}

\begin{proposition}[Comparison of quotient and predictor norms]
\label{prop:quotient-predictor-norm-comparison}
Let
\(
Q=UU^\top\in\mathcal{M}_{r}^{+}
\),
let
\(
\xi\in T_Q\mathcal{M}_{r}^{+}
\),
and let
\(
\Delta=\xi_U^\uparrow
\)
be its horizontal lift. Then
\begin{equation}
    2\lambda_r(Q)\|\Delta\|_F^2
    \le
    \|\xi\|_F^2
    \le
    4\lambda_1(Q)\|\Delta\|_F^2.
    \label{eq:quotient-predictor-norm-comparison}
\end{equation}
Equivalently,
\begin{equation}
    2\lambda_r(Q)g_Q(\xi,\xi)
    \le
    \|\xi\|_F^2
    \le
    4\lambda_1(Q)g_Q(\xi,\xi).
    \label{eq:metric-predictor-norm-comparison}
\end{equation}
\end{proposition}

\begin{proof}
Set
\(
S:=U^\top U
\).
The nonzero eigenvalues of \(Q=UU^\top\) coincide with the
eigenvalues of \(S\). Since
\(
\xi=\Delta U^\top+U\Delta^\top
\),
\begin{equation}
\begin{aligned}
    \|\xi\|_F^2
    &=
    \|\Delta U^\top\|_F^2
    +
    \|U\Delta^\top\|_F^2
    +
    2
    \langle
        \Delta U^\top,
        U\Delta^\top
    \rangle
    \\
    &=
    2\operatorname{tr}(\Delta^\top\Delta S)
    +
    2\operatorname{tr}\bigl((\Delta^\top U)^2\bigr).
    \label{eq:tangent-norm-expansion}
\end{aligned}
\end{equation}
Because \(\Delta\) is horizontal, \(\Delta^\top U\) is symmetric.
Therefore,
\[
    \operatorname{tr}\bigl((\Delta^\top U)^2\bigr)
    =
    \|\Delta^\top U\|_F^2
    \ge0.
\]
Consequently,
\[
    \|\xi\|_F^2
    \ge
    2\operatorname{tr}(\Delta^\top\Delta S)
    \ge
    2\lambda_r(S)\|\Delta\|_F^2.
\]

For the upper bound,
\[
    \operatorname{tr}(\Delta^\top\Delta S)
    \le
    \lambda_1(S)\|\Delta\|_F^2
\]
and
\[
    \|\Delta^\top U\|_F^2
    \le
    \|U\|_{\mathrm{op}}^2\|\Delta\|_F^2
    =
    \lambda_1(S)\|\Delta\|_F^2.
\]
Substitution into \cref{eq:tangent-norm-expansion} yields
\[
    \|\xi\|_F^2
    \le
    4\lambda_1(S)\|\Delta\|_F^2.
\]
Since
\(
\lambda_j(S)=\lambda_j(Q)
\)
for \(1\le j\le r\), the result follows.
\end{proof}

\subsection{Procrustes quotient distance}
\label{subsec:procrustes-distance}

For \([U],[V]\in\mathcal{F}_{d,r}/O(r)\), define
\begin{equation}
    d_{\mathrm{P}}([U],[V])
    :=
    \min_{R\in O(r)}
    \|U-VR\|_F.
    \label{eq:procrustes-distance}
\end{equation}

\begin{proposition}[Procrustes distance and optimal alignment]
\label{prop:procrustes-distance}
The quantity in \cref{eq:procrustes-distance} is a metric on
\(\mathcal{F}_{d,r}/O(r)\). Moreover,
\begin{equation}
    d_{\mathrm{P}}([U],[V])^2
    =
    \|U\|_F^2
    +
    \|V\|_F^2
    -
    2\|U^\top V\|_*,
    \label{eq:procrustes-distance-formula}
\end{equation}
where \(\|\cdot\|_*\) denotes the nuclear norm.

There exists an optimal \(R_\star\in O(r)\) such that, with
\[
    \widetilde V:=VR_\star,
\]
the matrix \(\widetilde V^\top U\) is symmetric positive
semidefinite. Consequently,
\begin{equation}
    \widetilde V^\top(U-\widetilde V)
    \quad\text{is symmetric}.
    \label{eq:procrustes-horizontal-displacement}
\end{equation}
If \(V^\top U\) is nonsingular, the optimal alignment is unique.
\end{proposition}

\begin{proof}
Compactness of \(O(r)\) guarantees the existence of a minimizer.
Expanding the objective gives
\begin{equation}
    \|U-VR\|_F^2
    =
    \|U\|_F^2
    +
    \|V\|_F^2
    -
    2\operatorname{tr}(U^\top VR).
    \label{eq:procrustes-expansion}
\end{equation}
Let
\[
    U^\top V=P\Sigma Q^\top
\]
be a singular-value decomposition. Then
\[
    \max_{R\in O(r)}
    \operatorname{tr}(U^\top VR)
    =
    \operatorname{tr}(\Sigma)
    =
    \|U^\top V\|_*,
\]
and one maximizer is
\[
    R_\star=QP^\top.
\]
This proves \cref{eq:procrustes-distance-formula}. For this choice,
\[
    \widetilde V^\top U
    =
    R_\star^\top V^\top U
    =
    P\Sigma P^\top,
\]
which is symmetric positive semidefinite. Since
\(\widetilde V^\top\widetilde V\) is symmetric,
\cref{eq:procrustes-horizontal-displacement} follows.

The action of \(O(r)\) on \(\mathcal{F}_{d,r}\) is isometric.
Symmetry and definiteness of \(d_{\mathrm P}\) follow immediately.
For the triangle inequality, let \(R_1,R_2\in O(r)\). Then
\[
\begin{aligned}
    d_{\mathrm P}([U],[W])
    &\le
    \|U-WR_2R_1\|_F
    \\
    &\le
    \|U-VR_1\|_F
    +
    \|VR_1-WR_2R_1\|_F
    \\
    &=
    \|U-VR_1\|_F
    +
    \|V-WR_2\|_F.
\end{aligned}
\]
Taking the infimum over \(R_1\) and \(R_2\) proves the triangle
inequality.

If \(V^\top U\) is nonsingular, its orthogonal polar factor is unique,
and hence the Procrustes minimizer is unique.
\end{proof}

\begin{remark}
\label{rem:procrustes-versus-geodesic}
The Procrustes distance in \cref{eq:procrustes-distance} is the
orbit-space metric induced by the ambient Frobenius norm. It is not
identified here with the geodesic distance associated with the
Riemannian metric \(g\).
\end{remark}

\subsection{Exact projection of Euclidean gradient flow}
\label{subsec:quotient-gradient-flow}

Define the measurement operator
\begin{equation}
    \mathcal{A}:
    \mathbb{S}^d
    \longrightarrow
    \mathbb{R}^n,
    \qquad
    [\mathcal{A}(Q)]_i
    :=
    \langle A_i,Q\rangle,
    \label{eq:measurement-operator}
\end{equation}
and its adjoint
\begin{equation}
    \mathcal{A}^*(z)
    :=
    \sum_{i=1}^{n}z_iA_i.
    \label{eq:measurement-adjoint}
\end{equation}
Then
\[
    \ell(Q)
    =
    \frac{1}{2n}
    \|\mathcal{A}(Q)-y\|_2^2,
\]
and its Euclidean gradient in \(\mathbb{S}^d\) is
\begin{equation}
    G(Q)
    :=
    \nabla\ell(Q)
    =
    \frac1n
    \mathcal{A}^*
    \bigl(
        \mathcal{A}(Q)-y
    \bigr)
    \in\mathbb{S}^d.
    \label{eq:predictor-gradient}
\end{equation}

\begin{lemma}[Factor gradient]
\label{lem:factor-gradient}
For every \(U\in\mathbb{R}^{d\times r}\),
\begin{equation}
    \nabla L(U)
    =
    2G(UU^\top)U.
    \label{eq:factor-gradient}
\end{equation}
\end{lemma}

\begin{proof}
For \(\Delta\in\mathbb{R}^{d\times r}\),
\[
    D\pi_U[\Delta]
    =
    \Delta U^\top+U\Delta^\top.
\]
Therefore,
\[
\begin{aligned}
    DL(U)[\Delta]
    &=
    \left\langle
        G(UU^\top),
        \Delta U^\top+U\Delta^\top
    \right\rangle
    \\
    &=
    2
    \left\langle
        G(UU^\top)U,
        \Delta
    \right\rangle,
\end{aligned}
\]
where symmetry of \(G(UU^\top)\) was used.
\end{proof}

For \(U\in\mathcal{F}_{d,r}\), the gradient is horizontal because
\begin{equation}
    U^\top\nabla L(U)
    =
    2U^\top G(UU^\top)U
    \label{eq:factor-gradient-horizontal}
\end{equation}
is symmetric.

\begin{theorem}[Exact quotient gradient flow]
\label{thm:exact-quotient-gradient-flow}
Let \(U_0\in\mathcal{F}_{d,r}\). The Euclidean gradient-flow equation
\begin{equation}
    \dot U(t)
    =
    -\nabla L(U(t))
    =
    -2G(Q(t))U(t),
    \qquad
    Q(t):=U(t)U(t)^\top,
    \label{eq:factor-gradient-flow}
\end{equation}
has a unique solution for every \(t\ge0\). Moreover:
\begin{enumerate}
    \item
    \(U(t)\in\mathcal{F}_{d,r}\) for every finite \(t\ge0\);

    \item
    \(\dot U(t)\in\mathcal{H}_{U(t)}\);

    \item
    the predictor satisfies
    \begin{equation}
        \dot Q(t)
        =
        -2
        \bigl(
            G(Q(t))Q(t)+Q(t)G(Q(t))
        \bigr);
        \label{eq:predictor-gradient-flow}
    \end{equation}

    \item
    the Riemannian gradient of
    \(\ell|_{\mathcal{M}_{r}^{+}}\) is
    \begin{equation}
        \operatorname{grad}_g\ell(Q)
        =
        2\bigl(G(Q)Q+QG(Q)\bigr),
        \label{eq:quotient-riemannian-gradient}
    \end{equation}
    and hence
    \begin{equation}
        \dot Q(t)
        =
        -\operatorname{grad}_g\ell(Q(t)).
        \label{eq:quotient-gradient-flow}
    \end{equation}
\end{enumerate}
\end{theorem}

\begin{proof}
The vector field
\[
    U\longmapsto-2G(UU^\top)U
\]
is smooth on \(\mathbb{R}^{d\times r}\), so there is a unique maximal
solution on an interval \([0,T_{\max})\).

Along this solution,
\begin{equation}
    \frac{d}{dt}L(U(t))
    =
    -\|\nabla L(U(t))\|_F^2.
    \label{eq:factor-flow-energy-identity}
\end{equation}
Since \(L\ge0\),
\begin{equation}
    \int_0^T
    \|\dot U(t)\|_F^2\,dt
    \le
    L(U_0)
    \qquad
    \text{for every }T<T_{\max}.
    \label{eq:finite-flow-energy}
\end{equation}
If \(T_{\max}<\infty\), then, for \(0\le s<t<T_{\max}\),
\[
\begin{aligned}
    \|U(t)-U(s)\|_F
    &\le
    \int_s^t\|\dot U(\tau)\|_F\,d\tau
    \\
    &\le
    \sqrt{t-s}
    \left(
        \int_s^t
        \|\dot U(\tau)\|_F^2\,d\tau
    \right)^{1/2}.
\end{aligned}
\]
The right-hand side tends to zero as \(s,t\uparrow T_{\max}\).
Thus \(U(t)\) converges to a finite limit as
\(t\uparrow T_{\max}\). Smoothness of the vector field then permits
continuation beyond \(T_{\max}\), a contradiction. Hence
\(T_{\max}=\infty\).

To prove rank preservation, define
\[
    B(t):=-2G(Q(t)).
\]
Then
\(
\dot U(t)=B(t)U(t)
\).
Let \(P(t)\) solve
\begin{equation}
    \dot P(t)=B(t)P(t),
    \qquad
    P(0)=I_d.
    \label{eq:fundamental-matrix}
\end{equation}
By uniqueness,
\begin{equation}
    U(t)=P(t)U_0.
    \label{eq:factor-flow-fundamental-solution}
\end{equation}
Liouville's formula gives
\begin{equation}
    \det P(t)
    =
    \exp
    \left(
        \int_0^t
        \operatorname{tr}B(s)\,ds
    \right)
    \neq0
    \label{eq:fundamental-matrix-determinant}
\end{equation}
for every finite \(t\). Therefore \(P(t)\) is invertible and
\[
    \operatorname{rank}(U(t))
    =
    \operatorname{rank}(U_0)
    =
    r.
\]

By \cref{eq:factor-gradient-horizontal},
\(\nabla L(U(t))\) is horizontal, so
\(\dot U(t)\in\mathcal{H}_{U(t)}\).

Differentiating \(Q(t)=U(t)U(t)^\top\) and using
\cref{eq:factor-gradient-flow} yields
\[
\begin{aligned}
    \dot Q(t)
    &=
    \dot U(t)U(t)^\top
    +
    U(t)\dot U(t)^\top
    \\
    &=
    -2G(Q(t))Q(t)
    -
    2Q(t)G(Q(t)),
\end{aligned}
\]
which proves \cref{eq:predictor-gradient-flow}.

Finally, let
\(
\xi\in T_Q\mathcal{M}_{r}^{+}
\)
and let
\(
\Delta=\xi_U^\uparrow
\)
be its horizontal lift. Then
\[
\begin{aligned}
    D\ell(Q)[\xi]
    &=
    \left\langle
        G(Q),
        \Delta U^\top+U\Delta^\top
    \right\rangle
    \\
    &=
    2\langle G(Q)U,\Delta\rangle.
\end{aligned}
\]
The vector \(2G(Q)U\) is horizontal. Therefore,
\[
    D\ell(Q)[\xi]
    =
    g_Q
    \left(
        D\pi_U[2G(Q)U],
        \xi
    \right).
\]
By the defining property of the Riemannian gradient,
\[
\begin{aligned}
    \operatorname{grad}_g\ell(Q)
    &=
    D\pi_U[2G(Q)U]
    \\
    &=
    2\bigl(G(Q)Q+QG(Q)\bigr),
\end{aligned}
\]
which proves
\cref{eq:quotient-riemannian-gradient,eq:quotient-gradient-flow}.
\end{proof}

As an immediate consequence of
\cref{thm:exact-quotient-gradient-flow},
\begin{equation}
    \frac{d}{dt}\ell(Q(t))
    =
    -
    g_{Q(t)}
    \left(
        \operatorname{grad}_g\ell(Q(t)),
        \operatorname{grad}_g\ell(Q(t))
    \right).
    \label{eq:quotient-energy-identity}
\end{equation}

\subsection{Exact predictor recursion for parameter gradient descent}
\label{subsec:discrete-predictor-dynamics}

Consider ordinary Euclidean gradient descent:
\begin{equation}
    U_{k+1}
    =
    U_k-\eta\nabla L(U_k)
    =
    U_k-2\eta G(Q_k)U_k,
    \qquad
    Q_k:=U_kU_k^\top.
    \label{eq:factor-gradient-descent}
\end{equation}
The iteration is equivariant under the orthogonal action. If
\(
\widetilde U_k=U_kR
\)
for some \(R\in O(r)\), then
\(
\widetilde Q_k=Q_k
\)
and
\(
\widetilde U_{k+1}=U_{k+1}R
\).
Hence \cref{eq:factor-gradient-descent} induces a well-defined map on
the quotient.

\begin{proposition}[Exact discrete predictor dynamics]
\label{prop:exact-discrete-predictor-dynamics}
The predictor sequence generated by
\cref{eq:factor-gradient-descent} satisfies
\begin{equation}
    Q_{k+1}
    =
    \bigl(I_d-2\eta G(Q_k)\bigr)
    Q_k
    \bigl(I_d-2\eta G(Q_k)\bigr),
    \label{eq:exact-discrete-congruence}
\end{equation}
or equivalently,
\begin{equation}
\begin{aligned}
    Q_{k+1}
    &=
    Q_k
    -
    2\eta
    \bigl(
        G(Q_k)Q_k+Q_kG(Q_k)
    \bigr)
    \\
    &\quad
    +
    4\eta^2
    G(Q_k)Q_kG(Q_k).
    \label{eq:exact-discrete-expanded}
\end{aligned}
\end{equation}
Consequently:
\begin{enumerate}
    \item
    \(Q_{k+1}\succeq0\) whenever \(Q_k\succeq0\);

    \item
    if \(Q_k\in\mathcal{M}_{r}^{+}\), then
    \(\operatorname{rank}(Q_{k+1})=r\) if and only if
    \begin{equation}
        \ker
        \bigl(
            I_d-2\eta G(Q_k)
        \bigr)
        \cap
        \operatorname{range}(Q_k)
        =
        \{0\};
        \label{eq:exact-rank-preservation-condition}
    \end{equation}

    \item
    the sufficient condition
    \begin{equation}
        2\eta\|G(Q_k)\|_{\mathrm{op}}<1
        \label{eq:sufficient-rank-preservation}
    \end{equation}
    guarantees \(Q_{k+1}\in\mathcal{M}_{r}^{+}\);

    \item
    whenever \(Q_k\in\mathcal{M}_{r}^{+}\),
    \begin{equation}
        \frac{Q_{k+1}-Q_k}{\eta}
        =
        -\operatorname{grad}_g\ell(Q_k)
        +
        4\eta G(Q_k)Q_kG(Q_k);
        \label{eq:discrete-gradient-plus-correction}
    \end{equation}

    \item
    the correction term satisfies
    \begin{equation}
        \left\|
            4\eta G(Q_k)Q_kG(Q_k)
        \right\|_F
        \le
        4\eta
        \|G(Q_k)\|_{\mathrm{op}}^2
        \|Q_k\|_F.
        \label{eq:discrete-correction-bound}
    \end{equation}
\end{enumerate}
\end{proposition}

\begin{proof}
Set
\[
    M_k:=I_d-2\eta G(Q_k).
\]
Then
\[
    U_{k+1}=M_kU_k.
\]
Because \(G(Q_k)\) is symmetric, \(M_k^\top=M_k\), and therefore
\[
    Q_{k+1}
    =
    M_kU_kU_k^\top M_k
    =
    M_kQ_kM_k.
\]
This proves \cref{eq:exact-discrete-congruence}; expanding the product
gives \cref{eq:exact-discrete-expanded}.

\cref{eq:exact-discrete-congruence} is a congruence
transformation, so \(Q_{k+1}\succeq0\) whenever \(Q_k\succeq0\).

Moreover,
\[
    \operatorname{rank}(Q_{k+1})
    =
    \operatorname{rank}(M_kU_k).
\]
Since
\(
\operatorname{range}(Q_k)=\operatorname{range}(U_k)
\),
the rank remains equal to \(r\) precisely when \(M_k\) is injective
on \(\operatorname{range}(Q_k)\), which is equivalent to
\cref{eq:exact-rank-preservation-condition}.

Under \cref{eq:sufficient-rank-preservation}, every eigenvalue of
\(M_k\) belongs to \((0,2)\). Hence \(M_k\) is invertible, which
implies rank preservation.

Using \cref{eq:quotient-riemannian-gradient} in
\cref{eq:exact-discrete-expanded} gives
\cref{eq:discrete-gradient-plus-correction}.

Finally,
\[
\begin{aligned}
    \|G(Q_k)Q_kG(Q_k)\|_F
    &\le
    \|G(Q_k)\|_{\mathrm{op}}
    \|Q_kG(Q_k)\|_F
    \\
    &\le
    \|G(Q_k)\|_{\mathrm{op}}^2
    \|Q_k\|_F,
\end{aligned}
\]
which proves \cref{eq:discrete-correction-bound}.
\end{proof}

\begin{remark}[Continuous and discrete quotient dynamics]
\label{rem:continuous-discrete-quotient}
By \cref{thm:exact-quotient-gradient-flow}, Euclidean parameter
gradient flow induces exactly the Riemannian gradient flow on
\((\mathcal{M}_{r}^{+},g)\). In contrast,
\cref{eq:discrete-gradient-plus-correction} shows that finite-step
parameter gradient descent is not, in general, identical to a
Riemannian gradient-descent step. Its first-order term is the
quotient-gradient direction, while
\[
    4\eta^2G(Q_k)Q_kG(Q_k)
\]
is the exact factorization-induced discretization correction.
\end{remark}

\section{Effective Curvature and Recovery under Gaussian Rank-One Measurements}
\label{sec:gaussian-recovery}

We now specialize the PSD-factorized model to Gaussian rank-one
measurements. Let
\begin{equation}
x_1,\ldots,x_n
\overset{\mathrm{iid}}{\sim}
N(0,I_d),
\label{eq:gaussian-design}
\end{equation}
and let
\begin{equation}
Q_\star
=
U_\star U_\star^\top
\in\mathcal{M}_{r}^{+},
\qquad
U_\star\in\mathcal{F}_{d,r},
\label{eq:true-psd-predictor}
\end{equation}
be the target predictor. The responses are
\begin{equation}
y_i
=
x_i^\top Q_\star x_i,
\qquad
1\le i\le n.
\label{eq:quadratic-observations}
\end{equation}

The empirical predictor loss is
\begin{equation}
\ell_n(Q)
:=
\frac{1}{2n}
\sum_{i=1}^{n}
\bigl(
x_i^\top(Q-Q_\star)x_i
\bigr)^2,
\qquad
Q\in\mathbb{S}^d,
\label{eq:empirical-predictor-loss}
\end{equation}
and the corresponding factorized loss is
\begin{equation}
L_n(U)
:=
\ell_n(UU^\top).
\label{eq:empirical-factor-loss}
\end{equation}

For notational convenience, throughout this section we write
\begin{equation}
d_{\mathrm P}(U,V)
:=
d_{\mathrm P}([U],[V]),
\label{eq:procrustes-distance-shorthand}
\end{equation}
where $d_{\mathrm P}$ is the quotient distance characterized in
\cref{prop:procrustes-distance}.

We first identify the quotient Hessian at $Q_\star$. We then prove
a uniform full-space deviation bound for the empirical
quadratic-measurement operator, construct a moment-based spectral
initializer, and establish local exponential convergence of ordinary
Euclidean factor gradient flow and local linear convergence of factor
gradient descent. The resulting statistical guarantee is deliberately
conservative: the empirical operator is controlled on all of
$\mathbb{S}^d$, rather than only on the tangent space or on a
low-rank secant set.

\subsection{Empirical normal operator and effective Hessian}
\label{subsec:effective-hessian}

Define the measurement operator
\begin{equation}
\mathcal{A}_n:
\mathbb{S}^d
\longrightarrow
\mathbb{R}^n,
\qquad
[\mathcal{A}_n(H)]_i
:=
x_i^\top Hx_i,
\label{eq:gaussian-measurement-operator}
\end{equation}
and the empirical normal operator
\begin{equation}
\mathcal{T}_n:
\mathbb{S}^d
\longrightarrow
\mathbb{S}^d,
\qquad
\mathcal{T}_n(H)
:=
\frac1n
\sum_{i=1}^{n}
(x_i^\top Hx_i)x_ix_i^\top.
\label{eq:empirical-normal-operator}
\end{equation}
For $H,K\in\mathbb{S}^d$,
\begin{equation}
\langle H,\mathcal{T}_n(K)\rangle
=
\frac1n
\sum_{i=1}^{n}
(x_i^\top Hx_i)(x_i^\top Kx_i).
\label{eq:empirical-normal-bilinear-form}
\end{equation}
Thus $\mathcal{T}_n$ is self-adjoint and positive semidefinite with
respect to the Frobenius inner product.

Writing
\begin{equation}
E(Q)
:=
Q-Q_\star,
\label{eq:predictor-error}
\end{equation}
we have
\begin{equation}
\ell_n(Q)
=
\frac12
\left\langle
E(Q),
\mathcal{T}_n(E(Q))
\right\rangle,
\label{eq:loss-normal-operator-form}
\end{equation}
\begin{equation}
\nabla\ell_n(Q)
=
\mathcal{T}_n(E(Q)),
\label{eq:predictor-gradient-normal-operator}
\end{equation}
and
\begin{equation}
\nabla L_n(U)
=
2\mathcal{T}_n(UU^\top-Q_\star)U.
\label{eq:factor-gradient-gaussian}
\end{equation}

At $Q_\star$, the ambient differential of $\ell_n$ vanishes.
Consequently, the Hessian bilinear form of the restriction of
$\ell_n$ to $\mathcal{M}_{r}^{+}$ is independent of the choice of
connection at that point and equals the restriction of the ambient
second differential.

\begin{theorem}[Effective Hessian at the interpolating solution]
\label{thm:effective-hessian}
Let
$\xi,\zeta\in T_{Q_\star}\mathcal{M}_{r}^{+}$.
Then
\begin{equation}
\operatorname{Hess}_g\ell_n(Q_\star)[\xi,\zeta]
=
\langle\xi,\mathcal{T}_n(\zeta)\rangle
=
\frac1n
\sum_{i=1}^{n}
(x_i^\top\xi x_i)
(x_i^\top\zeta x_i).
\label{eq:effective-hessian-bilinear}
\end{equation}
In particular,
\begin{equation}
\operatorname{Hess}_g\ell_n(Q_\star)[\xi,\xi]
=
\frac1n
\|\mathcal{A}_n(\xi)\|_2^2,
\label{eq:effective-hessian-quadratic}
\end{equation}
and the nullspace of the associated self-adjoint Hessian operator is
\begin{equation}
\ker
\operatorname{Hess}_g\ell_n(Q_\star)
=
T_{Q_\star}\mathcal{M}_{r}^{+}
\cap
\ker\mathcal{A}_n.
\label{eq:effective-hessian-kernel}
\end{equation}
\end{theorem}

\begin{proof}
The ambient loss in \cref{eq:empirical-predictor-loss} is quadratic,
and its first and second differentials are
\[
D\ell_n(Q)[H]
=
\left\langle
\mathcal{T}_n(Q-Q_\star),
H
\right\rangle
\]
and
\[
D^2\ell_n(Q)[H,K]
=
\langle H,\mathcal{T}_n(K)\rangle.
\]
In particular,
\[
D\ell_n(Q_\star)=0.
\]

Let $\nabla^g$ be the Levi--Civita connection of the quotient metric
$g$. For smooth tangent-vector extensions $X$ and $Y$ satisfying
$X(Q_\star)=\xi$ and $Y(Q_\star)=\zeta$,
\[
\begin{aligned}
\operatorname{Hess}_g\ell_n(Q_\star)[\xi,\zeta]
&=
X(Y\ell_n)(Q_\star)
-
D\ell_n(Q_\star)
\bigl[
(\nabla^g_XY)(Q_\star)
\bigr]
\\
&=
X(Y\ell_n)(Q_\star).
\end{aligned}
\]
Because $D\ell_n(Q_\star)=0$, this value is precisely the intrinsic
second differential of the restriction of $\ell_n$ to
$\mathcal{M}_{r}^{+}$, and hence
\[
\operatorname{Hess}_g\ell_n(Q_\star)[\xi,\zeta]
=
D^2\ell_n(Q_\star)[\xi,\zeta]
=
\langle\xi,\mathcal{T}_n(\zeta)\rangle.
\]
Using \cref{eq:empirical-normal-bilinear-form} proves
\cref{eq:effective-hessian-bilinear}, and taking
$\zeta=\xi$ gives \cref{eq:effective-hessian-quadratic}.

The quadratic form in \cref{eq:effective-hessian-quadratic} vanishes
if and only if
\[
x_i^\top\xi x_i=0
\qquad
\text{for all }1\le i\le n,
\]
which is equivalent to
$\mathcal{A}_n(\xi)=0$.
Since the Hessian form is positive semidefinite, the zero set of its
quadratic form equals the nullspace of its associated self-adjoint
operator. Restricting to
$T_{Q_\star}\mathcal{M}_{r}^{+}$
proves \cref{eq:effective-hessian-kernel}.
\end{proof}

The effective eigenvalues are defined relative to the quotient
metric:
\begin{equation}
\lambda_{\min}^{\mathrm{eff}}(Q_\star)
:=
\inf_{
\xi\in
T_{Q_\star}\mathcal{M}_{r}^{+}
\setminus\{0\}
}
\frac{
\langle\xi,\mathcal{T}_n(\xi)\rangle
}{
g_{Q_\star}(\xi,\xi)
},
\label{eq:min-effective-eigenvalue}
\end{equation}
and
\begin{equation}
\lambda_{\max}^{\mathrm{eff}}(Q_\star)
:=
\sup_{
\xi\in
T_{Q_\star}\mathcal{M}_{r}^{+}
\setminus\{0\}
}
\frac{
\langle\xi,\mathcal{T}_n(\xi)\rangle
}{
g_{Q_\star}(\xi,\xi)
}.
\label{eq:max-effective-eigenvalue}
\end{equation}
These are generalized eigenvalues: the numerator measures
prediction-space curvature, whereas the denominator is the minimum
factor-space energy required to realize the tangent perturbation.

\subsection{Population curvature}
\label{subsec:population-curvature}

Let $x\sim N(0,I_d)$, and define
\begin{equation}
\mathcal{T}(H)
:=
\mathbb{E}
\bigl[
(x^\top Hx)xx^\top
\bigr].
\label{eq:population-normal-operator}
\end{equation}

\begin{lemma}[Gaussian population operator]
\label{lem:gaussian-population-operator}
For every $H\in\mathbb{S}^d$,
\begin{equation}
\mathcal{T}(H)
=
2H+\operatorname{tr}(H)I_d.
\label{eq:population-operator-formula}
\end{equation}
Consequently,
\begin{equation}
\langle H,\mathcal{T}(H)\rangle
=
2\|H\|_F^2+\operatorname{tr}(H)^2,
\label{eq:population-quadratic-form}
\end{equation}
and
\begin{equation}
2\|H\|_F^2
\le
\langle H,\mathcal{T}(H)\rangle
\le
(d+2)\|H\|_F^2.
\label{eq:population-coercivity-smoothness}
\end{equation}
\end{lemma}

\begin{proof}
For indices $a,b,j,k$, Isserlis' identity gives
\[
\mathbb{E}[x_ax_bx_jx_k]
=
\delta_{ab}\delta_{jk}
+
\delta_{aj}\delta_{bk}
+
\delta_{ak}\delta_{bj}.
\]
Therefore,
\[
\begin{aligned}
[\mathcal{T}(H)]_{ab}
&=
\sum_{j,k=1}^{d}
H_{jk}
\mathbb{E}[x_ax_bx_jx_k]
\\
&=
\delta_{ab}\operatorname{tr}(H)
+
H_{ab}
+
H_{ba}.
\end{aligned}
\]
Since $H$ is symmetric,
\[
\mathcal{T}(H)
=
2H+\operatorname{tr}(H)I_d.
\]
Taking the Frobenius inner product with $H$ proves
\cref{eq:population-quadratic-form}. Finally,
\[
\operatorname{tr}(H)^2
\le
d\|H\|_F^2
\]
gives \cref{eq:population-coercivity-smoothness}.
\end{proof}

Define the population effective eigenvalues by replacing
$\mathcal{T}_n$ with $\mathcal{T}$ in
\cref{eq:min-effective-eigenvalue,eq:max-effective-eigenvalue}.
Combining
\cref{lem:gaussian-population-operator,prop:quotient-predictor-norm-comparison}
gives
\begin{equation}
\lambda_{\min}^{\mathrm{pop}}(Q_\star)
\ge
4\lambda_r(Q_\star),
\label{eq:population-effective-lower}
\end{equation}
and
\begin{equation}
\lambda_{\max}^{\mathrm{pop}}(Q_\star)
\le
4(d+2)\lambda_1(Q_\star).
\label{eq:population-effective-upper}
\end{equation}
Indeed, for every
$\xi\in T_{Q_\star}\mathcal{M}_{r}^{+}$,
\[
\begin{aligned}
\langle\xi,\mathcal{T}(\xi)\rangle
&\ge
2\|\xi\|_F^2
\\
&\ge
4\lambda_r(Q_\star)
g_{Q_\star}(\xi,\xi),
\end{aligned}
\]
and the upper bound follows analogously.

\subsection{Full-space concentration of the empirical operator}
\label{subsec:full-space-concentration}

For a linear operator
$\mathcal{B}:\mathbb{S}^d\to\mathbb{S}^d$,
define
\begin{equation}
\|\mathcal{B}\|_{\mathrm{op}}
:=
\sup_{\|H\|_F=1}
\|\mathcal{B}(H)\|_F.
\label{eq:operator-norm-on-symmetric-matrices}
\end{equation}
Set
\begin{equation}
\kappa_d
:=
d(d+2)(d+4)(d+6).
\label{eq:kappa-d-definition}
\end{equation}

\begin{theorem}[Full-space empirical-operator deviation bound]
\label{thm:full-space-operator-concentration}
For every $\varepsilon>0$,
\begin{equation}
\mathbb{P}
\left(
\|\mathcal{T}_n-\mathcal{T}\|_{\mathrm{op}}
>
\varepsilon
\right)
\le
\frac{\kappa_d}{n\varepsilon^2}.
\label{eq:operator-concentration-probability}
\end{equation}
Consequently, for every $0<\delta<1$, the condition
\begin{equation}
n
\ge
\frac{\kappa_d}{\delta\varepsilon^2}
\label{eq:operator-concentration-sample-size}
\end{equation}
implies
\begin{equation}
\|\mathcal{T}_n-\mathcal{T}\|_{\mathrm{op}}
\le
\varepsilon
\label{eq:operator-concentration-event}
\end{equation}
with probability at least $1-\delta$.
\end{theorem}

\begin{proof}
Let
\[
X_i:=x_ix_i^\top\in\mathbb{S}^d,
\]
and define the rank-one operator
\[
\mathscr{X}_i
:=
X_i\otimes X_i,
\qquad
\mathscr{X}_i(H)
:=
\langle X_i,H\rangle X_i.
\]
Then
\[
\mathcal{T}_n
=
\frac1n
\sum_{i=1}^{n}\mathscr{X}_i,
\qquad
\mathcal{T}
=
\mathbb{E}\mathscr{X}_1.
\]

Let $\|\cdot\|_{\mathrm{HS}}$ denote the Hilbert--Schmidt norm on
the finite-dimensional Hilbert space of linear operators on
$\mathbb{S}^d$. Independence and centering give
\begin{equation}
\begin{aligned}
\mathbb{E}
\|\mathcal{T}_n-\mathcal{T}\|_{\mathrm{HS}}^2
&=
\frac1n
\mathbb{E}
\|\mathscr{X}_1-\mathcal{T}\|_{\mathrm{HS}}^2
\\
&=
\frac1n
\left(
\mathbb{E}\|\mathscr{X}_1\|_{\mathrm{HS}}^2
-
\|\mathcal{T}\|_{\mathrm{HS}}^2
\right)
\\
&\le
\frac1n
\mathbb{E}
\|\mathscr{X}_1\|_{\mathrm{HS}}^2.
\label{eq:hs-variance-bound}
\end{aligned}
\end{equation}
For a rank-one operator $Z\otimes Z$,
\[
\|Z\otimes Z\|_{\mathrm{HS}}^2
=
\|Z\|_F^4.
\]
Since
\[
\|x_1x_1^\top\|_F
=
\|x_1\|_2^2,
\]
we obtain
\[
\|\mathscr{X}_1\|_{\mathrm{HS}}^2
=
\|x_1\|_2^8.
\]
For $x_1\sim N(0,I_d)$,
\begin{equation}
\mathbb{E}\|x_1\|_2^8
=
d(d+2)(d+4)(d+6)
=
\kappa_d.
\label{eq:eighth-gaussian-radius-moment}
\end{equation}
Thus
\begin{equation}
\mathbb{E}
\|\mathcal{T}_n-\mathcal{T}\|_{\mathrm{HS}}^2
\le
\frac{\kappa_d}{n}.
\label{eq:hs-mean-square-bound}
\end{equation}

Because
\[
\|\mathcal{B}\|_{\mathrm{op}}
\le
\|\mathcal{B}\|_{\mathrm{HS}},
\]
Markov's inequality gives
\[
\begin{aligned}
\mathbb{P}
\left(
\|\mathcal{T}_n-\mathcal{T}\|_{\mathrm{op}}
>
\varepsilon
\right)
&\le
\mathbb{P}
\left(
\|\mathcal{T}_n-\mathcal{T}\|_{\mathrm{HS}}^2
>
\varepsilon^2
\right)
\\
&\le
\frac{\kappa_d}{n\varepsilon^2}.
\end{aligned}
\]
\end{proof}

\begin{corollary}[Global empirical coercivity and smoothness]
\label{cor:global-empirical-coercivity}
On the event
\begin{equation}
\|\mathcal{T}_n-\mathcal{T}\|_{\mathrm{op}}
\le
\varepsilon,
\label{eq:generic-concentration-event}
\end{equation}
one has
\begin{equation}
(2-\varepsilon)\|H\|_F^2
\le
\langle H,\mathcal{T}_n(H)\rangle
\le
(d+2+\varepsilon)\|H\|_F^2
\label{eq:empirical-quadratic-form-bounds}
\end{equation}
for every $H\in\mathbb{S}^d$. Moreover,
\begin{equation}
\|\mathcal{T}_n(H)\|_F
\le
(d+2+\varepsilon)\|H\|_F.
\label{eq:empirical-operator-upper-bound}
\end{equation}
In particular, when $0<\varepsilon\le1$,
\begin{equation}
\|H\|_F^2
\le
\langle H,\mathcal{T}_n(H)\rangle
\le
(d+3)\|H\|_F^2.
\label{eq:unit-accuracy-empirical-bounds}
\end{equation}
\end{corollary}

\begin{proof}
By \cref{lem:gaussian-population-operator},
\[
2\|H\|_F^2
\le
\langle H,\mathcal{T}(H)\rangle
\le
(d+2)\|H\|_F^2.
\]
Furthermore,
\[
\begin{aligned}
\left|
\left\langle
H,
(\mathcal{T}_n-\mathcal{T})(H)
\right\rangle
\right|
&\le
\|H\|_F
\|(\mathcal{T}_n-\mathcal{T})(H)\|_F
\\
&\le
\varepsilon\|H\|_F^2.
\end{aligned}
\]
This proves \cref{eq:empirical-quadratic-form-bounds}.

Since $\mathcal{T}_n$ is self-adjoint and positive semidefinite,
\[
\|\mathcal{T}_n\|_{\mathrm{op}}
=
\sup_{\|H\|_F=1}
\langle H,\mathcal{T}_n(H)\rangle.
\]
The upper bound in
\cref{eq:empirical-quadratic-form-bounds} therefore gives
\cref{eq:empirical-operator-upper-bound}.
\end{proof}

Combining
\cref{thm:effective-hessian,cor:global-empirical-coercivity,prop:quotient-predictor-norm-comparison}
shows that, on the event
\cref{eq:generic-concentration-event},
\begin{equation}
\lambda_{\min}^{\mathrm{eff}}(Q_\star)
\ge
2(2-\varepsilon)\lambda_r(Q_\star),
\label{eq:empirical-effective-lower}
\end{equation}
and
\begin{equation}
\lambda_{\max}^{\mathrm{eff}}(Q_\star)
\le
4(d+2+\varepsilon)\lambda_1(Q_\star).
\label{eq:empirical-effective-upper}
\end{equation}

\subsection{Deterministic local quotient regularity}
\label{subsec:deterministic-local-regularity}

The effective Hessian describes infinitesimal curvature at
$Q_\star$. Controlling an optimization trajectory requires a
neighborhood-level secant estimate. We first establish such an
estimate for a general positive operator.

Let
\[
\mathscr{T}:
\mathbb{S}^d
\longrightarrow
\mathbb{S}^d
\]
be self-adjoint and positive semidefinite, and suppose that
\begin{equation}
m\|H\|_F^2
\le
\langle H,\mathscr{T}(H)\rangle
\le
M\|H\|_F^2
\qquad
\text{for every }H\in\mathbb{S}^d,
\label{eq:abstract-operator-bounds}
\end{equation}
where
$0<m\le M$.
Consider
\begin{equation}
L_{\mathscr{T}}(U)
:=
\frac12
\left\langle
UU^\top-Q_\star,
\mathscr{T}(UU^\top-Q_\star)
\right\rangle.
\label{eq:abstract-factor-loss}
\end{equation}

Set
\begin{equation}
\sigma_\star
:=
\sigma_r(U_\star)
=
\sqrt{\lambda_r(Q_\star)},
\label{eq:sigma-star}
\end{equation}
\begin{equation}
\beta_\star
:=
\|U_\star\|_{\mathrm{op}}
=
\sqrt{\lambda_1(Q_\star)},
\label{eq:beta-star}
\end{equation}
and
\begin{equation}
\rho_\star
:=
\frac{m\sigma_\star}{4M}.
\label{eq:abstract-basin-radius}
\end{equation}
Since $m\le M$,
\begin{equation}
\rho_\star
\le
\frac{\sigma_\star}{4}.
\label{eq:basin-radius-upper}
\end{equation}

\begin{lemma}[Local comparison of factor and predictor errors]
\label{lem:factor-predictor-local-comparison}
Suppose
\begin{equation}
d_{\mathrm P}(U,U_\star)
\le
\rho_\star.
\label{eq:local-factor-neighborhood}
\end{equation}
Let $R\in O(r)$ be a Procrustes minimizer and define
\begin{equation}
\widetilde U_\star
:=
U_\star R,
\qquad
\Delta
:=
U-\widetilde U_\star,
\qquad
E
:=
UU^\top-Q_\star.
\label{eq:aligned-factor-error}
\end{equation}
Then
$\widetilde U_\star^\top\Delta$
is symmetric, and
\begin{equation}
\sigma_\star\|\Delta\|_F
\le
\|E\|_F
\le
(2\beta_\star+\rho_\star)\|\Delta\|_F.
\label{eq:factor-predictor-error-comparison}
\end{equation}
\end{lemma}

\begin{proof}
By \cref{prop:procrustes-distance},
$\widetilde U_\star^\top U$
is symmetric positive semidefinite. Therefore,
\[
\widetilde U_\star^\top\Delta
=
\widetilde U_\star^\top U
-
\widetilde U_\star^\top\widetilde U_\star
\]
is symmetric. Hence $\Delta$ is horizontal at
$\widetilde U_\star$.

Expanding the predictor error gives
\begin{equation}
E
=
\widetilde U_\star\Delta^\top
+
\Delta\widetilde U_\star^\top
+
\Delta\Delta^\top.
\label{eq:predictor-error-expansion}
\end{equation}
Because $\Delta$ is horizontal at $\widetilde U_\star$,
\cref{prop:quotient-predictor-norm-comparison} gives
\begin{equation}
\left\|
\widetilde U_\star\Delta^\top
+
\Delta\widetilde U_\star^\top
\right\|_F
\ge
\sqrt{2}\,\sigma_\star\|\Delta\|_F.
\label{eq:linearized-predictor-lower}
\end{equation}
Also,
\[
\|\Delta\Delta^\top\|_F
\le
\|\Delta\|_F^2.
\]
Using
$\|\Delta\|_F\le\rho_\star\le\sigma_\star/4$,
we obtain
\[
\begin{aligned}
\|E\|_F
&\ge
\sqrt{2}\,\sigma_\star\|\Delta\|_F
-
\|\Delta\|_F^2
\\
&\ge
\left(
\sqrt{2}-\frac14
\right)
\sigma_\star\|\Delta\|_F
\\
&\ge
\sigma_\star\|\Delta\|_F.
\end{aligned}
\]
For the upper bound,
\[
\begin{aligned}
\|E\|_F
&\le
2\|\widetilde U_\star\|_{\mathrm{op}}
\|\Delta\|_F
+
\|\Delta\|_F^2
\\
&\le
(2\beta_\star+\rho_\star)\|\Delta\|_F.
\end{aligned}
\]
\end{proof}

\begin{theorem}[Local secant regularity]
\label{thm:local-secant-regularity}
Under \cref{eq:abstract-operator-bounds}, suppose
\[
d_{\mathrm P}(U,U_\star)
\le
\rho_\star.
\]
Let $\Delta$ and $E$ be defined by
\cref{eq:aligned-factor-error}. Then
\begin{equation}
\left\langle
\nabla L_{\mathscr{T}}(U),
\Delta
\right\rangle
\ge
\frac m2\|E\|_F^2
\ge
\frac{m\sigma_\star^2}{2}
\|\Delta\|_F^2.
\label{eq:local-secant-lower-bound}
\end{equation}
Moreover,
\begin{equation}
\|\nabla L_{\mathscr{T}}(U)\|_F
\le
L_\star\|\Delta\|_F,
\label{eq:local-gradient-upper-bound}
\end{equation}
where
\begin{equation}
L_\star
:=
2M
(2\beta_\star+\rho_\star)
(\beta_\star+\rho_\star).
\label{eq:local-gradient-constant}
\end{equation}
The loss satisfies
\begin{equation}
\frac m2\|E\|_F^2
\le
L_{\mathscr{T}}(U)
\le
\frac M2\|E\|_F^2.
\label{eq:local-loss-equivalence}
\end{equation}
\end{theorem}

\begin{proof}
The gradient of \cref{eq:abstract-factor-loss} is
\begin{equation}
\nabla L_{\mathscr{T}}(U)
=
2\mathscr{T}(E)U.
\label{eq:abstract-factor-gradient}
\end{equation}
Since
$U=\widetilde U_\star+\Delta$,
\begin{equation}
\Delta U^\top+U\Delta^\top
=
E+\Delta\Delta^\top.
\label{eq:secant-expansion}
\end{equation}
Therefore,
\begin{equation}
\begin{aligned}
\left\langle
\nabla L_{\mathscr{T}}(U),
\Delta
\right\rangle
&=
2\langle
\mathscr{T}(E)U,
\Delta
\rangle
\\
&=
\left\langle
\mathscr{T}(E),
\Delta U^\top+U\Delta^\top
\right\rangle
\\
&=
\langle\mathscr{T}(E),E\rangle
+
\left\langle
\mathscr{T}(E),
\Delta\Delta^\top
\right\rangle.
\label{eq:secant-inner-product-expansion}
\end{aligned}
\end{equation}
The first term satisfies
\begin{equation}
\langle\mathscr{T}(E),E\rangle
\ge
m\|E\|_F^2.
\label{eq:main-secant-term}
\end{equation}

Because $\mathscr{T}$ is self-adjoint and positive semidefinite,
\cref{eq:abstract-operator-bounds} implies
\[
\|\mathscr{T}\|_{\mathrm{op}}
=
\sup_{\|H\|_F=1}
\langle H,\mathscr{T}(H)\rangle
\le
M.
\]
Hence
\[
\begin{aligned}
\left|
\left\langle
\mathscr{T}(E),
\Delta\Delta^\top
\right\rangle
\right|
&\le
\|\mathscr{T}(E)\|_F
\|\Delta\Delta^\top\|_F
\\
&\le
M\|E\|_F\|\Delta\|_F^2.
\end{aligned}
\]
By \cref{lem:factor-predictor-local-comparison},
\[
\|\Delta\|_F
\le
\frac{\|E\|_F}{\sigma_\star}.
\]
Since $\|\Delta\|_F\le\rho_\star$,
\[
\|\Delta\|_F^2
\le
\frac{\rho_\star}{\sigma_\star}
\|E\|_F.
\]
Therefore,
\begin{equation}
\left|
\left\langle
\mathscr{T}(E),
\Delta\Delta^\top
\right\rangle
\right|
\le
\frac{M\rho_\star}{\sigma_\star}
\|E\|_F^2
=
\frac m4\|E\|_F^2.
\label{eq:secant-remainder-bound}
\end{equation}
Combining
\cref{eq:secant-inner-product-expansion,eq:main-secant-term,eq:secant-remainder-bound}
gives
\[
\left\langle
\nabla L_{\mathscr{T}}(U),
\Delta
\right\rangle
\ge
\frac{3m}{4}\|E\|_F^2
\ge
\frac m2\|E\|_F^2.
\]
The second inequality in
\cref{eq:local-secant-lower-bound} follows from
\cref{lem:factor-predictor-local-comparison}.

For the gradient bound,
\[
\begin{aligned}
\|\nabla L_{\mathscr{T}}(U)\|_F
&=
2\|\mathscr{T}(E)U\|_F
\\
&\le
2M\|E\|_F\|U\|_{\mathrm{op}}
\\
&\le
2M
(2\beta_\star+\rho_\star)
(\beta_\star+\rho_\star)
\|\Delta\|_F.
\end{aligned}
\]
This proves \cref{eq:local-gradient-upper-bound}.

Finally,
\[
L_{\mathscr{T}}(U)
=
\frac12
\langle E,\mathscr{T}(E)\rangle,
\]
so \cref{eq:local-loss-equivalence} follows directly from
\cref{eq:abstract-operator-bounds}.
\end{proof}

For the empirical operator, on the event
\begin{equation}
\|\mathcal{T}_n-\mathcal{T}\|_{\mathrm{op}}
\le
1,
\label{eq:unit-operator-event}
\end{equation}
\cref{cor:global-empirical-coercivity} permits the choices
\begin{equation}
m=1,
\qquad
M=d+3.
\label{eq:empirical-m-M}
\end{equation}
The corresponding deterministic basin radius is
\begin{equation}
\rho_n
:=
\frac{\sigma_\star}{4(d+3)}.
\label{eq:empirical-basin-radius}
\end{equation}

\subsection{Moment-based spectral initialization}
\label{subsec:spectral-initialization}

Define the empirical response mean
\begin{equation}
\overline y
:=
\frac1n
\sum_{i=1}^{n}y_i
\label{eq:empirical-response-mean}
\end{equation}
and the moment matrix
\begin{equation}
M_n
:=
\frac1{2n}
\sum_{i=1}^{n}
y_i
(x_ix_i^\top-I_d).
\label{eq:moment-matrix}
\end{equation}
By \cref{lem:gaussian-population-operator},
\begin{equation}
\mathbb{E}M_n
=
Q_\star.
\label{eq:moment-matrix-unbiased}
\end{equation}

For $M\in\mathbb{S}^d$, let
$\mathcal{P}_{r,+}(M)$
denote the matrix obtained by retaining the $r$ largest positive
eigenvalues of $M$ whenever $M$ has at least $r$ positive
eigenvalues, and setting all remaining eigenvalues to zero. Define
\begin{equation}
Q_0
:=
\mathcal{P}_{r,+}(M_n),
\label{eq:spectral-initial-predictor}
\end{equation}
and choose any factor
\begin{equation}
U_0\in\mathbb{R}^{d\times r}
\quad\text{such that}\quad
Q_0=U_0U_0^\top.
\label{eq:spectral-initial-factor}
\end{equation}
The event established below guarantees that $M_n$ has at least
$r$ positive eigenvalues, so this definition is unambiguous there.

\begin{lemma}[Response-mean concentration]
\label{lem:response-mean-concentration}
For every $\varepsilon>0$,
\begin{equation}
\mathbb{P}
\left(
\left|
\overline y-\operatorname{tr}(Q_\star)
\right|
>
\varepsilon\|Q_\star\|_F
\right)
\le
\frac{2}{n\varepsilon^2}.
\label{eq:response-mean-concentration}
\end{equation}
\end{lemma}

\begin{proof}
For $x\sim N(0,I_d)$,
\[
\mathbb{E}[x^\top Q_\star x]
=
\operatorname{tr}(Q_\star),
\]
and
\[
\operatorname{Var}(x^\top Q_\star x)
=
2\|Q_\star\|_F^2.
\]
Therefore,
\[
\operatorname{Var}(\overline y)
=
\frac{2}{n}\|Q_\star\|_F^2.
\]
Chebyshev's inequality proves
\cref{eq:response-mean-concentration}.
\end{proof}

\begin{lemma}[PSD truncation and factor perturbation]
\label{lem:psd-truncation-factor-perturbation}
Let $M\in\mathbb{S}^d$ satisfy
\begin{equation}
\|M-Q_\star\|_{\mathrm{op}}
\le
\gamma
<
\frac12\lambda_r(Q_\star).
\label{eq:moment-operator-error}
\end{equation}
Let
\[
Q_0
:=
\mathcal{P}_{r,+}(M).
\]
Then
\begin{equation}
Q_0\in\mathcal{M}_{r}^{+},
\label{eq:truncated-matrix-rank-r}
\end{equation}
\begin{equation}
\|Q_0-Q_\star\|_{\mathrm{op}}
\le
2\gamma,
\label{eq:truncated-operator-error}
\end{equation}
and, for arbitrary factors
$Q_0=U_0U_0^\top$
and
$Q_\star=U_\star U_\star^\top$,
\begin{equation}
d_{\mathrm P}(U_0,U_\star)
\le
2\sqrt{r\gamma}.
\label{eq:truncated-factor-error}
\end{equation}
\end{lemma}

\begin{proof}
Order the eigenvalues nonincreasingly. Weyl's inequality gives
\[
\lambda_r(M)
\ge
\lambda_r(Q_\star)-\gamma
>
0.
\]
Thus $M$ has at least $r$ positive eigenvalues. For every $j>r$,
\[
|\lambda_j(M)|
=
|\lambda_j(M)-\lambda_j(Q_\star)|
\le
\gamma,
\]
because $\lambda_j(Q_\star)=0$. Consequently,
$Q_0=\mathcal{P}_{r,+}(M)$ has rank exactly $r$, and all discarded
eigenvalues of $M$ have magnitude at most $\gamma$. Therefore,
\[
\|M-Q_0\|_{\mathrm{op}}
\le
\gamma.
\]
It follows that
\[
\begin{aligned}
\|Q_0-Q_\star\|_{\mathrm{op}}
&\le
\|Q_0-M\|_{\mathrm{op}}
+
\|M-Q_\star\|_{\mathrm{op}}
\\
&\le
2\gamma.
\end{aligned}
\]

The orthogonal Procrustes distance between PSD factors equals the
Bures--Wasserstein distance. The Powers--St{\o}rmer inequality gives
\begin{equation}
d_{\mathrm P}(U_0,U_\star)^2
\le
\|Q_0-Q_\star\|_*
\qquad
\text{\cite{PowersStormer1970}}.
\label{eq:powers-stormer-application}
\end{equation}
Since
\[
\operatorname{rank}(Q_0-Q_\star)
\le
2r,
\]
\cref{eq:truncated-operator-error} implies
\[
\begin{aligned}
\|Q_0-Q_\star\|_*
&\le
2r\|Q_0-Q_\star\|_{\mathrm{op}}
\\
&\le
4r\gamma.
\end{aligned}
\]
Taking square roots proves
\cref{eq:truncated-factor-error}.
\end{proof}

Define
\begin{equation}
\gamma_\star
:=
\frac{\lambda_r(Q_\star)}
{256r(d+3)^2}.
\label{eq:explicit-gamma-star}
\end{equation}
By \cref{eq:empirical-basin-radius},
\begin{equation}
\gamma_\star
=
\frac{\rho_n^2}{16r}.
\label{eq:gamma-star-basin-relation}
\end{equation}
Also define
\begin{equation}
\varepsilon_\star
:=
\frac{
\lambda_r(Q_\star)
}{
256r(d+3)^2\|Q_\star\|_F
}.
\label{eq:explicit-epsilon-star}
\end{equation}
Since
$0<\lambda_r(Q_\star)\le\|Q_\star\|_F$,
we have
\begin{equation}
0<\varepsilon_\star<1.
\label{eq:epsilon-star-less-than-one}
\end{equation}

\begin{theorem}[Initialization enters the quotient basin]
\label{thm:initialization-enters-basin}
Suppose
\begin{equation}
\|\mathcal{T}_n-\mathcal{T}\|_{\mathrm{op}}
\le
\varepsilon_\star
\label{eq:initialization-operator-event}
\end{equation}
and
\begin{equation}
\left|
\overline y-\operatorname{tr}(Q_\star)
\right|
\le
\varepsilon_\star\|Q_\star\|_F.
\label{eq:initialization-mean-event}
\end{equation}
Then
\begin{equation}
Q_0\in\mathcal{M}_{r}^{+},
\qquad
U_0\in\mathcal{F}_{d,r},
\label{eq:initialization-full-rank}
\end{equation}
and
\begin{equation}
d_{\mathrm P}(U_0,U_\star)
\le
\frac{\rho_n}{2}.
\label{eq:initialization-basin-bound}
\end{equation}
\end{theorem}

\begin{proof}
Using
\[
M_n
=
\frac12
\left[
\mathcal{T}_n(Q_\star)
-
\overline y I_d
\right]
\]
and
\[
Q_\star
=
\frac12
\left[
\mathcal{T}(Q_\star)
-
\operatorname{tr}(Q_\star)I_d
\right],
\]
we obtain
\begin{equation}
M_n-Q_\star
=
\frac12
(\mathcal{T}_n-\mathcal{T})(Q_\star)
-
\frac12
\bigl(
\overline y-\operatorname{tr}(Q_\star)
\bigr)I_d.
\label{eq:moment-error-decomposition}
\end{equation}
Therefore,
\[
\begin{aligned}
\|M_n-Q_\star\|_{\mathrm{op}}
&\le
\frac12
\|(\mathcal{T}_n-\mathcal{T})(Q_\star)\|_F
\\
&\quad+
\frac12
\left|
\overline y-\operatorname{tr}(Q_\star)
\right|
\\
&\le
\varepsilon_\star\|Q_\star\|_F
\\
&=
\gamma_\star.
\end{aligned}
\]
Moreover,
\[
\gamma_\star
=
\frac{\lambda_r(Q_\star)}
{256r(d+3)^2}
<
\frac12\lambda_r(Q_\star),
\]
so \cref{lem:psd-truncation-factor-perturbation} applies. It gives
\[
\begin{aligned}
d_{\mathrm P}(U_0,U_\star)
&\le
2\sqrt{r\gamma_\star}
\\
&=
\frac{\sqrt{\lambda_r(Q_\star)}}{8(d+3)}
\\
&=
\frac{\rho_n}{2}.
\end{aligned}
\]
The rank conclusions follow from the same lemma.
\end{proof}

\begin{corollary}[Explicit high-probability initialization event]
\label{cor:explicit-high-probability-initialization}
Let $0<\delta<1$. If
\begin{equation}
n
\ge
\frac{\kappa_d+2}{\delta}
\left(
\frac{
256r(d+3)^2\|Q_\star\|_F
}{
\lambda_r(Q_\star)
}
\right)^2,
\label{eq:explicit-initialization-sample-size}
\end{equation}
then, with probability at least $1-\delta$, the following
statements hold simultaneously:
\begin{equation}
\|\mathcal{T}_n-\mathcal{T}\|_{\mathrm{op}}
\le
\varepsilon_\star,
\label{eq:high-probability-operator-event}
\end{equation}
\begin{equation}
\left|
\overline y-\operatorname{tr}(Q_\star)
\right|
\le
\varepsilon_\star\|Q_\star\|_F,
\label{eq:high-probability-mean-event}
\end{equation}
and
\begin{equation}
d_{\mathrm P}(U_0,U_\star)
\le
\frac{\rho_n}{2}.
\label{eq:high-probability-initialization-bound}
\end{equation}
In particular, $Q_0\in\mathcal{M}_{r}^{+}$ and
$U_0\in\mathcal{F}_{d,r}$ on this event.
\end{corollary}

\begin{proof}
By
\cref{thm:full-space-operator-concentration,eq:explicit-epsilon-star},
\[
\mathbb{P}
\left(
\|\mathcal{T}_n-\mathcal{T}\|_{\mathrm{op}}
>
\varepsilon_\star
\right)
\le
\frac{\kappa_d}
{n\varepsilon_\star^2}.
\]
By \cref{lem:response-mean-concentration},
\[
\mathbb{P}
\left(
\left|
\overline y-\operatorname{tr}(Q_\star)
\right|
>
\varepsilon_\star\|Q_\star\|_F
\right)
\le
\frac{2}{n\varepsilon_\star^2}.
\]
Condition \cref{eq:explicit-initialization-sample-size} is equivalent
to
\[
n
\ge
\frac{\kappa_d+2}
{\delta\varepsilon_\star^2}.
\]
The union bound therefore proves
\cref{eq:high-probability-operator-event,eq:high-probability-mean-event}
simultaneously with probability at least $1-\delta$. On this event,
\cref{thm:initialization-enters-basin} gives
\cref{eq:high-probability-initialization-bound} and the rank
conclusions.
\end{proof}

\begin{remark}[Scale of the conservative sample bound]
\label{rem:conservative-sample-scale}
Since
\[
\kappa_d
=
d(d+2)(d+4)(d+6),
\]
the sufficient condition
\cref{eq:explicit-initialization-sample-size} has the leading-order
scale
\begin{equation}
n
=
O
\left(
\frac{
r^2d^8
}{
\delta
}
\left(
\frac{
\|Q_\star\|_F
}{
\lambda_r(Q_\star)
}
\right)^2
\right).
\label{eq:conservative-sample-asymptotics}
\end{equation}
This bound results from controlling
$\mathcal{T}_n-\mathcal{T}$ on the entire ambient space
$\mathbb{S}^d$ by a second-moment argument. It is not claimed to be
statistically optimal for low-rank recovery or to scale with the
intrinsic dimension of $\mathcal{M}_{r}^{+}$.
\end{remark}

\subsection{Convergence of parameter gradient flow}
\label{subsec:gradient-flow-convergence}

Define
\begin{equation}
\alpha_\star
:=
\frac{m\sigma_\star^2}{2}.
\label{eq:abstract-flow-rate}
\end{equation}

\begin{theorem}[Local exponential convergence of factor gradient flow]
\label{thm:local-gradient-flow-convergence}
Suppose
\begin{equation}
d_{\mathrm P}(U(0),U_\star)
\le
\rho_\star.
\label{eq:flow-initial-basin}
\end{equation}
Let $U(t)$ be the maximal solution of
\begin{equation}
\dot U(t)
=
-2\mathscr{T}
\bigl(
U(t)U(t)^\top-Q_\star
\bigr)U(t).
\label{eq:abstract-factor-gradient-flow}
\end{equation}
Then the solution exists for every $t\ge0$, remains in the
closed radius-$\rho_\star$ Procrustes ball, has full column rank,
and satisfies
\begin{equation}
d_{\mathrm P}(U(t),U_\star)
\le
e^{-\alpha_\star t}
d_{\mathrm P}(U(0),U_\star).
\label{eq:abstract-flow-factor-convergence}
\end{equation}
Moreover,
\begin{equation}
\begin{aligned}
\|U(t)U(t)^\top-Q_\star\|_F
&\le
\frac{
2\beta_\star+\rho_\star
}{
\sigma_\star
}
e^{-\alpha_\star t}
\\
&\quad\times
\|U(0)U(0)^\top-Q_\star\|_F.
\label{eq:abstract-flow-predictor-convergence}
\end{aligned}
\end{equation}
\end{theorem}

\begin{proof}
The vector field in
\cref{eq:abstract-factor-gradient-flow} is smooth, so a unique maximal
solution exists on some interval $[0,T_{\max})$.

Define
\[
\phi(t)
:=
\frac12
d_{\mathrm P}(U(t),U_\star)^2.
\]
For each $t<T_{\max}$, choose a Procrustes minimizer
$R_t\in O(r)$, and set
\[
\Delta_t
:=
U(t)-U_\star R_t.
\]
Using the same $R_t$ as a comparison at time $t+h$, one obtains
the upper right Dini derivative bound
\begin{equation}
D^+\phi(t)
\le
\langle\Delta_t,\dot U(t)\rangle.
\label{eq:dini-procrustes-derivative}
\end{equation}
As long as
$d_{\mathrm P}(U(t),U_\star)\le\rho_\star$,
\cref{thm:local-secant-regularity} gives
\[
\begin{aligned}
D^+\phi(t)
&\le
-
\left\langle
\nabla L_{\mathscr{T}}(U(t)),
\Delta_t
\right\rangle
\\
&\le
-\alpha_\star\|\Delta_t\|_F^2
\\
&=
-2\alpha_\star\phi(t).
\end{aligned}
\]
The comparison principle for upper Dini derivatives therefore yields
\[
\phi(t)
\le
e^{-2\alpha_\star t}\phi(0)
\]
up to the first possible exit time from the radius-$\rho_\star$
ball. In particular,
\[
d_{\mathrm P}(U(t),U_\star)
\le
d_{\mathrm P}(U(0),U_\star)
\le
\rho_\star,
\]
so such an exit cannot occur. Thus
\cref{eq:abstract-flow-factor-convergence} holds throughout the
maximal interval.

The same estimate gives
\[
\sigma_r(U(t))
\ge
\sigma_\star
-
d_{\mathrm P}(U(t),U_\star)
\ge
\sigma_\star-\rho_\star
\ge
\frac{3\sigma_\star}{4},
\]
so $U(t)$ has full column rank.

It remains to rule out finite-time blow-up. Along the gradient flow,
\[
\frac{d}{dt}L_{\mathscr{T}}(U(t))
=
-\|\nabla L_{\mathscr{T}}(U(t))\|_F^2
\le
0.
\]
Using \cref{eq:abstract-operator-bounds},
\[
\frac m2
\|U(t)U(t)^\top-Q_\star\|_F^2
\le
L_{\mathscr{T}}(U(0)).
\]
Hence $\|U(t)U(t)^\top\|_F$ is uniformly bounded on
$[0,T_{\max})$. Since
\[
\|U(t)\|_F^2
=
\operatorname{tr}(U(t)U(t)^\top)
\le
\sqrt{d}\,\|U(t)U(t)^\top\|_F,
\]
the factor $U(t)$ is uniformly bounded. Standard continuation for
smooth finite-dimensional ODEs then implies $T_{\max}=\infty$.

Finally, \cref{lem:factor-predictor-local-comparison} gives
\[
\|U(t)U(t)^\top-Q_\star\|_F
\le
(2\beta_\star+\rho_\star)
d_{\mathrm P}(U(t),U_\star),
\]
whereas
\[
d_{\mathrm P}(U(0),U_\star)
\le
\frac{
\|U(0)U(0)^\top-Q_\star\|_F
}{
\sigma_\star
}.
\]
Combining these inequalities with
\cref{eq:abstract-flow-factor-convergence} proves
\cref{eq:abstract-flow-predictor-convergence}.
\end{proof}

For the empirical loss, on the event
\cref{eq:unit-operator-event}, we have
\[
m=1,
\qquad
M=d+3,
\]
and hence
\[
\rho_\star=\rho_n,
\qquad
\alpha_\star
=
\frac12\lambda_r(Q_\star).
\]

\begin{corollary}[Gaussian-design gradient-flow recovery]
\label{cor:gaussian-gradient-flow-recovery}
Suppose
\[
\|\mathcal{T}_n-\mathcal{T}\|_{\mathrm{op}}
\le
1
\]
and
\[
d_{\mathrm P}(U_0,U_\star)
\le
\rho_n.
\]
Then the ordinary Euclidean factor gradient flow
\begin{equation}
\dot U(t)
=
-2\mathcal{T}_n
\bigl(
U(t)U(t)^\top-Q_\star
\bigr)U(t)
\label{eq:empirical-factor-gradient-flow}
\end{equation}
exists globally, remains full rank, and satisfies
\begin{equation}
d_{\mathrm P}(U(t),U_\star)
\le
e^{-\lambda_r(Q_\star)t/2}
d_{\mathrm P}(U_0,U_\star)
\label{eq:empirical-flow-factor-rate}
\end{equation}
for every $t\ge0$.
\end{corollary}

\begin{proof}
Apply \cref{thm:local-gradient-flow-convergence} with
$m=1$
and
$M=d+3$.
\end{proof}
\subsection{Convergence of parameter gradient descent}
\label{subsec:gradient-descent-convergence}

Consider the iteration
\begin{equation}
U_{k+1}
=
U_k
-
2\eta
\mathscr{T}
(U_kU_k^\top-Q_\star)U_k.
\label{eq:abstract-factor-gradient-descent}
\end{equation}

\begin{theorem}[Local linear convergence of factor gradient descent]
\label{thm:local-gradient-descent-convergence}
Let
\begin{equation}
L_\star
:=
2M
(2\beta_\star+\rho_\star)
(\beta_\star+\rho_\star),
\label{eq:discrete-local-L-star}
\end{equation}
and let
\[
\alpha_\star
=
\frac{m\sigma_\star^2}{2}.
\]
Suppose
\begin{equation}
d_{\mathrm P}(U_0,U_\star)
\le
\rho_\star
\label{eq:discrete-initial-basin}
\end{equation}
and
\begin{equation}
0<\eta
\le
\frac{\alpha_\star}{L_\star^2}.
\label{eq:abstract-oracle-step-size}
\end{equation}
Then every iterate belongs to $\mathcal{F}_{d,r}$, remains in the
closed radius-$\rho_\star$ Procrustes ball, and satisfies
\begin{equation}
d_{\mathrm P}(U_k,U_\star)^2
\le
(1-\eta\alpha_\star)^k
d_{\mathrm P}(U_0,U_\star)^2.
\label{eq:discrete-squared-convergence}
\end{equation}
Consequently,
\begin{equation}
d_{\mathrm P}(U_k,U_\star)
\le
\left(
1-\frac{\eta\alpha_\star}{2}
\right)^k
d_{\mathrm P}(U_0,U_\star),
\label{eq:discrete-factor-convergence}
\end{equation}
and
\begin{equation}
\begin{aligned}
\|U_kU_k^\top-Q_\star\|_F
&\le
\frac{
2\beta_\star+\rho_\star
}{
\sigma_\star
}
\left(
1-\frac{\eta\alpha_\star}{2}
\right)^k
\\
&\quad\times
\|U_0U_0^\top-Q_\star\|_F.
\label{eq:discrete-predictor-convergence}
\end{aligned}
\end{equation}
\end{theorem}

\begin{proof}
Assume inductively that
\[
d_{\mathrm P}(U_k,U_\star)
\le
\rho_\star.
\]
Choose a Procrustes minimizer $R_k\in O(r)$, and set
\[
\widetilde U_{\star,k}
:=
U_\star R_k,
\qquad
\Delta_k
:=
U_k-\widetilde U_{\star,k}.
\]
Using the same aligned target as a comparison for the next iterate,
\[
\begin{aligned}
d_{\mathrm P}(U_{k+1},U_\star)^2
&\le
\|U_{k+1}-\widetilde U_{\star,k}\|_F^2
\\
&=
\|\Delta_k-\eta\nabla L_{\mathscr{T}}(U_k)\|_F^2
\\
&=
\|\Delta_k\|_F^2
-
2\eta
\left\langle
\nabla L_{\mathscr{T}}(U_k),
\Delta_k
\right\rangle
\\
&\quad+
\eta^2
\|\nabla L_{\mathscr{T}}(U_k)\|_F^2.
\end{aligned}
\]
By \cref{thm:local-secant-regularity},
\[
\left\langle
\nabla L_{\mathscr{T}}(U_k),
\Delta_k
\right\rangle
\ge
\alpha_\star\|\Delta_k\|_F^2
\]
and
\[
\|\nabla L_{\mathscr{T}}(U_k)\|_F
\le
L_\star\|\Delta_k\|_F.
\]
Therefore,
\begin{equation}
\begin{aligned}
d_{\mathrm P}(U_{k+1},U_\star)^2
&\le
\left(
1
-
2\eta\alpha_\star
+
\eta^2L_\star^2
\right)
d_{\mathrm P}(U_k,U_\star)^2
\\
&\le
(1-\eta\alpha_\star)
d_{\mathrm P}(U_k,U_\star)^2,
\label{eq:one-step-distance-contraction}
\end{aligned}
\end{equation}
where the second inequality follows from
\cref{eq:abstract-oracle-step-size}. This proves
\cref{eq:discrete-squared-convergence} by induction and shows that all
iterates remain in the radius-$\rho_\star$ ball.

Because
$\rho_\star\le\sigma_\star/4$,
Weyl's singular-value inequality gives
\[
\begin{aligned}
\sigma_r(U_k)
&=
\sigma_r(U_kR_k^\top)
\\
&\ge
\sigma_r(U_\star)
-
\|U_kR_k^\top-U_\star\|_{\mathrm{op}}
\\
&\ge
\sigma_\star-\rho_\star
\\
&\ge
\frac{3\sigma_\star}{4}
>
0.
\end{aligned}
\]
Thus every $U_k$ has full column rank.

Furthermore,
\[
L_\star
\ge
4M\beta_\star^2
\ge
4m\sigma_\star^2
=
8\alpha_\star.
\]
Hence
\[
0
<
\eta\alpha_\star
\le
\frac{\alpha_\star^2}{L_\star^2}
\le
\frac1{64}
<
1.
\]
Using
\[
\sqrt{1-z}
\le
1-\frac z2,
\qquad
0\le z\le1,
\]
in \cref{eq:discrete-squared-convergence} proves
\cref{eq:discrete-factor-convergence}. The predictor estimate follows
from \cref{lem:factor-predictor-local-comparison}.
\end{proof}

For the empirical problem, define
\begin{equation}
\alpha_n
:=
\frac12\lambda_r(Q_\star),
\label{eq:empirical-alpha}
\end{equation}
and
\begin{equation}
L_n^\star
:=
2(d+3)
(2\beta_\star+\rho_n)
(\beta_\star+\rho_n).
\label{eq:empirical-L-star}
\end{equation}

\begin{corollary}[Gaussian-design recovery with an oracle step size]
\label{cor:gaussian-gradient-descent-recovery}
Suppose
\[
\|\mathcal{T}_n-\mathcal{T}\|_{\mathrm{op}}
\le
1
\]
and
\[
d_{\mathrm P}(U_0,U_\star)
\le
\rho_n.
\]
If
\begin{equation}
0<\eta
\le
\frac{\alpha_n}{(L_n^\star)^2},
\label{eq:empirical-oracle-step-size}
\end{equation}
then every iterate has full column rank and
\begin{equation}
d_{\mathrm P}(U_k,U_\star)
\le
\left(
1-
\frac{\eta\lambda_r(Q_\star)}4
\right)^k
d_{\mathrm P}(U_0,U_\star).
\label{eq:empirical-gradient-descent-rate}
\end{equation}
\end{corollary}

\begin{proof}
Apply \cref{thm:local-gradient-descent-convergence} with
\[
m=1,
\qquad
M=d+3.
\]
\end{proof}

\begin{remark}[Oracle nature of the step-size condition]
\label{rem:oracle-step-size}
The sufficient bound in
\cref{eq:empirical-oracle-step-size} depends on the unknown signal
quantities
$\lambda_r(Q_\star)$
and
$\lambda_1(Q_\star)$.
It is therefore an oracle step-size guarantee. No data-driven
learning-rate rule is claimed in this section.
\end{remark}

\subsection{End-to-end recovery}
\label{subsec:end-to-end-recovery}

\begin{theorem}[End-to-end recovery with a moment initializer and an oracle step size]
\label{thm:end-to-end-gaussian-recovery}
Let
$Q_\star\in\mathcal{M}_{r}^{+}$,
and define
\[
\sigma_\star
=
\sqrt{\lambda_r(Q_\star)},
\qquad
\beta_\star
=
\sqrt{\lambda_1(Q_\star)},
\]
\begin{equation}
\rho_n
=
\frac{
\sqrt{\lambda_r(Q_\star)}
}{
4(d+3)
},
\label{eq:end-to-end-rho}
\end{equation}
and
\begin{equation}
N_\star(\delta)
:=
\frac{\kappa_d+2}{\delta}
\left(
\frac{
256r(d+3)^2\|Q_\star\|_F
}{
\lambda_r(Q_\star)
}
\right)^2.
\label{eq:end-to-end-explicit-sample-bound}
\end{equation}
Let $0<\delta<1$, and suppose
\begin{equation}
n
\ge
N_\star(\delta).
\label{eq:end-to-end-sample-assumption}
\end{equation}
Construct $M_n$, $Q_0$, and $U_0$ according to
\cref{eq:moment-matrix,eq:spectral-initial-predictor,eq:spectral-initial-factor}.
Then, with probability at least $1-\delta$, the following
statements hold simultaneously.

\begin{enumerate}
\item
\emph{Full-space empirical regularity.}
\begin{equation}
\|\mathcal{T}_n-\mathcal{T}\|_{\mathrm{op}}
\le
\varepsilon_\star
<
1.
\label{eq:end-to-end-operator-event}
\end{equation}

\item
\emph{Effective curvature.}
For every nonzero
$\xi\in T_{Q_\star}\mathcal{M}_{r}^{+}$,
\begin{equation}
    2\lambda_r(Q_\star)
    \le
    \frac{
        \operatorname{Hess}_g
        \ell_n(Q_\star)[\xi,\xi]
    }{
        g_{Q_\star}(\xi,\xi)
    }
    \le
    4(d+3)\lambda_1(Q_\star).
    \label{eq:end-to-end-effective-curvature}
\end{equation}

\item
\emph{Spectral initialization.}
\begin{equation}
    Q_0\in\mathcal{M}_{r}^{+},
    \qquad
    U_0\in\mathcal{F}_{d,r},
    \label{eq:end-to-end-initial-rank}
\end{equation}
and
\begin{equation}
    d_{\mathrm P}(U_0,U_\star)
    \le
    \frac{\rho_n}{2}.
    \label{eq:end-to-end-initial-distance}
\end{equation}

\item
\emph{Gradient flow.}
The ordinary Euclidean factor gradient flow exists globally,
remains full rank, and satisfies
\begin{equation}
    d_{\mathrm P}(U(t),U_\star)
    \le
    e^{-\lambda_r(Q_\star)t/2}
    d_{\mathrm P}(U_0,U_\star)
    \label{eq:end-to-end-flow-factor-rate}
\end{equation}
for every $t\ge0$.

\item
\emph{Gradient descent with an oracle step size.}
If
\begin{equation}
    0<\eta
    \le
    \frac{
        \lambda_r(Q_\star)/2
    }{
        \left[
            2(d+3)
            (2\beta_\star+\rho_n)
            (\beta_\star+\rho_n)
        \right]^2
    },
    \label{eq:end-to-end-oracle-step-size}
\end{equation}
then every iterate has full column rank and
\begin{equation}
    d_{\mathrm P}(U_k,U_\star)
    \le
    \left(
        1-
        \frac{\eta\lambda_r(Q_\star)}4
    \right)^k
    d_{\mathrm P}(U_0,U_\star).
    \label{eq:end-to-end-gd-factor-rate}
\end{equation}

\item
\emph{Predictor recovery.}
Define
\begin{equation}
    C_\star
    :=
    \frac{
        2\beta_\star+\rho_n
    }{
        \sigma_\star
    }.
    \label{eq:end-to-end-predictor-constant}
\end{equation}
Then
\begin{equation}
\begin{aligned}
    \|U(t)U(t)^\top-Q_\star\|_F
    &\le
    C_\star
    e^{-\lambda_r(Q_\star)t/2}
    \\
    &\quad\times
    \|Q_0-Q_\star\|_F,
    \label{eq:end-to-end-flow-predictor-rate}
\end{aligned}
\end{equation}
and, under
\cref{eq:end-to-end-oracle-step-size},
\begin{equation}
\begin{aligned}
    \|U_kU_k^\top-Q_\star\|_F
    &\le
    C_\star
    \left(
        1-
        \frac{\eta\lambda_r(Q_\star)}4
    \right)^k
    \\
    &\quad\times
    \|Q_0-Q_\star\|_F.
    \label{eq:end-to-end-gd-predictor-rate}
\end{aligned}
\end{equation}

\end{enumerate}
\end{theorem}

\begin{proof}
By
\cref{eq:end-to-end-sample-assumption,cor:explicit-high-probability-initialization},
with probability at least $1-\delta$,
\begin{equation}
\|\mathcal{T}_n-\mathcal{T}\|_{\mathrm{op}}
\le
\varepsilon_\star
<
1,
\label{eq:end-to-end-proof-operator-event}
\end{equation}
\begin{equation}
\left|
\overline y-\operatorname{tr}(Q_\star)
\right|
\le
\varepsilon_\star\|Q_\star\|_F,
\label{eq:end-to-end-mean-event}
\end{equation}
and
\[
d_{\mathrm P}(U_0,U_\star)
\le
\frac{\rho_n}{2}.
\]
The same event also gives
$Q_0\in\mathcal{M}_{r}^{+}$ and
$U_0\in\mathcal{F}_{d,r}$.

By
\cref{thm:effective-hessian,cor:global-empirical-coercivity},
for every
$\xi\in T_{Q_\star}\mathcal{M}_{r}^{+}$,
\[
\begin{aligned}
\operatorname{Hess}_g
\ell_n(Q_\star)[\xi,\xi]
&=
\langle\xi,\mathcal{T}_n(\xi)\rangle
\\
&\ge
(2-\varepsilon_\star)\|\xi\|_F^2.
\end{aligned}
\]
Using
\cref{prop:quotient-predictor-norm-comparison} and
$\varepsilon_\star<1$,
\[
\begin{aligned}
\operatorname{Hess}_g
\ell_n(Q_\star)[\xi,\xi]
&\ge
2(2-\varepsilon_\star)
\lambda_r(Q_\star)
g_{Q_\star}(\xi,\xi)
\\
&\ge
2\lambda_r(Q_\star)
g_{Q_\star}(\xi,\xi).
\end{aligned}
\]
Similarly,
\[
\begin{aligned}
\operatorname{Hess}_g
\ell_n(Q_\star)[\xi,\xi]
&\le
(d+2+\varepsilon_\star)\|\xi\|_F^2
\\
&\le
4(d+3)\lambda_1(Q_\star)
g_{Q_\star}(\xi,\xi).
\end{aligned}
\]
This proves
\cref{eq:end-to-end-effective-curvature}.

Since
\[
d_{\mathrm P}(U_0,U_\star)
\le
\frac{\rho_n}{2}
<
\rho_n,
\]
\cref{cor:gaussian-gradient-flow-recovery} gives
\cref{eq:end-to-end-flow-factor-rate}, and
\cref{cor:gaussian-gradient-descent-recovery} gives
\cref{eq:end-to-end-gd-factor-rate} under
\cref{eq:end-to-end-oracle-step-size}.

Finally,
\cref{lem:factor-predictor-local-comparison} gives the predictor
bounds
\cref{eq:end-to-end-flow-predictor-rate,eq:end-to-end-gd-predictor-rate}.
\end{proof}

\begin{remark}[Interpretation of the recovery theorem]
\label{rem:effective-curvature-versus-trajectory}
The effective Hessian in
\cref{thm:effective-hessian} identifies the infinitesimal quotient
curvature at the interpolating solution. The convergence of the
actual factor trajectory is controlled by the neighborhood-level
secant inequality in \cref{thm:local-secant-regularity}, which is
derived from the same full-space empirical-operator bound. Thus the
gradient-flow and gradient-descent conclusions do not follow merely
from the Hessian spectrum at a single point.

The sample requirement in
\cref{thm:end-to-end-gaussian-recovery} is an explicit sufficient
condition obtained from full-space second-moment control. It is not a
tangent-space concentration result and is not asserted to have
near-optimal dependence on $d$, $r$, or $\delta$.
\end{remark}

\section{Trace Bias under Commuting Underdetermined Measurements}
\label{sec:commuting-trace-bias}

We study the interpolant selected by the PSD-factorized model when the
measurements do not uniquely identify the predictor. Throughout this
section, we set
\begin{equation}
    r=d,
    \label{eq:commuting-full-factor-width}
\end{equation}
so that the factor variable belongs to
\(\mathbb{R}^{d\times d}\).

Let
\[
    A_1,\ldots,A_n\in\mathbb{S}^d
\]
be pairwise commuting:
\begin{equation}
    A_iA_j=A_jA_i
    \qquad
    \text{for all }1\le i,j\le n.
    \label{eq:commuting-measurements}
\end{equation}
Since the matrices are real symmetric and commute, there is an
orthogonal decomposition
\begin{equation}
    \mathbb{R}^d
    =
    E_1\oplus\cdots\oplus E_m
    \label{eq:joint-eigenspace-decomposition}
\end{equation}
into maximal joint eigenspaces. Let \(P_a\) be the orthogonal
projector onto \(E_a\), and let
\begin{equation}
    d_a:=\operatorname{rank}(P_a).
    \label{eq:joint-eigenspace-multiplicity}
\end{equation}
Then
\begin{equation}
    P_aP_b=\delta_{ab}P_a,
    \qquad
    \sum_{a=1}^{m}P_a=I_d.
    \label{eq:joint-projector-relations}
\end{equation}
For every \(i\), there are scalars \(c_{ia}\in\mathbb{R}\) such that
\begin{equation}
    A_i
    =
    \sum_{a=1}^{m}c_{ia}P_a.
    \label{eq:measurement-joint-spectral-form}
\end{equation}
Maximality of the joint eigenspaces means that the vectors
\[
    (c_{1a},\ldots,c_{na})\in\mathbb{R}^n
\]
are distinct for distinct values of \(a\). Thus the projectors
\(P_a\) and multiplicities \(d_a\) are intrinsic, up to a permutation
of the index \(a\).

Define
\begin{equation}
    B\in\mathbb{R}^{n\times m},
    \qquad
    B_{ia}:=d_ac_{ia}.
    \label{eq:reduced-measurement-matrix}
\end{equation}
We assume
\begin{equation}
    1\le s:=\operatorname{rank}(B)<m.
    \label{eq:underdetermined-rank-condition}
\end{equation}
No measurement equation is deleted when the rows of \(B\) are
linearly dependent. Deleting or reweighting redundant equations
generally changes the squared loss and therefore changes the training
dynamics.

The factorized objective is
\begin{equation}
    L(U)
    :=
    \frac{1}{2n}
    \sum_{i=1}^{n}
    \bigl(
        \langle A_i,UU^\top\rangle-y_i
    \bigr)^2,
    \qquad
    U\in\mathbb{R}^{d\times d}.
    \label{eq:commuting-factor-loss}
\end{equation}

For \(q\in\mathbb{R}^m\), define
\begin{equation}
    Q(q)
    :=
    \sum_{a=1}^{m}q_aP_a.
    \label{eq:joint-spectral-predictor-map}
\end{equation}
Then
\begin{equation}
\begin{aligned}
    \langle A_i,Q(q)\rangle
    &=
    \sum_{a=1}^{m}
    c_{ia}q_a\operatorname{tr}(P_a)
    \\
    &=
    \sum_{a=1}^{m}d_ac_{ia}q_a
    \\
    &=
    [Bq]_i.
    \label{eq:reduced-measurement-identity}
\end{aligned}
\end{equation}

Define
\begin{equation}
    \mathcal{F}_{+}
    :=
    \left\{
        q\in\mathbb{R}_{+}^{m}:Bq=y
    \right\},
    \label{eq:closed-positive-feasible-set}
\end{equation}
and
\begin{equation}
    \mathcal{F}_{++}
    :=
    \left\{
        q\in\mathbb{R}_{++}^{m}:Bq=y
    \right\}.
    \label{eq:strict-positive-feasible-set}
\end{equation}
We assume
\begin{equation}
    \mathcal{F}_{++}\neq\varnothing.
    \label{eq:strict-feasibility-assumption}
\end{equation}
This implies \(y\in\operatorname{range}(B)\), but does not require
\(B\) to have full row rank. Moreover, because
\(\ker B\neq\{0\}\) and \(\mathcal{F}_{++}\neq\varnothing\), the
strictly positive feasible set contains infinitely many points.

Let
\begin{equation}
    \sigma_B^{+}
    :=
    \min
    \left\{
        \sigma_j(B):\sigma_j(B)>0
    \right\}
    \label{eq:smallest-positive-singular-value}
\end{equation}
denote the smallest positive singular value of \(B\).

We first allow an arbitrary strictly positive initialization
\begin{equation}
    q^0\in\mathbb{R}_{++}^{m},
    \qquad
    Q^0:=Q(q^0),
    \label{eq:general-positive-spectral-initialization}
\end{equation}
and use the canonical factor representative
\begin{equation}
    U^0
    :=
    Q(q^0)^{1/2}
    =
    \sum_{a=1}^{m}\sqrt{q_a^0}\,P_a.
    \label{eq:general-positive-factor-representative}
\end{equation}
The isotropic family
\(q_\varepsilon^0=\varepsilon^2\mathbf{1}\) is introduced later when
the small-initialization limit is studied.

\subsection{Intrinsic reduction to the joint spectral algebra}
\label{subsec:intrinsic-commuting-reduction}

Define the joint spectral algebra
\begin{equation}
    \mathscr{A}
    :=
    \left\{
        \sum_{a=1}^{m}z_aP_a:
        z\in\mathbb{R}^{m}
    \right\}.
    \label{eq:joint-spectral-algebra}
\end{equation}
Every matrix in \(\mathscr{A}\) acts as a scalar on each maximal
joint eigenspace \(E_a\). Its definition is therefore independent of
the choice of an orthonormal basis inside \(E_a\).

\begin{theorem}[Invariant joint-spectral reduction]
\label{thm:invariant-joint-spectral-reduction}
Let \(U:[0,T_{\max})\to\mathbb{R}^{d\times d}\) be the maximal
solution of the Euclidean parameter gradient flow
\begin{equation}
    \dot U(t)=-\nabla L(U(t)),
    \qquad
    U(0)=U^0,
    \label{eq:commuting-factor-gradient-flow}
\end{equation}
where \(0<T_{\max}\le\infty\). Then there exist unique functions
\[
    u_a:[0,T_{\max})\to\mathbb{R}_{++},
    \qquad
    1\le a\le m,
\]
such that
\begin{equation}
    U(t)
    =
    \sum_{a=1}^{m}u_a(t)P_a
    \label{eq:joint-spectral-factor-trajectory}
\end{equation}
for every \(t\in[0,T_{\max})\). Consequently,
\begin{equation}
    Q(t)
    :=
    U(t)U(t)^\top
    =
    \sum_{a=1}^{m}q_a(t)P_a,
    \qquad
    q_a(t):=u_a(t)^2.
    \label{eq:joint-spectral-predictor-trajectory}
\end{equation}

Define
\begin{equation}
    f(q)
    :=
    \frac{1}{2n}\|Bq-y\|_2^2
    \label{eq:reduced-quadratic-loss}
\end{equation}
and
\begin{equation}
    g(q)
    :=
    \nabla f(q)
    =
    \frac1nB^\top(Bq-y).
    \label{eq:reduced-quadratic-gradient}
\end{equation}
Then
\begin{equation}
    \dot u_a(t)
    =
    -\frac{2}{d_a}
    g_a(q(t))u_a(t),
    \label{eq:reduced-factor-flow}
\end{equation}
and
\begin{equation}
    \dot q_a(t)
    =
    -\frac{4q_a(t)}{d_a}
    g_a(q(t)),
    \qquad
    1\le a\le m.
    \label{eq:reduced-predictor-flow}
\end{equation}
In particular,
\begin{equation}
    u_a(t)
    =
    \sqrt{q_a^0}
    \exp
    \left(
        -\frac{2}{d_a}
        \int_0^t g_a(q(s))\,ds
    \right)
    >
    0
    \label{eq:continuous-factor-positivity}
\end{equation}
for every \(t<T_{\max}\), and hence
\begin{equation}
    q(t)\in\mathbb{R}_{++}^{m}
    \qquad
    \text{for every }t<T_{\max}.
    \label{eq:continuous-predictor-positivity}
\end{equation}

The same invariant reduction holds for parameter gradient descent
\begin{equation}
    U_{k+1}=U_k-\eta\nabla L(U_k).
    \label{eq:commuting-factor-gradient-descent}
\end{equation}
If
\[
    U_k
    =
    \sum_{a=1}^{m}u_{a,k}P_a,
    \qquad
    q_{a,k}:=u_{a,k}^2,
\]
then
\begin{equation}
    u_{a,k+1}
    =
    u_{a,k}
    \left(
        1-\frac{2\eta}{d_a}g_a(q_k)
    \right),
    \label{eq:reduced-factor-gradient-descent}
\end{equation}
and
\begin{equation}
    q_{a,k+1}
    =
    q_{a,k}
    \left(
        1-\frac{2\eta}{d_a}g_a(q_k)
    \right)^2.
    \label{eq:reduced-predictor-gradient-descent}
\end{equation}
\end{theorem}

\begin{proof}
Suppose
\[
    U=\sum_{a=1}^{m}u_aP_a.
\]
Then
\[
    Q=UU^\top
    =
    \sum_{a=1}^{m}u_a^2P_a.
\]
Writing \(q_a=u_a^2\), \cref{eq:reduced-measurement-identity}
gives
\[
    \langle A_i,Q\rangle=[Bq]_i.
\]

The predictor gradient is
\begin{equation}
    G(Q)
    :=
    \frac1n
    \sum_{i=1}^{n}
    \bigl(
        [Bq]_i-y_i
    \bigr)A_i.
    \label{eq:commuting-predictor-gradient}
\end{equation}
Using \cref{eq:measurement-joint-spectral-form},
\begin{equation}
    G(Q)
    =
    \sum_{a=1}^{m}\gamma_a(q)P_a,
    \label{eq:commuting-gradient-spectral-form}
\end{equation}
where
\begin{equation}
\begin{aligned}
    \gamma_a(q)
    &=
    \frac1n
    \sum_{i=1}^{n}
    c_{ia}
    \bigl(
        [Bq]_i-y_i
    \bigr)
    \\
    &=
    \frac{1}{nd_a}
    [B^\top(Bq-y)]_a
    \\
    &=
    \frac{1}{d_a}g_a(q).
    \label{eq:commuting-gradient-coefficient}
\end{aligned}
\end{equation}

By \cref{lem:factor-gradient},
\[
    \nabla L(U)=2G(Q)U.
\]
The projector relations in \cref{eq:joint-projector-relations} give
\begin{equation}
    \nabla L(U)
    =
    2
    \sum_{a=1}^{m}
    \frac{g_a(q)}{d_a}
    u_aP_a.
    \label{eq:commuting-factor-gradient-spectral-form}
\end{equation}
Thus the vector field is tangent to \(\mathscr{A}\). Since
\(U(0)=U^0\in\mathscr{A}\), uniqueness of the maximal solution
implies \cref{eq:joint-spectral-factor-trajectory} on
\([0,T_{\max})\).

\cref{eq:commuting-factor-gradient-spectral-form} gives
\cref{eq:reduced-factor-flow}. Since \(q_a=u_a^2\),
\[
    \dot q_a
    =
    2u_a\dot u_a
    =
    -\frac{4q_a}{d_a}g_a(q),
\]
which proves \cref{eq:reduced-predictor-flow}. Solving the scalar
linear equation for \(u_a\) gives
\cref{eq:continuous-factor-positivity}.

For gradient descent,
\[
\begin{aligned}
    U_{k+1}
    &=
    U_k
    -
    2\eta G(Q_k)U_k
    \\
    &=
    \sum_{a=1}^{m}
    u_{a,k}
    \left(
        1-\frac{2\eta}{d_a}g_a(q_k)
    \right)P_a.
\end{aligned}
\]
This proves \cref{eq:reduced-factor-gradient-descent}; squaring the
scalar coefficients gives
\cref{eq:reduced-predictor-gradient-descent}.
\end{proof}

\begin{remark}[Independence of the factor representative]
\label{rem:spectral-representative-independence}
The substantive initialization is the predictor
\(Q(0)=Q^0\), not the canonical factor in
\cref{eq:general-positive-factor-representative}. Because \(Q^0\) is
positive definite, every square factor
\(\widetilde U^0\) satisfying
\(\widetilde U^0(\widetilde U^0)^\top=Q^0\) has the form
\[
    \widetilde U^0=U^0R
    \qquad
    \text{for some }R\in O(d).
\]
Right-orthogonal equivariance of the factor dynamics gives
\[
    \widetilde U(t)=U(t)R
\]
on the common interval of existence, and therefore
\[
    \widetilde U(t)\widetilde U(t)^\top
    =
    U(t)U(t)^\top.
\]
Thus the predictor trajectory depends only on \(Q^0\).
\end{remark}

\subsection{Weighted mirror-flow structure}
\label{subsec:weighted-mirror-flow}

Define
\begin{equation}
    h(q)
    :=
    \frac14
    \sum_{a=1}^{m}
    d_a
    \bigl(
        q_a\log q_a-q_a
    \bigr),
    \qquad
    q\in\mathbb{R}_{++}^{m}.
    \label{eq:weighted-entropy-potential}
\end{equation}
Then
\begin{equation}
    [\nabla h(q)]_a
    =
    \frac{d_a}{4}\log q_a,
    \label{eq:weighted-entropy-gradient}
\end{equation}
and
\begin{equation}
    \nabla^2h(q)
    =
    \frac14
    \operatorname{diag}
    \left(
        \frac{d_1}{q_1},
        \ldots,
        \frac{d_m}{q_m}
    \right).
    \label{eq:weighted-entropy-hessian}
\end{equation}

\begin{proposition}[Exact weighted mirror flow]
\label{prop:exact-weighted-mirror-flow}
The reduced dynamics in \cref{eq:reduced-predictor-flow} satisfy
\begin{equation}
    \dot q
    =
    -
    \bigl(
        \nabla^2h(q)
    \bigr)^{-1}
    \nabla f(q).
    \label{eq:weighted-mirror-flow}
\end{equation}
Equivalently,
\begin{equation}
    \frac{d}{dt}\nabla h(q(t))
    =
    -\nabla f(q(t))
    =
    -\frac1nB^\top(Bq(t)-y).
    \label{eq:weighted-dual-flow}
\end{equation}
\end{proposition}

\begin{proof}
By \cref{eq:weighted-entropy-hessian},
\[
    \bigl(
        \nabla^2h(q)
    \bigr)^{-1}
    =
    4
    \operatorname{diag}
    \left(
        \frac{q_1}{d_1},
        \ldots,
        \frac{q_m}{d_m}
    \right).
\]
Substitution of \cref{eq:reduced-quadratic-gradient} gives
\cref{eq:reduced-predictor-flow}. Multiplying by
\(\nabla^2h(q)\) gives \cref{eq:weighted-dual-flow}.
\end{proof}

The Bregman divergence generated by \(h\) is
\begin{equation}
    D_h(p,q)
    :=
    h(p)-h(q)-\langle\nabla h(q),p-q\rangle.
    \label{eq:weighted-bregman-divergence-definition}
\end{equation}
For \(p,q\in\mathbb{R}_{++}^{m}\),
\begin{equation}
    D_h(p,q)
    =
    \frac14
    \sum_{a=1}^{m}
    d_a
    \left[
        p_a\log\frac{p_a}{q_a}
        -
        p_a
        +
        q_a
    \right].
    \label{eq:weighted-bregman-divergence-formula}
\end{equation}
Whenever the second argument is strictly positive, we use the same
notation for the continuous extension to
\(p\in\mathbb{R}_{+}^{m}\), with \(0\log0:=0\).

Fix an arbitrary strictly positive feasible point
\begin{equation}
    q^\dagger\in\mathcal{F}_{++}.
    \label{eq:fixed-strict-feasible-point}
\end{equation}

\begin{lemma}[Bregman Lyapunov identity]
\label{lem:bregman-lyapunov-identity}
Along the mirror flow, for every \(t<T_{\max}\),
\begin{equation}
    \frac{d}{dt}
    D_h(q^\dagger,q(t))
    =
    -\frac1n
    \|Bq(t)-y\|_2^2.
    \label{eq:bregman-lyapunov-identity}
\end{equation}
\end{lemma}

\begin{proof}
Differentiating with respect to the second argument gives
\[
\begin{aligned}
    \frac{d}{dt}D_h(q^\dagger,q(t))
    &=
    -
    \left\langle
        \nabla^2h(q(t))\dot q(t),
        q^\dagger-q(t)
    \right\rangle
    \\
    &=
    \left\langle
        \nabla f(q(t)),
        q^\dagger-q(t)
    \right\rangle.
\end{aligned}
\]
Since \(Bq^\dagger=y\),
\[
\begin{aligned}
    \left\langle
        \nabla f(q(t)),
        q^\dagger-q(t)
    \right\rangle
    &=
    \frac1n
    \left\langle
        Bq(t)-y,
        Bq^\dagger-Bq(t)
    \right\rangle
    \\
    &=
    -\frac1n
    \|Bq(t)-y\|_2^2.
\end{aligned}
\]
\end{proof}

\begin{lemma}[Compact reverse-divergence levels]
\label{lem:compact-reverse-bregman-levels}
For fixed \(p\in\mathbb{R}_{++}^{m}\) and \(c<\infty\), the set
\begin{equation}
    \mathcal{K}(p,c)
    :=
    \left\{
        q\in\mathbb{R}_{++}^{m}:
        D_h(p,q)\le c
    \right\}
    \label{eq:reverse-bregman-level-set}
\end{equation}
is compact as a subset of \(\mathbb{R}^{m}\) and is contained in
\(\mathbb{R}_{++}^{m}\). In particular, there are constants
\begin{equation}
    0<
    \underline q(p,c)
    \le
    \overline q(p,c)
    <\infty
    \label{eq:reverse-level-coordinate-bounds}
\end{equation}
such that
\begin{equation}
    \underline q(p,c)
    \le
    q_a
    \le
    \overline q(p,c)
    \label{eq:reverse-level-coordinate-bound}
\end{equation}
for every \(q\in\mathcal{K}(p,c)\) and every \(a\).
\end{lemma}

\begin{proof}
For fixed \(p_a>0\), define
\[
    \phi_{p_a}(x)
    :=
    p_a\log\frac{p_a}{x}-p_a+x,
    \qquad x>0.
\]
The function \(\phi_{p_a}\) is nonnegative and continuous, and
\[
    \phi_{p_a}(x)\longrightarrow+\infty
    \quad\text{as }x\downarrow0
    \quad\text{or as }x\uparrow\infty.
\]
Because
\[
    D_h(p,q)
    =
    \frac14\sum_{a=1}^{m}d_a\phi_{p_a}(q_a)
\]
and every summand is nonnegative, the inequality
\(D_h(p,q)\le c\) places each coordinate \(q_a\) in a compact
interval contained in \((0,\infty)\). Their Cartesian product is a
compact subset of \(\mathbb{R}_{++}^{m}\). The level set is closed
inside this product by continuity of \(D_h(p,\cdot)\), hence is
compact in \(\mathbb{R}^{m}\).
\end{proof}

\subsection{Global convergence and finite-initialization selection}
\label{subsec:finite-initialization-selection}

Define the Bregman projection of the positive initialization \(q^0\)
onto the closed feasible set by
\begin{equation}
    \Pi_h(q^0)
    :=
    \underset{q\in\mathcal{F}_{+}}{\arg\min}
    \;
    D_h(q,q^0).
    \label{eq:finite-initialization-bregman-projection}
\end{equation}

\begin{lemma}[Existence, uniqueness, and interiority of the Bregman projection]
\label{lem:bregman-projection-interiority}
Problem \cref{eq:finite-initialization-bregman-projection} has a
unique minimizer, and
\begin{equation}
    \Pi_h(q^0)\in\mathcal{F}_{++}.
    \label{eq:bregman-projection-strict-positive}
\end{equation}
\end{lemma}

\begin{proof}
The objective is
\[
    D_h(q,q^0)
    =
    \frac14
    \sum_{a=1}^{m}
    d_a
    \left[
        q_a\log\frac{q_a}{q_a^0}
        -
        q_a
        +
        q_a^0
    \right],
    \qquad q\in\mathbb{R}_{+}^{m},
\]
with \(0\log0:=0\). It is continuous on
\(\mathbb{R}_{+}^{m}\), coercive, and strictly convex. Since
\(\mathcal{F}_{+}\) is nonempty and closed, a unique minimizer exists.

Suppose that a minimizer \(\bar q\in\mathcal{F}_{+}\) has at least
one zero coordinate. Let \(q^\dagger\in\mathcal{F}_{++}\), and set
\[
    q(t)
    :=
    (1-t)\bar q+tq^\dagger,
    \qquad
    0<t<1.
\]
Then \(q(t)\in\mathcal{F}_{++}\). For every coordinate with
\(\bar q_a=0\), the corresponding directional quotient contains
\[
    \frac{d_aq_a^\dagger}{4}\log t,
\]
which tends to \(-\infty\) as \(t\downarrow0\). Coordinates with
\(\bar q_a>0\) contribute finite directional derivatives. Hence
\[
    D_h(q(t),q^0)
    <
    D_h(\bar q,q^0)
\]
for all sufficiently small \(t>0\), contradicting optimality.
Therefore the unique minimizer is strictly positive.
\end{proof}

\begin{theorem}[Global convergence to the Bregman projection]
\label{thm:global-convergence-bregman-projection}
The maximal solution in
\cref{thm:invariant-joint-spectral-reduction} is global:
\(T_{\max}=\infty\). The trajectory remains in a compact subset of
\(\mathbb{R}_{++}^{m}\), and
\begin{equation}
    q(t)\longrightarrow \Pi_h(q^0)
    \qquad
    \text{as }t\to\infty.
    \label{eq:continuous-convergence-to-bregman-projection}
\end{equation}

Define
\begin{equation}
    D^0
    :=
    D_h(q^\dagger,q^0),
    \label{eq:initial-reverse-divergence}
\end{equation}
\begin{equation}
    \underline q_0
    :=
    \underline q(q^\dagger,D^0),
    \label{eq:continuous-lower-coordinate-bound}
\end{equation}
and
\begin{equation}
    d_{\max}
    :=
    \max_{1\le a\le m}d_a.
    \label{eq:max-joint-multiplicity}
\end{equation}
Then
\begin{equation}
    \|Bq(t)-y\|_2
    \le
    \exp
    \left(
        -
        \frac{
            4\underline q_0
            (\sigma_B^{+})^2
        }{
            nd_{\max}
        }t
    \right)
    \|Bq^0-y\|_2.
    \label{eq:continuous-residual-rate}
\end{equation}
\end{theorem}

\begin{proof}
By \cref{lem:bregman-lyapunov-identity},
\[
    D_h(q^\dagger,q(t))
    \le
    D_h(q^\dagger,q^0)
    =
    D^0
\]
for every \(t<T_{\max}\). Hence the trajectory remains in the compact
set
\[
    \mathcal{K}(q^\dagger,D^0)
    \subset\mathbb{R}_{++}^{m}.
\]
The reduced vector field is smooth on a neighborhood of this compact
set. If \(T_{\max}<\infty\), the reduced solution extends beyond
\(T_{\max}\); reconstructing
\(U(t)=\sum_a\sqrt{q_a(t)}P_a\) then extends the original factor
solution as well, contradicting maximality. Hence
\(T_{\max}=\infty\).

Set
\begin{equation}
    r(t):=Bq(t)-y.
    \label{eq:continuous-residual}
\end{equation}
Since \(Bq^\dagger=y\),
\begin{equation}
    r(t)
    =
    B(q(t)-q^\dagger)
    \in
    \operatorname{range}(B).
    \label{eq:residual-in-range-B}
\end{equation}
Using \cref{eq:reduced-predictor-flow},
\begin{equation}
    \dot r(t)
    =
    -\frac4n
    B
    \operatorname{diag}
    \left(
        \frac{q_1(t)}{d_1},
        \ldots,
        \frac{q_m(t)}{d_m}
    \right)
    B^\top r(t).
    \label{eq:continuous-residual-dynamics}
\end{equation}
Therefore,
\[
\begin{aligned}
    \frac{d}{dt}
    \frac12\|r(t)\|_2^2
    &=
    -\frac4n
    r(t)^\top
    B
    \operatorname{diag}
    \left(
        \frac{q_a(t)}{d_a}
    \right)
    B^\top r(t)
    \\
    &\le
    -
    \frac{
        4\underline q_0
    }{
        nd_{\max}
    }
    \|B^\top r(t)\|_2^2.
\end{aligned}
\]
Because \(r(t)\in\operatorname{range}(B)\),
\begin{equation}
    \|B^\top r(t)\|_2
    \ge
    \sigma_B^{+}\|r(t)\|_2.
    \label{eq:positive-singular-value-residual-bound}
\end{equation}
Thus
\[
    \frac{d}{dt}
    \frac12\|r(t)\|_2^2
    \le
    -
    \frac{
        4\underline q_0
        (\sigma_B^{+})^2
    }{
        nd_{\max}
    }
    \|r(t)\|_2^2.
\]
Gronwall's inequality gives
\cref{eq:continuous-residual-rate}.

Since \(q(t)\) remains in a compact set,
\[
    \|\dot q(t)\|_2
    \le
    C_0\|r(t)\|_2
\]
for some finite constant \(C_0\). The residual estimate implies
\(\dot q\in L^1([0,\infty))\), so \(q(t)\) has finite length and
converges to a limit
\[
    q_\infty\in\mathbb{R}_{++}^{m}.
\]
The residual estimate gives \(Bq_\infty=y\).

Integrating \cref{eq:weighted-dual-flow} gives
\[
    \nabla h(q(t))
    -
    \nabla h(q^0)
    =
    -\frac1n
    B^\top
    \int_0^t
    (Bq(s)-y)\,ds.
\]
Passing to the limit yields
\begin{equation}
    \nabla h(q_\infty)
    -
    \nabla h(q^0)
    \in
    \operatorname{range}(B^\top).
    \label{eq:continuous-limit-kkt}
\end{equation}
Together with \(Bq_\infty=y\) and \(q_\infty>0\), this is the
first-order optimality condition for minimizing \(D_h(q,q^0)\) over
\(Bq=y\). By
\cref{lem:bregman-projection-interiority}, the minimizer is unique.
Hence \(q_\infty=\Pi_h(q^0)\).
\end{proof}

\subsection{Small initialization and global minimum-trace bias}
\label{subsec:small-initialization-trace-bias}

For \(0<\varepsilon<1\), define the isotropic initialization
\begin{equation}
    q_\varepsilon^0
    :=
    \varepsilon^2\mathbf{1},
    \label{eq:reduced-isotropic-initialization}
\end{equation}
Then
\begin{equation}
    Q(q_\varepsilon^0)
    =
    \varepsilon^2I_d,
    \label{eq:isotropic-predictor-initialization}
\end{equation}
and its canonical factor representative is
\begin{equation}
    U_\varepsilon^0
    =
    \varepsilon I_d.
    \label{eq:isotropic-factor-representative}
\end{equation}
Define the selected interpolant
\begin{equation}
    q_\varepsilon
    :=
    \Pi_h(q_\varepsilon^0).
    \label{eq:isotropic-bregman-selected-point}
\end{equation}
Equivalently,
\begin{equation}
    q_\varepsilon
    =
    \underset{\substack{Bq=y\\q\ge0}}{\arg\min}
    \;
    \sum_{a=1}^{m}
    d_a
    \left[
        q_a\log\frac{q_a}{\varepsilon^2}
        -
        q_a
        +
        \varepsilon^2
    \right],
    \label{eq:explicit-finite-initialization-selection}
\end{equation}
and the unique minimizer is strictly positive.

For \(q\in\mathbb{R}_{+}^{m}\), define
\begin{equation}
    \tau(q)
    :=
    \sum_{a=1}^{m}d_aq_a.
    \label{eq:weighted-trace-functional}
\end{equation}
By \cref{eq:joint-spectral-predictor-map},
\begin{equation}
    \operatorname{tr}(Q(q))
    =
    \tau(q).
    \label{eq:matrix-trace-reduced-trace}
\end{equation}
Let
\begin{equation}
    \tau_\star
    :=
    \min_{q\in\mathcal{F}_{+}}
    \tau(q),
    \label{eq:minimum-reduced-trace}
\end{equation}
and define
\begin{equation}
    \mathcal{S}_{\mathrm{tr}}
    :=
    \left\{
        q\in\mathcal{F}_{+}:
        \tau(q)=\tau_\star
    \right\}.
    \label{eq:minimum-trace-reduced-set}
\end{equation}
Because every \(d_a>0\), trace sublevel sets in
\(\mathbb{R}_{+}^{m}\) are compact. Thus
\(\mathcal{S}_{\mathrm{tr}}\) is nonempty, compact, and convex.

Define
\begin{equation}
    \psi(q)
    :=
    \sum_{a=1}^{m}
    d_a
    \bigl(
        q_a\log q_a-q_a
    \bigr),
    \qquad
    q\in\mathbb{R}_{+}^{m},
    \label{eq:weighted-entropy-secondary-functional}
\end{equation}
with \(0\log0:=0\). Then
\begin{equation}
    4D_h(q,\varepsilon^2\mathbf{1})
    =
    a_\varepsilon\tau(q)
    +
    \psi(q)
    +
    d\varepsilon^2,
    \label{eq:bregman-small-initialization-expansion}
\end{equation}
where
\begin{equation}
    a_\varepsilon
    :=
    \log\frac1{\varepsilon^2},
    \qquad
    d=\sum_{a=1}^{m}d_a.
    \label{eq:small-initialization-leading-coefficient}
\end{equation}

\begin{theorem}[Minimum-trace limit]
\label{thm:minimum-trace-limit}
As \(\varepsilon\downarrow0\),
\begin{equation}
    \operatorname{dist}
    \bigl(
        q_\varepsilon,
        \mathcal{S}_{\mathrm{tr}}
    \bigr)
    \longrightarrow0.
    \label{eq:reduced-minimum-trace-convergence}
\end{equation}
Moreover, for every fixed \(q^\star\in\mathcal{S}_{\mathrm{tr}}\),
\begin{equation}
    0
    \le
    \tau(q_\varepsilon)-\tau_\star
    \le
    \frac{
        \psi(q^\star)+d
    }{
        \log(1/\varepsilon^2)
    }.
    \label{eq:minimum-trace-gap-rate}
\end{equation}
\end{theorem}

\begin{proof}
By optimality of \(q_\varepsilon\) and
\cref{eq:bregman-small-initialization-expansion},
\[
    a_\varepsilon\tau(q_\varepsilon)
    +
    \psi(q_\varepsilon)
    \le
    a_\varepsilon\tau_\star
    +
    \psi(q^\star).
\]
For every \(x\ge0\),
\[
    x\log x-x\ge-1.
\]
Therefore,
\begin{equation}
    \psi(q)\ge-d
    \qquad
    \text{for every }q\in\mathbb{R}_{+}^{m}.
    \label{eq:entropy-functional-lower-bound}
\end{equation}
Since \(\tau(q_\varepsilon)\ge\tau_\star\),
\[
    0
    \le
    a_\varepsilon
    \bigl(
        \tau(q_\varepsilon)-\tau_\star
    \bigr)
    \le
    \psi(q^\star)-\psi(q_\varepsilon)
    \le
    \psi(q^\star)+d.
\]
This proves \cref{eq:minimum-trace-gap-rate}.

For sufficiently small \(\varepsilon\), the same estimate places
\(q_\varepsilon\) in a fixed compact trace sublevel. Let
\(\varepsilon_k\downarrow0\), and choose a convergent subsequence
\[
    q_{\varepsilon_{k_\ell}}
    \longrightarrow
    \bar q.
\]
Because \(\mathcal{F}_{+}\) is closed, \(\bar q\in\mathcal{F}_{+}\).
The trace-gap estimate gives \(\tau(\bar q)=\tau_\star\), so
\(\bar q\in\mathcal{S}_{\mathrm{tr}}\). Hence every cluster point
belongs to \(\mathcal{S}_{\mathrm{tr}}\), which proves
\cref{eq:reduced-minimum-trace-convergence}.
\end{proof}

The reduced minimum-trace problem is equivalent to the
minimum-trace problem over all positive-semidefinite interpolants, not
merely over the invariant algebra.

\begin{proposition}[Equivalence with global PSD trace minimization]
\label{prop:global-psd-trace-equivalence}
Define
\begin{equation}
    \mathcal{Q}_{+}
    :=
    \left\{
        Q\in\mathbb{S}_{+}^{d}:
        \langle A_i,Q\rangle=y_i
        \text{ for every }i
    \right\}.
    \label{eq:global-psd-feasible-set}
\end{equation}
Then
\begin{equation}
    \min_{Q\in\mathcal{Q}_{+}}
    \operatorname{tr}(Q)
    =
    \min_{q\in\mathcal{F}_{+}}
    \tau(q)
    =
    \tau_\star.
    \label{eq:global-reduced-trace-equivalence}
\end{equation}

More precisely, for every \(Q\in\mathcal{Q}_{+}\), define
\begin{equation}
    q_a(Q)
    :=
    \frac{
        \operatorname{tr}(P_aQ)
    }{
        d_a
    }.
    \label{eq:block-average-eigenvalue}
\end{equation}
Then \(q(Q)\in\mathcal{F}_{+}\),
\begin{equation}
    \operatorname{tr}(Q)
    =
    \tau(q(Q)),
    \label{eq:trace-preserved-by-block-averaging}
\end{equation}
and \(Q(q(Q))\in\mathcal{Q}_{+}\) has the same measurements and trace
as \(Q\). Consequently, every matrix
\[
    Q(q),
    \qquad
    q\in\mathcal{S}_{\mathrm{tr}},
\]
is a globally minimum-trace PSD interpolant.
\end{proposition}

\begin{proof}
Let \(Q\in\mathcal{Q}_{+}\). Since
\(P_aQP_a\succeq0\),
\[
    \operatorname{tr}(P_aQ)
    =
    \operatorname{tr}(P_aQP_a)
    \ge0.
\]
Thus \(q_a(Q)\ge0\). Using
\cref{eq:measurement-joint-spectral-form},
\[
\begin{aligned}
    \langle A_i,Q\rangle
    &=
    \sum_{a=1}^{m}
    c_{ia}\operatorname{tr}(P_aQ)
    \\
    &=
    \sum_{a=1}^{m}
    d_ac_{ia}q_a(Q)
    \\
    &=
    [Bq(Q)]_i.
\end{aligned}
\]
Since \(Q\) is feasible, \(Bq(Q)=y\), so
\(q(Q)\in\mathcal{F}_{+}\). Moreover,
\[
\begin{aligned}
    \operatorname{tr}(Q)
    &=
    \sum_{a=1}^{m}\operatorname{tr}(P_aQ)
    \\
    &=
    \sum_{a=1}^{m}d_aq_a(Q)
    \\
    &=
    \tau(q(Q)).
\end{aligned}
\]
The matrix
\[
    Q(q(Q))
    =
    \sum_{a=1}^{m}q_a(Q)P_a
\]
is positive semidefinite, has measurements \(Bq(Q)=y\), and has the
same trace as \(Q\). Conversely, every
\(q\in\mathcal{F}_{+}\) produces a feasible PSD matrix \(Q(q)\) with
trace \(\tau(q)\). Hence the minimum values coincide, and the final
statement follows.
\end{proof}

Define
\begin{equation}
    \mathcal{Q}_{\mathrm{tr}}
    :=
    \underset{Q\in\mathcal{Q}_{+}}{\arg\min}
    \operatorname{tr}(Q).
    \label{eq:global-minimum-trace-set}
\end{equation}
The set \(\mathcal{Q}_{\mathrm{tr}}\) is nonempty and compact because
\(\mathcal{Q}_{+}\neq\varnothing\) and PSD trace sublevel sets are
compact.

\begin{corollary}[Global minimum-trace bias]
\label{cor:global-minimum-trace-bias}
Let
\begin{equation}
    Q_\varepsilon
    :=
    Q(q_\varepsilon).
    \label{eq:finite-initialization-selected-matrix}
\end{equation}
Then
\begin{equation}
    \operatorname{dist}_F
    \bigl(
        Q_\varepsilon,
        \mathcal{Q}_{\mathrm{tr}}
    \bigr)
    \longrightarrow0
    \qquad
    \text{as }\varepsilon\downarrow0.
    \label{eq:global-minimum-trace-matrix-convergence}
\end{equation}
\end{corollary}

\begin{proof}
For \(p,q\in\mathbb{R}^{m}\),
\[
    \|Q(p)-Q(q)\|_F^2
    =
    \sum_{a=1}^{m}d_a(p_a-q_a)^2
    \le
    d_{\max}\|p-q\|_2^2.
\]
By \cref{thm:minimum-trace-limit}, the distance from
\(q_\varepsilon\) to \(\mathcal{S}_{\mathrm{tr}}\) tends to zero.
By \cref{prop:global-psd-trace-equivalence},
\[
    Q(\mathcal{S}_{\mathrm{tr}})
    \subseteq
    \mathcal{Q}_{\mathrm{tr}}.
\]
The result follows.
\end{proof}

\subsection{Entropic tie-breaking inside the invariant algebra}
\label{subsec:entropic-tie-breaking}

Since \(x\mapsto x\log x\) is strictly convex on
\(\mathbb{R}_{+}\), the problem
\begin{equation}
    q_{\mathrm{ent}}
    :=
    \underset{q\in\mathcal{S}_{\mathrm{tr}}}{\arg\min}
    \;
    \sum_{a=1}^{m}
    d_aq_a\log q_a
    \label{eq:entropy-tie-breaking-problem}
\end{equation}
has a unique solution.

\begin{theorem}[Entropic selection among minimum-trace solutions]
\label{thm:entropic-tie-breaking}
The full family \(q_\varepsilon\) converges:
\begin{equation}
    q_\varepsilon
    \longrightarrow
    q_{\mathrm{ent}}
    \qquad
    \text{as }\varepsilon\downarrow0.
    \label{eq:entropy-selected-limit}
\end{equation}
Consequently,
\begin{equation}
    Q_\varepsilon
    \longrightarrow
    Q_{\mathrm{ent}}
    :=
    Q(q_{\mathrm{ent}}).
    \label{eq:entropy-selected-matrix-limit}
\end{equation}
The matrix \(Q_{\mathrm{ent}}\) is a globally minimum-trace PSD
interpolant.

Within the invariant minimum-trace family
\[
    Q(\mathcal{S}_{\mathrm{tr}})
    =
    \left\{
        Q(q):q\in\mathcal{S}_{\mathrm{tr}}
    \right\},
\]
\(Q_{\mathrm{ent}}\) is equivalently the unique minimizer of
\begin{equation}
    Q
    \longmapsto
    \operatorname{tr}(Q\log Q),
    \label{eq:matrix-entropy-secondary-objective}
\end{equation}
with \(0\log0:=0\). If \(\tau_\star>0\), it is equivalently the
unique maximizer of the Shannon entropy of the normalized eigenvalue
distribution among matrices in
\(Q(\mathcal{S}_{\mathrm{tr}})\).
\end{theorem}

\begin{proof}
Fix \(q\in\mathcal{S}_{\mathrm{tr}}\). By optimality of
\(q_\varepsilon\),
\[
    a_\varepsilon\tau(q_\varepsilon)
    +
    \psi(q_\varepsilon)
    \le
    a_\varepsilon\tau_\star
    +
    \psi(q).
\]
Since \(\tau(q_\varepsilon)\ge\tau_\star\),
\begin{equation}
    \psi(q_\varepsilon)
    \le
    \psi(q)
    \qquad
    \text{for every }q\in\mathcal{S}_{\mathrm{tr}}.
    \label{eq:secondary-functional-comparison}
\end{equation}

By \cref{thm:minimum-trace-limit}, every cluster point of
\(q_\varepsilon\) belongs to \(\mathcal{S}_{\mathrm{tr}}\). Let
\[
    q_{\varepsilon_k}\longrightarrow\bar q
\]
along a convergent subsequence. Passing to the limit in
\cref{eq:secondary-functional-comparison} gives
\[
    \psi(\bar q)
    \le
    \psi(q)
    \qquad
    \text{for every }q\in\mathcal{S}_{\mathrm{tr}}.
\]
On \(\mathcal{S}_{\mathrm{tr}}\), the trace is constant, so minimizing
\(\psi\) is equivalent to minimizing the objective in
\cref{eq:entropy-tie-breaking-problem}. Strict convexity gives the
unique minimizer \(q_{\mathrm{ent}}\). Hence every cluster point is
\(q_{\mathrm{ent}}\), which proves
\cref{eq:entropy-selected-limit}. Continuity of \(Q(\cdot)\) gives
\cref{eq:entropy-selected-matrix-limit}.

By \cref{prop:global-psd-trace-equivalence},
\(Q_{\mathrm{ent}}\in\mathcal{Q}_{\mathrm{tr}}\). Functional calculus
gives
\[
    \operatorname{tr}(Q(q)\log Q(q))
    =
    \sum_{a=1}^{m}d_aq_a\log q_a.
\]
If \(\tau_\star>0\), the normalized eigenvalues of \(Q(q)\) are
\(q_a/\tau_\star\), with multiplicity \(d_a\). Their Shannon entropy
is
\[
    -
    \sum_{a=1}^{m}
    d_a
    \frac{q_a}{\tau_\star}
    \log
    \left(
        \frac{q_a}{\tau_\star}
    \right).
\]
Since \(\tau_\star\) is fixed, maximizing this entropy is equivalent
to minimizing
\(\sum_ad_aq_a\log q_a\).
\end{proof}

\begin{remark}[Scope of the entropy statement]
\label{rem:entropy-scope}
The minimum-trace conclusion in
\cref{cor:global-minimum-trace-bias} is global over all PSD
interpolants. The entropy tie-breaking result in
\cref{thm:entropic-tie-breaking} is proved only within the invariant
joint spectral algebra \(\mathscr{A}\). No claim is made that
\(Q_{\mathrm{ent}}\) maximizes von Neumann entropy over the entire
global set \(\mathcal{Q}_{\mathrm{tr}}\).
\end{remark}

\subsection{Finite-step parameter gradient descent}
\label{subsec:commuting-finite-step-gradient-descent}

We now fix the arbitrary positive initialization \(q^0\) from
\cref{eq:general-positive-spectral-initialization} and compare
finite-step parameter gradient descent with the continuous mirror-flow
selection \(\Pi_h(q^0)\).

Define
\begin{equation}
    D_{\mathrm{disc}}^0
    :=
    D_h(q^\dagger,q^0),
    \label{eq:discrete-initial-reverse-divergence}
\end{equation}
and
\begin{equation}
    \mathcal{K}_0^{-}
    :=
    \left\{
        q\in\mathbb{R}_{++}^{m}:
        D_h(q^\dagger,q)\le D_{\mathrm{disc}}^0
    \right\}.
    \label{eq:discrete-reverse-level-set}
\end{equation}
By \cref{lem:compact-reverse-bregman-levels},
\(\mathcal{K}_0^{-}\) is compact. Define
\begin{equation}
    \underline q_0
    :=
    \min_{\substack{
        q\in\mathcal{K}_0^{-}\\
        1\le a\le m
    }}
    q_a,
    \qquad
    \overline q_0
    :=
    \max_{\substack{
        q\in\mathcal{K}_0^{-}\\
        1\le a\le m
    }}
    q_a,
    \label{eq:discrete-coordinate-bounds}
\end{equation}
and
\begin{equation}
    G_0
    :=
    \max_{q\in\mathcal{K}_0^{-}}
    \|g(q)\|_\infty.
    \label{eq:discrete-gradient-bound}
\end{equation}
Also set
\begin{equation}
    d_{\min}
    :=
    \min_{1\le a\le m}d_a,
    \qquad
    d_{\max}
    :=
    \max_{1\le a\le m}d_a,
    \label{eq:min-max-multiplicity}
\end{equation}
and
\begin{equation}
    C_0
    :=
    \max_{1\le a\le m}
    \frac{
        2q_a^\dagger+\overline q_0
    }{
        d_a
    }.
    \label{eq:discrete-bregman-remainder-constant}
\end{equation}

Define
\begin{equation}
\begin{aligned}
    \eta_0
    :=
    \min
    \Biggl\{
        &
        \frac{d_{\min}}{4G_0},
        \;
        \frac{
            n
        }{
            2C_0\|B\|_{\mathrm{op}}^2
        },
        \\
        &
        \frac{
            2n\underline q_0d_{\min}^2
        }{
            25\overline q_0^2
            d_{\max}\|B\|_{\mathrm{op}}^2
        },
        \;
        \frac{
            nd_{\max}
        }{
            4\underline q_0
            (\sigma_B^{+})^2
        }
    \Biggr\},
    \label{eq:fixed-epsilon-step-size-threshold}
\end{aligned}
\end{equation}
where the first term is interpreted as \(+\infty\) if \(G_0=0\).

\begin{theorem}[Convergence of finite-step parameter gradient descent]
\label{thm:finite-step-gradient-descent-convergence}
Let
\begin{equation}
    0<\eta\le\eta_0.
    \label{eq:finite-step-size-assumption}
\end{equation}
Initialize
\[
    u_{a,0}=\sqrt{q_a^0},
    \qquad
    q_{a,0}=q_a^0,
\]
and apply
\cref{eq:reduced-factor-gradient-descent,eq:reduced-predictor-gradient-descent}.
Then:

\begin{enumerate}
    \item
    \(u_{a,k}>0\) for every \(a\) and every \(k\);

    \item
    \(q_k\in\mathcal{K}_0^{-}\) for every \(k\);

    \item
    the reverse Bregman divergence satisfies
    \begin{equation}
        D_h(q^\dagger,q_{k+1})
        \le
        D_h(q^\dagger,q_k)
        -
        \frac{\eta}{2n}
        \|Bq_k-y\|_2^2;
        \label{eq:discrete-bregman-descent}
    \end{equation}

    \item
    the reduced loss contracts linearly:
    \begin{equation}
        f(q_{k+1})
        \le
        \left(
            1-
            \frac{
                4\eta\underline q_0
                (\sigma_B^{+})^2
            }{
                nd_{\max}
            }
        \right)
        f(q_k);
        \label{eq:discrete-loss-linear-contraction}
    \end{equation}

    \item
    \(q_k\) converges to a strictly positive interpolant
    \begin{equation}
        q_\eta
        \in
        \mathcal{F}_{++}.
        \label{eq:finite-step-limit}
    \end{equation}
\end{enumerate}
\end{theorem}

\begin{proof}
Assume inductively that \(q_k\in\mathcal{K}_0^{-}\). Set
\begin{equation}
    g_k:=g(q_k),
    \qquad
    z_{a,k}
    :=
    \frac{2\eta}{d_a}g_{a,k}.
    \label{eq:discrete-z-definition}
\end{equation}
The first bound in
\cref{eq:fixed-epsilon-step-size-threshold} gives
\begin{equation}
    |z_{a,k}|\le\frac12.
    \label{eq:discrete-z-bound}
\end{equation}
Therefore \(1-z_{a,k}>0\), and
\cref{eq:reduced-factor-gradient-descent} implies
\(u_{a,k+1}>0\).

Using
\[
    q_{a,k+1}
    =
    q_{a,k}(1-z_{a,k})^2,
\]
we obtain
\begin{equation}
\begin{aligned}
    4
    \bigl[
        D_h(q^\dagger,q_{k+1})
        -
        D_h(q^\dagger,q_k)
    \bigr]
    &=
    \sum_{a=1}^{m}
    d_a
    \Bigl[
        -2q_a^\dagger\log(1-z_{a,k})
        \\
        &\qquad
        +
        q_{a,k}
        (-2z_{a,k}+z_{a,k}^2)
    \Bigr].
    \label{eq:exact-discrete-bregman-increment}
\end{aligned}
\end{equation}
For \(|z|\le1/2\),
\begin{equation}
    -\log(1-z)
    \le
    z+z^2.
    \label{eq:logarithm-upper-bound}
\end{equation}
Substituting
\cref{eq:logarithm-upper-bound,eq:discrete-z-definition} into
\cref{eq:exact-discrete-bregman-increment} gives
\begin{equation}
\begin{aligned}
    D_h(q^\dagger,q_{k+1})
    -
    D_h(q^\dagger,q_k)
    &\le
    \eta
    \langle q^\dagger-q_k,g_k\rangle
    \\
    &\quad+
    \eta^2
    \sum_{a=1}^{m}
    \frac{
        2q_a^\dagger+q_{a,k}
    }{
        d_a
    }
    g_{a,k}^2.
    \label{eq:discrete-bregman-increment-bound}
\end{aligned}
\end{equation}
Since \(Bq^\dagger=y\),
\begin{equation}
    \langle q^\dagger-q_k,g_k\rangle
    =
    -\frac1n
    \|Bq_k-y\|_2^2.
    \label{eq:discrete-feasible-inner-product}
\end{equation}
Also,
\begin{equation}
    \|g_k\|_2^2
    \le
    \frac{
        \|B\|_{\mathrm{op}}^2
    }{
        n^2
    }
    \|Bq_k-y\|_2^2.
    \label{eq:gradient-residual-upper-bound}
\end{equation}
The second bound in
\cref{eq:fixed-epsilon-step-size-threshold} gives
\cref{eq:discrete-bregman-descent}. Hence
\(q_{k+1}\in\mathcal{K}_0^{-}\), closing the induction.

Let
\begin{equation}
    \Delta q_k:=q_{k+1}-q_k.
    \label{eq:discrete-q-increment}
\end{equation}
Then
\begin{equation}
    [\Delta q_k]_a
    =
    -\frac{4\eta q_{a,k}}{d_a}g_{a,k}
    +
    \frac{4\eta^2q_{a,k}}{d_a^2}g_{a,k}^2.
    \label{eq:exact-discrete-q-increment}
\end{equation}
Since
\(
\eta|g_{a,k}|/d_a\le1/4
\),
\begin{equation}
\begin{aligned}
    \langle g_k,\Delta q_k\rangle
    &\le
    -3\eta
    \sum_{a=1}^{m}
    \frac{q_{a,k}}{d_a}g_{a,k}^2
    \\
    &\le
    -
    \frac{
        3\eta\underline q_0
    }{
        d_{\max}
    }
    \|g_k\|_2^2.
    \label{eq:discrete-gradient-increment-inner-product}
\end{aligned}
\end{equation}
Furthermore,
\begin{equation}
    \|\Delta q_k\|_2
    \le
    \frac{
        5\eta\overline q_0
    }{
        d_{\min}
    }
    \|g_k\|_2.
    \label{eq:discrete-q-increment-bound}
\end{equation}

Since \(f\) is quadratic,
\begin{equation}
    f(q_{k+1})
    =
    f(q_k)
    +
    \langle g_k,\Delta q_k\rangle
    +
    \frac1{2n}
    \|B\Delta q_k\|_2^2.
    \label{eq:exact-discrete-loss-expansion}
\end{equation}
The third bound in
\cref{eq:fixed-epsilon-step-size-threshold}, together with
\cref{eq:discrete-gradient-increment-inner-product,eq:discrete-q-increment-bound},
implies
\begin{equation}
    f(q_{k+1})
    \le
    f(q_k)
    -
    \frac{
        2\eta\underline q_0
    }{
        d_{\max}
    }
    \|g_k\|_2^2.
    \label{eq:discrete-loss-descent-gradient}
\end{equation}

Set \(r_k:=Bq_k-y\). Since \(Bq^\dagger=y\),
\(r_k\in\operatorname{range}(B)\). Therefore,
\begin{equation}
\begin{aligned}
    \|g_k\|_2^2
    &=
    \frac1{n^2}
    \|B^\top r_k\|_2^2
    \\
    &\ge
    \frac{
        (\sigma_B^{+})^2
    }{
        n^2
    }
    \|r_k\|_2^2
    \\
    &=
    \frac{
        2(\sigma_B^{+})^2
    }{
        n
    }
    f(q_k).
    \label{eq:reduced-loss-PL-inequality}
\end{aligned}
\end{equation}
Combining
\cref{eq:discrete-loss-descent-gradient,eq:reduced-loss-PL-inequality}
gives \cref{eq:discrete-loss-linear-contraction}. The fourth bound in
\cref{eq:fixed-epsilon-step-size-threshold} ensures that the
contraction factor is nonnegative.

The contraction gives geometric decay of \(f(q_k)\). By
\cref{eq:discrete-q-increment-bound} and
\cref{eq:gradient-residual-upper-bound},
\[
    \sum_{k=0}^{\infty}
    \|\Delta q_k\|_2
    <
    \infty.
\]
Thus \(q_k\) converges to some
\(q_\eta\in\mathcal{K}_0^{-}\). The compact set
\(\mathcal{K}_0^{-}\) is bounded away from the boundary of the
positive orthant, so \(q_\eta>0\). Finally,
\cref{eq:discrete-loss-linear-contraction} implies
\(Bq_\eta=y\).
\end{proof}

\subsection{Finite-step selection error and the sequential limit}
\label{subsec:finite-step-selection-error}

Define the forward-divergence level set
\begin{equation}
    \mathcal{K}_0^{+}
    :=
    \left\{
        q\in\mathbb{R}_{+}^{m}:
        D_h(q,q^0)
        \le
        D_h(q^\dagger,q^0)
    \right\}.
    \label{eq:forward-bregman-level-set}
\end{equation}
This set is compact. By optimality of the continuous Bregman
projection,
\[
    \Pi_h(q^0)\in\mathcal{K}_0^{+},
\]
whereas
\cref{thm:finite-step-gradient-descent-convergence} gives
\(q_\eta\in\mathcal{K}_0^{-}\). Define
\begin{equation}
    \widehat q_0
    :=
    \max
    \left\{
        \max_{q\in\mathcal{K}_0^{-}}\|q\|_\infty,
        \max_{q\in\mathcal{K}_0^{+}}\|q\|_\infty
    \right\}.
    \label{eq:unified-coordinate-upper-bound}
\end{equation}
Then
\begin{equation}
    \|q_\eta\|_\infty
    \le
    \widehat q_0,
    \qquad
    \|\Pi_h(q^0)\|_\infty
    \le
    \widehat q_0.
    \label{eq:both-limits-coordinate-bound}
\end{equation}

\begin{theorem}[First-order finite-step selection error]
\label{thm:first-order-discrete-selection-error}
Under the assumptions of
\cref{thm:finite-step-gradient-descent-convergence},
\begin{equation}
    \|q_\eta-\Pi_h(q^0)\|_2
    \le
    \frac{
        4\widehat q_0d_{\max}
    }{
        d_{\min}^{2}\underline q_0
    }
    f(q^0)\eta.
    \label{eq:first-order-discrete-selection-error}
\end{equation}
In particular, for every fixed \(q^0\in\mathbb{R}_{++}^{m}\),
\begin{equation}
    q_\eta
    \longrightarrow
    \Pi_h(q^0)
    \qquad
    \text{as }\eta\downarrow0,
    \label{eq:fixed-epsilon-discrete-to-continuous-limit}
\end{equation}
where \(0<\eta\le\eta_0\).
\end{theorem}

\begin{proof}
Write \(g_k:=g(q_k)\). By
\cref{eq:reduced-predictor-gradient-descent},
\[
    q_{a,k+1}
    =
    q_{a,k}(1-z_{a,k})^2,
\]
where \(z_{a,k}\) is defined in
\cref{eq:discrete-z-definition}. Since
\[
    [\nabla h(q)]_a
    =
    \frac{d_a}{4}\log q_a,
\]
we have
\begin{equation}
\begin{aligned}
    [\nabla h(q_{k+1})-\nabla h(q_k)]_a
    &=
    \frac{d_a}{4}
    \log
    \left(
        \frac{q_{a,k+1}}{q_{a,k}}
    \right)
    \\
    &=
    \frac{d_a}{2}
    \log(1-z_{a,k}).
    \label{eq:discrete-dual-increment-exact}
\end{aligned}
\end{equation}
Define \(e_k\in\mathbb{R}^{m}\) by
\begin{equation}
    \nabla h(q_{k+1})-\nabla h(q_k)
    =
    -\eta g_k+e_k.
    \label{eq:discrete-dual-increment-error}
\end{equation}
Since
\(
\eta g_{a,k}=d_az_{a,k}/2
\),
\begin{equation}
    [e_k]_a
    =
    \frac{d_a}{2}
    \bigl(
        \log(1-z_{a,k})+z_{a,k}
    \bigr).
    \label{eq:dual-error-coordinate}
\end{equation}
For \(|z|\le1/2\),
\begin{equation}
    |\log(1-z)+z|
    \le
    z^2.
    \label{eq:logarithm-remainder-bound}
\end{equation}
Consequently,
\[
    |[e_k]_a|
    \le
    \frac{2\eta^2}{d_a}g_{a,k}^2
    \le
    \frac{2\eta^2}{d_{\min}}g_{a,k}^2,
\]
and hence
\begin{equation}
    \|e_k\|_2
    \le
    \frac{2\eta^2}{d_{\min}}
    \|g_k\|_2^2.
    \label{eq:dual-error-one-step-bound}
\end{equation}

Summing \cref{eq:discrete-loss-descent-gradient} gives
\begin{equation}
    \frac{
        2\eta\underline q_0
    }{
        d_{\max}
    }
    \sum_{k=0}^{\infty}
    \|g_k\|_2^2
    \le
    f(q^0).
    \label{eq:summed-gradient-square-bound}
\end{equation}
Combining this with
\cref{eq:dual-error-one-step-bound} yields
\begin{equation}
    \sum_{k=0}^{\infty}\|e_k\|_2
    \le
    \frac{
        \eta d_{\max}
    }{
        d_{\min}\underline q_0
    }
    f(q^0).
    \label{eq:absolute-summability-dual-error}
\end{equation}
Thus
\begin{equation}
    e_\eta
    :=
    \sum_{k=0}^{\infty}e_k
    \label{eq:accumulated-dual-error-definition}
\end{equation}
is well defined and satisfies
\begin{equation}
    \|e_\eta\|_2
    \le
    \frac{
        \eta d_{\max}
    }{
        d_{\min}\underline q_0
    }
    f(q^0).
    \label{eq:discrete-limit-kkt-error}
\end{equation}

For \(N\ge1\), summing
\cref{eq:discrete-dual-increment-error} gives
\begin{equation}
    \nabla h(q_N)-\nabla h(q^0)
    =
    w_N+e_N,
    \label{eq:finite-discrete-dual-sum}
\end{equation}
where
\begin{equation}
    w_N
    :=
    -\eta
    \sum_{k=0}^{N-1}g_k
    \in\operatorname{range}(B^\top),
    \qquad
    e_N
    :=
    \sum_{k=0}^{N-1}e_k.
    \label{eq:finite-range-error-decomposition}
\end{equation}
Since \(q_N\to q_\eta\) and \(e_N\to e_\eta\), the sequence
\(w_N\) converges to
\[
    w_\eta
    :=
    \nabla h(q_\eta)
    -
    \nabla h(q^0)
    -
    e_\eta.
\]
The subspace \(\operatorname{range}(B^\top)\) is closed, so
\begin{equation}
    w_\eta\in\operatorname{range}(B^\top).
    \label{eq:discrete-limit-range-component}
\end{equation}
Equivalently,
\begin{equation}
    \nabla h(q_\eta)-\nabla h(q^0)
    =
    w_\eta+e_\eta.
    \label{eq:discrete-limit-approximate-kkt}
\end{equation}

The continuous Bregman projection is strictly positive and satisfies
\begin{equation}
    \nabla h(\Pi_h(q^0))
    -
    \nabla h(q^0)
    =
    w^0
    \qquad
    \text{for some }
    w^0\in\operatorname{range}(B^\top).
    \label{eq:continuous-bregman-kkt}
\end{equation}
Both \(q_\eta\) and \(\Pi_h(q^0)\) are feasible, so
\begin{equation}
    v_\eta
    :=
    q_\eta-\Pi_h(q^0)
    \in\ker B.
    \label{eq:selection-difference-kernel}
\end{equation}
Subtracting \cref{eq:continuous-bregman-kkt} from
\cref{eq:discrete-limit-approximate-kkt}, taking the inner product
with \(v_\eta\), and using
\(\ker B=\operatorname{range}(B^\top)^\perp\), gives
\begin{equation}
    \left\langle
        v_\eta,
        \nabla h(q_\eta)-\nabla h(\Pi_h(q^0))
    \right\rangle
    =
    \langle v_\eta,e_\eta\rangle.
    \label{eq:selection-error-monotonicity-identity}
\end{equation}

By \cref{eq:both-limits-coordinate-bound}, every coordinate on the
line segment joining \(q_\eta\) and \(\Pi_h(q^0)\) is at most
\(\widehat q_0\). Therefore,
\begin{equation}
    \nabla^2h(q)
    \succeq
    \frac{
        d_{\min}
    }{
        4\widehat q_0
    }I_m
    \label{eq:segment-strong-convexity}
\end{equation}
along that segment. Hence
\begin{equation}
    \left\langle
        v_\eta,
        \nabla h(q_\eta)-\nabla h(\Pi_h(q^0))
    \right\rangle
    \ge
    \frac{
        d_{\min}
    }{
        4\widehat q_0
    }
    \|v_\eta\|_2^2.
    \label{eq:selection-strong-monotonicity}
\end{equation}
Combining
\cref{eq:selection-error-monotonicity-identity,eq:selection-strong-monotonicity}
with Cauchy--Schwarz gives
\[
    \frac{
        d_{\min}
    }{
        4\widehat q_0
    }
    \|v_\eta\|_2^2
    \le
    \|v_\eta\|_2\|e_\eta\|_2.
\]
If \(v_\eta=0\), the result is immediate. Otherwise, cancel one
factor of \(\|v_\eta\|_2\) and use
\cref{eq:discrete-limit-kkt-error} to obtain
\cref{eq:first-order-discrete-selection-error}.
\end{proof}

\begin{corollary}[Sequential small-step and small-initialization limit]
\label{cor:sequential-discrete-implicit-bias}
For each \(\varepsilon\in(0,1)\), apply
\cref{thm:finite-step-gradient-descent-convergence} with
\(q^0=q_\varepsilon^0\), and denote its step-size threshold by
\(\eta_\varepsilon\) and its limiting interpolant by
\(q_{\varepsilon,\eta}\). Then
\begin{equation}
    \lim_{\eta\downarrow0}
    q_{\varepsilon,\eta}
    =
    q_\varepsilon,
    \label{eq:fixed-epsilon-small-step-limit}
\end{equation}
where \(0<\eta\le\eta_\varepsilon\). Consequently,
\begin{equation}
    \lim_{\varepsilon\downarrow0}
    \lim_{\eta\downarrow0}
    q_{\varepsilon,\eta}
    =
    q_{\mathrm{ent}}.
    \label{eq:sequential-reduced-limit}
\end{equation}
Equivalently, with
\begin{equation}
    Q_{\varepsilon,\eta}
    :=
    Q(q_{\varepsilon,\eta}),
    \label{eq:discrete-selected-matrix}
\end{equation}
one has
\begin{equation}
    \lim_{\varepsilon\downarrow0}
    \lim_{\eta\downarrow0}
    Q_{\varepsilon,\eta}
    =
    Q_{\mathrm{ent}}.
    \label{eq:sequential-matrix-limit}
\end{equation}
In particular, the sequential limit is a globally minimum-trace PSD
interpolant.
\end{corollary}

\begin{proof}
For fixed \(\varepsilon\), all constants in
\cref{eq:first-order-discrete-selection-error} are finite and
independent of \(\eta\). Thus
\[
    \|q_{\varepsilon,\eta}-q_\varepsilon\|_2
    \longrightarrow0
    \qquad
    \text{as }\eta\downarrow0.
\]
This proves \cref{eq:fixed-epsilon-small-step-limit}. Combining it
with \cref{thm:entropic-tie-breaking} gives
\cref{eq:sequential-reduced-limit}. Continuity of \(Q(\cdot)\) gives
\cref{eq:sequential-matrix-limit}.
\end{proof}

\begin{remark}[No explicit joint scaling claim]
\label{rem:no-explicit-joint-scaling}
The constants entering the finite-step theorem, including
\[
    \underline q_0,
    \qquad
    \overline q_0,
    \qquad
    \widehat q_0,
    \qquad
    \eta_0,
\]
depend on the fixed initialization \(q^0\). Under the isotropic
specialization \(q^0=q_\varepsilon^0\), they may deteriorate as
\(\varepsilon\downarrow0\). Accordingly,
\cref{cor:sequential-discrete-implicit-bias} establishes only the
sequential limit
\[
    \eta\downarrow0
    \quad\text{followed by}\quad
    \varepsilon\downarrow0.
\]
No explicit simultaneous schedule
\(\eta=\eta(\varepsilon)\) is claimed.
\end{remark}

\section{Numerical Experiments}
\label{sec:numerical-experiments}

We now evaluate the main structural, optimization, and implicit-bias
predictions developed in Sections~2--4.  The experiments are organized
to follow the theoretical progression of the paper:
\[
    \text{exact quotient dynamics}
    \;\longrightarrow\;
    \text{effective curvature and recovery}
    \;\longrightarrow\;
    \text{Bregman, trace, and entropy selection}.
\]
All computations are implemented in PyTorch using double precision.
Unless stated otherwise, the predictor associated with a factor
\(U\in\mathbb{R}^{d\times r}\) is
\[
    Q(U)=UU^\top,
\]
and parameter-space errors are measured by the orthogonal Procrustes
distance
\[
    d_{\mathrm P}(U,V)
    :=
    \min_{R\in O(r)}
    \|U-VR\|_F.
\]
The continuous-time dynamics are integrated by a fourth-order
Runge--Kutta method, while the discrete experiments use the exact factor
gradient-descent recursion.  The figures report the raw numerical
values rather than normalized quantities chosen to enforce agreement
with the theory.

\subsection{Quotient invariance and exact predictor dynamics}
\label{subsec:experiment-quotient-dynamics}

We first test the exact identities of Section~\ref{sec:quotient-dynamics}.  We use
\[
    d=20,
    \qquad
    r=5,
    \qquad
    n=80,
\]
and generate five orthogonally equivalent initial factors
\[
    U_0^{(j)}=U_0R_j,
    \qquad
    R_j\in O(r).
\]
All representatives are trained using the same data and step size.
Although their factor trajectories differ, the quotient theory predicts
that their predictor trajectories coincide.

For a reference trajectory \(U_k^{(0)}\), define the representative
discrepancy
\begin{equation}
    E_{\mathrm{inv}}(k)
    :=
    \max_j
    \frac{
        \|U_k^{(j)}(U_k^{(j)})^\top
        -
        U_k^{(0)}(U_k^{(0)})^\top\|_F
    }{
        \|U_k^{(0)}(U_k^{(0)})^\top\|_F
    }.
    \label{eq:experimental-representative-discrepancy}
\end{equation}
The largest discrepancy over the complete trajectories is
\[
    \max_k E_{\mathrm{inv}}(k)
    =
    3.11\times10^{-15},
\]
which is at the level of double-precision roundoff.

We next test the exact predictor recursion of
\cref{prop:exact-discrete-predictor-dynamics}.  If
\[
    U_{k+1}
    =
    U_k-2\eta G_kU_k,
    \qquad
    G_k=\nabla_Q\ell(Q_k),
\]
then the predictor satisfies
\[
    Q_{k+1}
    =
    (I_d-2\eta G_k)
    Q_k
    (I_d-2\eta G_k).
\]
The relative recurrence residual
\begin{equation}
    E_{\mathrm{rec}}(k)
    :=
    \frac{
        \left\|
            Q_{k+1}
            -
            (I_d-2\eta G_k)Q_k(I_d-2\eta G_k)
        \right\|_F
    }{
        \|Q_{k+1}\|_F
    }
    \label{eq:experimental-recurrence-residual}
\end{equation}
has maximum value
\[
    \max_k E_{\mathrm{rec}}(k)
    =
    4.96\times10^{-16}.
\]

Finally, expanding the exact recursion gives
\[
    \frac{Q_{k+1}-Q_k}{\eta}
    +
    2(G_kQ_k+Q_kG_k)
    =
    4\eta G_kQ_kG_k.
\]
Hence the discrepancy between the finite-step predictor update and the
quotient-gradient direction must be linear in \(\eta\).  A log--log
regression gives the empirical slope
\[
    0.9984,
\]
in agreement with the exact \(O(\eta)\) correction.

\begin{figure}[H]
    \centering
    \begin{subfigure}[t]{0.32\textwidth}
        \centering
        \includegraphics[
            width=\linewidth
        ]{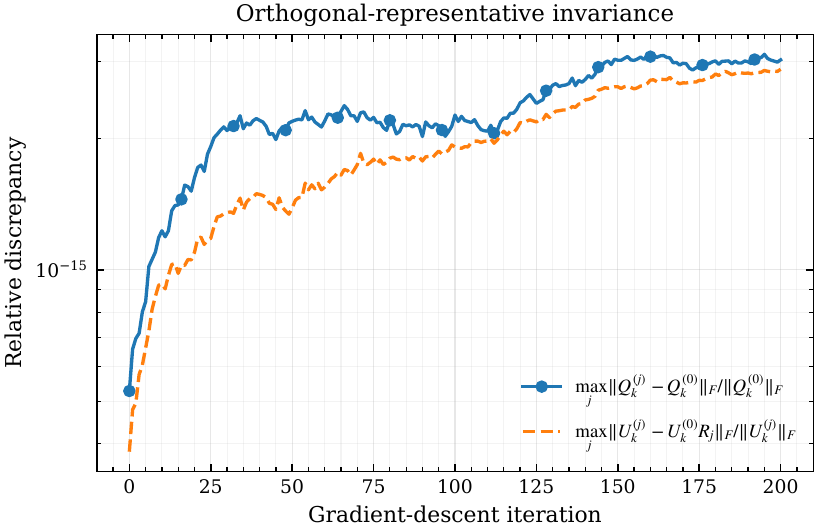}
        \caption{Predictor discrepancy across orthogonal representatives.}
        \label{fig:exp1-orthogonal-invariance}
    \end{subfigure}
    \hfill
    \begin{subfigure}[t]{0.32\textwidth}
        \centering
        \includegraphics[
            width=\linewidth
        ]{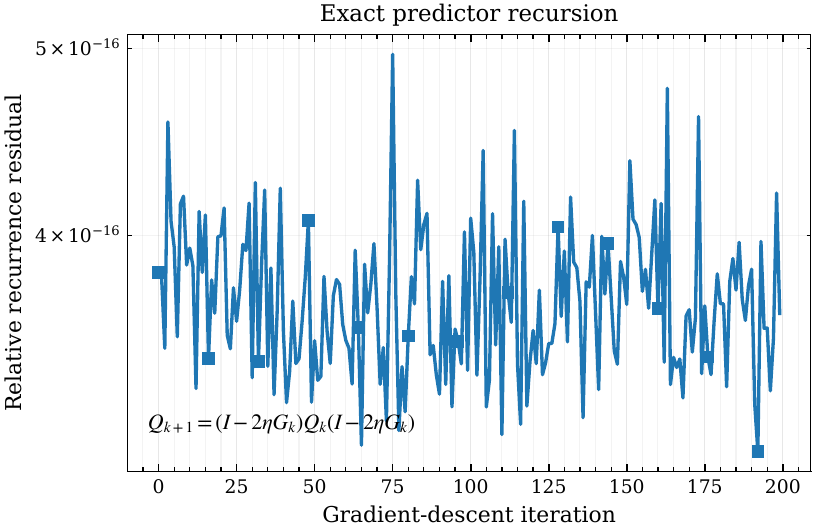}
        \caption{Residual of the exact predictor recursion.}
        \label{fig:exp1-exact-recurrence}
    \end{subfigure}
    \hfill
    \begin{subfigure}[t]{0.32\textwidth}
        \centering
        \includegraphics[
            width=\linewidth
        ]{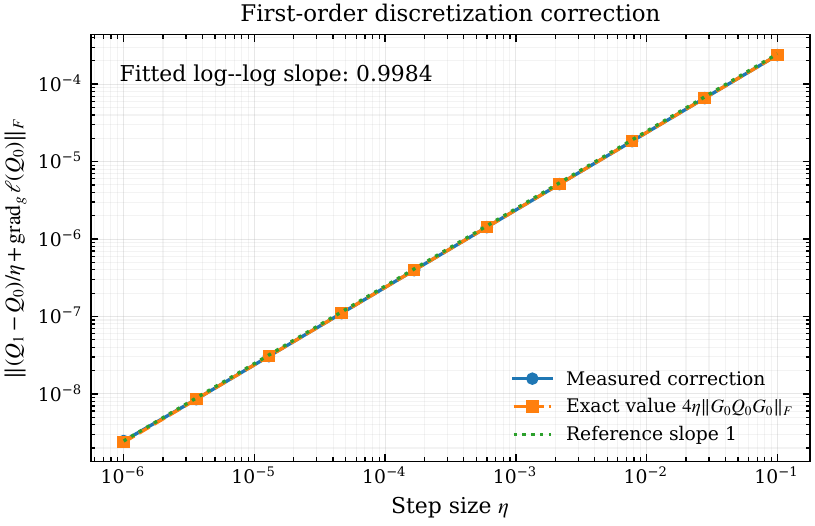}
        \caption{Linear scaling of the finite-step correction.}
        \label{fig:exp1-discretization}
    \end{subfigure}
    \caption{
        Exact quotient and predictor identities.
        Orthogonally equivalent factor initializations generate identical
        predictor trajectories up to floating-point roundoff.  The exact
        congruence recursion is also satisfied to numerical precision, and
        the discrepancy from the continuous quotient-gradient direction
        scales as \(O(\eta)\).
    }
    \label{fig:experiment1-quotient-dynamics}
\end{figure}

These results primarily serve as implementation checks for exact
algebraic identities.  They verify that the numerical dynamics used in
the remaining experiments are consistent with
\cref{thm:exact-quotient-gradient-flow,prop:exact-discrete-predictor-dynamics}.

\subsection{Effective curvature and the local training rate}
\label{subsec:experiment-effective-curvature}

We next test whether the generalized Hessian associated with the
quotient metric predicts the local rate of the actual factor gradient
flow.  We consider Gaussian rank-one measurements
\[
    x_i\sim N(0,I_d),
    \qquad
    y_i=x_i^\top Q_\star x_i,
\]
with
\[
    d=8,
    \qquad
    r=2,
    \qquad
    n=800.
\]
The largest nonzero eigenvalue of \(Q_\star\) is fixed, while the
smallest nonzero eigenvalue is varied over
\[
    \lambda_r(Q_\star)
    \in
    \{1,0.7,0.5,0.35,0.25\}.
\]

At \(U_\star\), we construct a Frobenius-orthonormal basis of the
horizontal space and solve the generalized eigenvalue problem
associated with the bilinear form in
\cref{thm:effective-hessian}.  The smallest quotient eigenvalue is
denoted by
\[
    \lambda_{\min}^{\mathrm{eff}}.
\]
The initial factor is then perturbed by a small displacement in the
corresponding slow eigendirection.  We estimate the observed local rate
from the Procrustes error by fitting
\begin{equation}
    \log d_{\mathrm P}(U(t),U_\star)
    =
    c-\widehat\lambda_{\mathrm{flow}}t
    \label{eq:experimental-local-rate-fit}
\end{equation}
over the local linear regime.

\begin{table}[H]
    \centering
    \caption{
        Effective curvature and measured local gradient-flow rate.
        The ratios are uniformly within \(2.4\times10^{-4}\) of one.
    }
    \label{tab:effective-curvature-rates}
    \begin{tabular}{cccc}
        \toprule
        \(\lambda_r(Q_\star)\)
        &
        \(\lambda_{\min}^{\mathrm{eff}}\)
        &
        \(\widehat\lambda_{\mathrm{flow}}\)
        &
        \(\widehat\lambda_{\mathrm{flow}}/
        \lambda_{\min}^{\mathrm{eff}}\)
        \\
        \midrule
        \(1.00\) & \(3.008130\) & \(3.008125\) & \(0.999998\) \\
        \(0.70\) & \(2.301692\) & \(2.301375\) & \(0.999862\) \\
        \(0.50\) & \(1.658780\) & \(1.659065\) & \(1.000172\) \\
        \(0.35\) & \(1.165631\) & \(1.165867\) & \(1.000203\) \\
        \(0.25\) & \(0.834176\) & \(0.834372\) & \(1.000234\) \\
        \bottomrule
    \end{tabular}
\end{table}

As shown in \cref{tab:effective-curvature-rates}, the measured rates
closely match the smallest effective eigenvalues:
\[
    \frac{
        \widehat\lambda_{\mathrm{flow}}
    }{
        \lambda_{\min}^{\mathrm{eff}}
    }
    \in
    [0.999862,1.000234].
\]
All exponential regressions have coefficients of determination
numerically indistinguishable from one.  Moreover, after rescaling time
by \(\lambda_{\min}^{\mathrm{eff}}\), the local trajectories nearly
collapse onto a common exponential profile.

\begin{figure}[H]
    \centering
    \begin{subfigure}[t]{0.32\textwidth}
        \centering
        \includegraphics[
            width=\linewidth
        ]{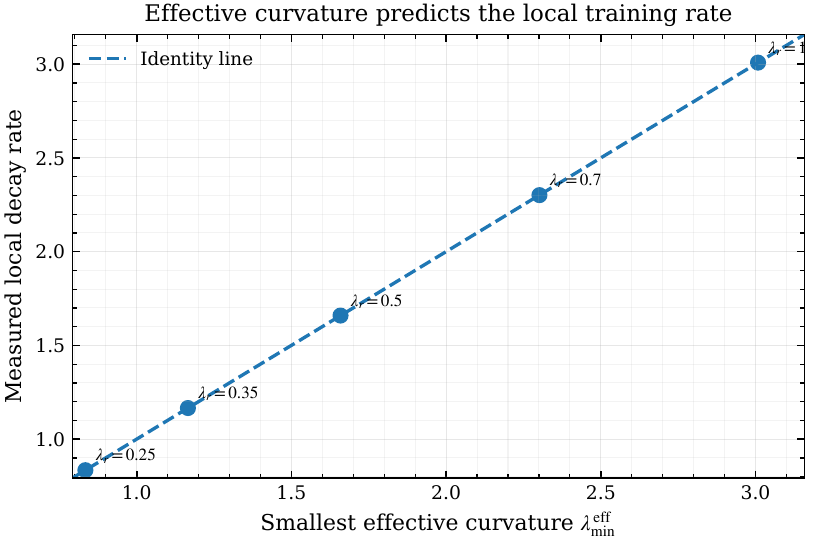}
        \caption{Measured rate versus effective curvature.}
        \label{fig:exp2-rate-vs-curvature}
    \end{subfigure}
    \hfill
    \begin{subfigure}[t]{0.32\textwidth}
        \centering
        \includegraphics[
            width=\linewidth
        ]{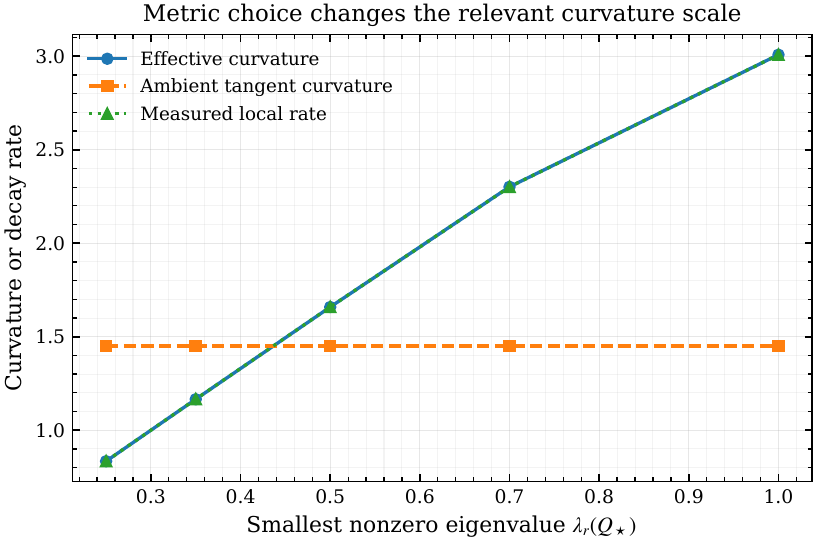}
        \caption{Comparison of curvature scales induced by different metrics.}
        \label{fig:exp2-metric-comparison}
    \end{subfigure}
    \hfill
    \begin{subfigure}[t]{0.32\textwidth}
        \centering
        \includegraphics[
            width=\linewidth
        ]{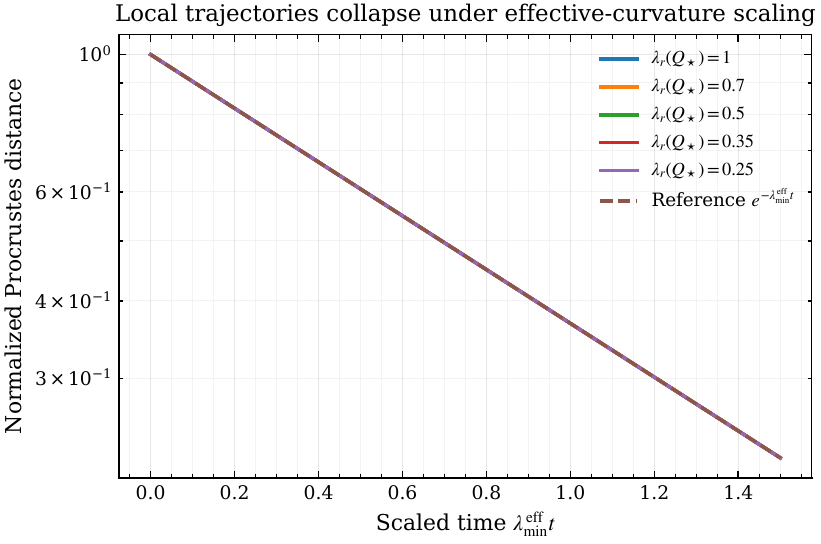}
        \caption{Local trajectories after curvature rescaling.}
        \label{fig:exp2-scaled-trajectories}
    \end{subfigure}
    \caption{
        Effective curvature and local optimization.
        The generalized quotient Hessian predicts the slow local
        convergence rate of the actual factor gradient flow.  The comparison
        is local: the perturbations are small and aligned with the slow
        effective-curvature eigendirection.
    }
    \label{fig:experiment2-effective-curvature}
\end{figure}

This experiment gives direct numerical support to the local
interpretation of \cref{thm:effective-hessian}.  It does not assert that
one Hessian eigenvalue determines the complete nonlinear trajectory
from an arbitrary initialization; rather, it confirms that the
effective Hessian gives the correct asymptotic scale in the local slow
mode.

\subsection{Spectral initialization and Gaussian recovery}
\label{subsec:experiment-spectral-initialization}

We next examine the end-to-end recovery mechanism of Section~3.  The
moment estimator is
\begin{equation}
    M_n
    =
    \frac{1}{2n}
    \sum_{i=1}^n
    y_i(x_ix_i^\top-I_d),
    \label{eq:experimental-moment-estimator}
\end{equation}
and the initializer is obtained by rank-\(r\) PSD truncation:
\[
    Q_0=\mathcal P_{r,+}(M_n),
    \qquad
    Q_0=U_0U_0^\top.
\]
We use
\[
    d=10,
    \qquad
    r=2,
    \qquad
    p
    =
    dr-\frac{r(r-1)}{2}
    =
    19,
\]
and perform \(16\) independent trials at each sample ratio
\[
    \frac{n}{p}
    \in
    \{2,4,8,16,32\}.
\]
The target satisfies
\[
    \lambda_r(Q_\star)=0.5.
\]

The theoretical basin radius used in
\cref{thm:initialization-enters-basin} is
\begin{equation}
    \rho_n
    =
    \frac{
        \sqrt{\lambda_r(Q_\star)}
    }{
        4(d+3)
    }.
    \label{eq:experimental-basin-radius}
\end{equation}
For each trial, we record whether the initializer is full rank, whether
\[
    d_{\mathrm P}(U_0,U_\star)\le \rho_n,
\]
and whether the final predictor satisfies the prescribed recovery
tolerance.

\begin{table}[H]
    \centering
    \caption{
        Spectral initialization and recovery over \(16\) independent trials.
        The theoretical basin is not reached at any tested sample size,
        although empirical recovery becomes reliable at substantially
        smaller \(n\).
    }
    \label{tab:spectral-initialization-recovery}
    \begin{tabular}{ccccc}
        \toprule
        \(n/p\)
        &
        \(n\)
        &
        Full-rank initialization
        &
        Basin hit rate
        &
        Recovery rate
        \\
        \midrule
        \(2\)  & \(38\)  & \(100\%\) & \(0\%\) & \(12.5\%\) \\
        \(4\)  & \(76\)  & \(100\%\) & \(0\%\) & \(93.75\%\) \\
        \(8\)  & \(152\) & \(100\%\) & \(0\%\) & \(100\%\) \\
        \(16\) & \(304\) & \(100\%\) & \(0\%\) & \(100\%\) \\
        \(32\) & \(608\) & \(100\%\) & \(0\%\) & \(100\%\) \\
        \bottomrule
    \end{tabular}
\end{table}

The moment-estimator error decreases with the sample size, and the
empirical recovery probability rises rapidly from \(12.5\%\) at
\(n/p=2\) to \(100\%\) at \(n/p\geq8\).  However, the theoretical basin
hit rate remains zero throughout the tested range.  The explicit
sufficient sample requirement obtained from the current constants in
\cref{thm:end-to-end-gaussian-recovery} is approximately
\[
    N_\star
    \simeq
    2.01\times10^{16},
\]
whereas the largest experimental sample size is \(608\).

\begin{figure}[H]
    \centering
    \begin{subfigure}[t]{0.49\textwidth}
        \centering
        \includegraphics[
            width=\linewidth
        ]{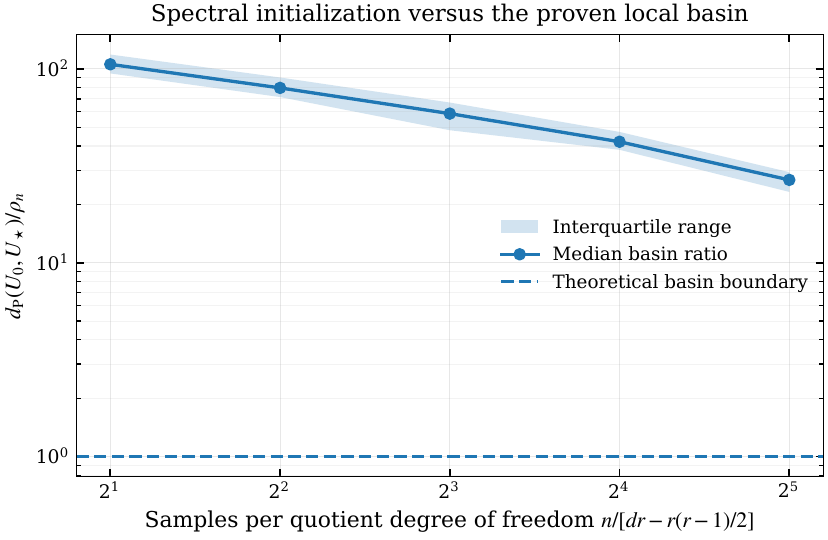}
        \caption{
            Ratio \(d_{\mathrm P}(U_0,U_\star)/\rho_n\).
        }
        \label{fig:exp3-basin-ratio}
    \end{subfigure}
    \hfill
    \begin{subfigure}[t]{0.49\textwidth}
        \centering
        \includegraphics[
            width=\linewidth
        ]{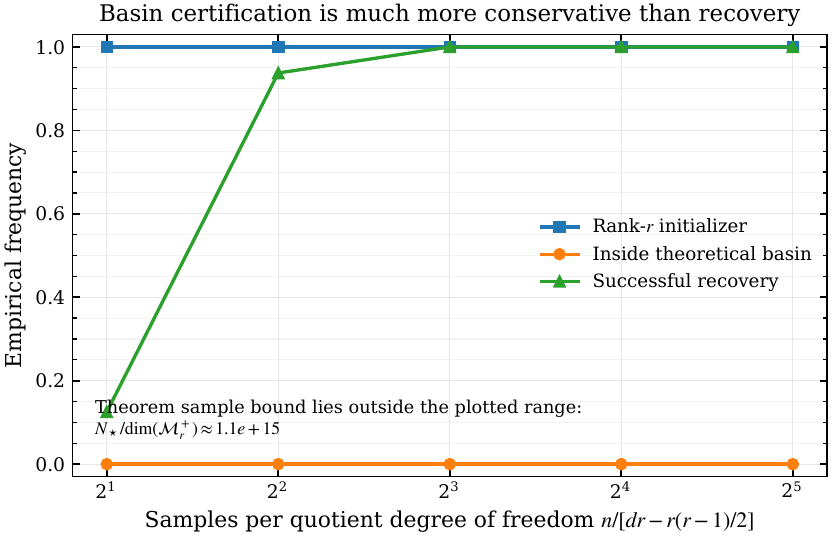}
        \caption{
            Full-rank, theoretical-basin, and empirical-recovery rates.
        }
        \label{fig:exp3-recovery-rates}
    \end{subfigure}
    \caption{
        Spectral initialization and Gaussian recovery.
        Recovery becomes reliable far below the explicit sample size required
        by the current sufficient condition.  Thus the experiment supports
        the qualitative initialization-and-recovery mechanism while exposing
        the conservativeness of the quantitative basin estimate.
    }
    \label{fig:experiment3-spectral-initialization}
\end{figure}

This result does not contradict
\cref{thm:initialization-enters-basin,thm:end-to-end-gaussian-recovery}, because those results provide
sufficient rather than necessary conditions.  At the same time, the
experiment is not a direct numerical verification of the stated basin
condition: every tested instance lies outside the region certified by
the theorem.  The appropriate conclusion is that the moment initializer
and subsequent factor optimization work in a substantially larger
empirical region than is captured by the current finite-sample
constants.

\subsection{Conditioning, linear convergence, and oracle step sizes}
\label{subsec:experiment-conditioning-step-size}

To isolate the optimization geometry from empirical-operator
fluctuations, we next use the exact Gaussian population operator
\[
    \mathcal T(H)
    =
    2H+\operatorname{tr}(H)I_d.
\]
We set
\[
    d=8,
    \qquad
    r=2,
\]
and vary the condition number over
\[
    \kappa(Q_\star)
    \in
    \{1,2,4,8\}.
\]
The largest eigenvalue is fixed at one, so that
\[
    \lambda_r(Q_\star)
    =
    \kappa(Q_\star)^{-1}.
\]

\begin{table}[H]
    \centering
    \caption{
        Conditioning, effective curvature, and local gradient-flow rate for
        the exact population objective.
    }
    \label{tab:conditioning-flow-rates}
    \begin{tabular}{ccccc}
        \toprule
        \(\kappa(Q_\star)\)
        &
        \(\lambda_r(Q_\star)\)
        &
        \(\lambda_{\min}^{\mathrm{eff}}\)
        &
        \(\widehat\lambda_{\mathrm{flow}}\)
        &
        Ratio
        \\
        \midrule
        \(1\) & \(1\)     & \(4.0\) & \(3.99991\) & \(0.999978\) \\
        \(2\) & \(0.5\)   & \(2.0\) & \(1.99995\) & \(0.999977\) \\
        \(4\) & \(0.25\)  & \(1.0\) & \(0.99998\) & \(0.999978\) \\
        \(8\) & \(0.125\) & \(0.5\) & \(0.49999\) & \(0.999978\) \\
        \bottomrule
    \end{tabular}
\end{table}

The effective curvature decreases linearly with
\(\lambda_r(Q_\star)\), and the measured gradient-flow rate agrees with
it to approximately five significant digits.  In this population
example,
\[
    \lambda_{\min}^{\mathrm{eff}}
    =
    4\lambda_r(Q_\star).
\]
By comparison, the secant argument in
\cref{thm:local-secant-regularity,thm:local-gradient-flow-convergence}
provides the lower rate
\[
    \alpha_\star
    =
    \lambda_r(Q_\star).
\]
Thus the theorem correctly captures the dependence on the smallest
signal eigenvalue, but its explicit rate constant is conservative by a
factor of approximately four in this example.

For gradient descent, we normalize the step size by the oracle
sufficient bound appearing in
\cref{thm:local-gradient-descent-convergence}.  The corresponding
oracle values for
\(\kappa(Q_\star)=1,2,4,8\) are approximately
\[
    5.40\times10^{-4},
    \qquad
    2.81\times10^{-4},
    \qquad
    1.45\times10^{-4},
    \qquad
    7.41\times10^{-5}.
\]
We sweep the multipliers
\[
    \frac{\eta}{\eta_{\mathrm{oracle}}}
    \in
    \{1,10,100,500,1000,2000,5000,10000\}.
\]
Step sizes far above the oracle bound can remain stable, but sufficiently
large multipliers produce oscillations and loss of accuracy.  For the
most poorly conditioned target,
\[
    \kappa(Q_\star)=8,
    \qquad
    \eta=10000\,\eta_{\mathrm{oracle}},
\]
the iteration diverges numerically.

\begin{figure}[H]
    \centering
    \begin{subfigure}[t]{0.32\textwidth}
        \centering
        \includegraphics[
            width=\linewidth
        ]{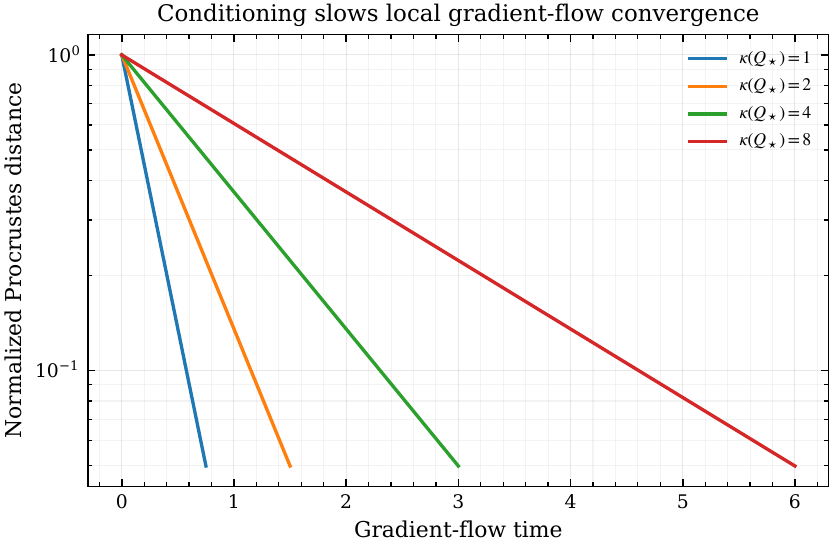}
        \caption{Gradient-flow decay for different condition numbers.}
        \label{fig:exp4-conditioning-flow}
    \end{subfigure}
    \hfill
    \begin{subfigure}[t]{0.32\textwidth}
        \centering
        \includegraphics[
            width=\linewidth
        ]{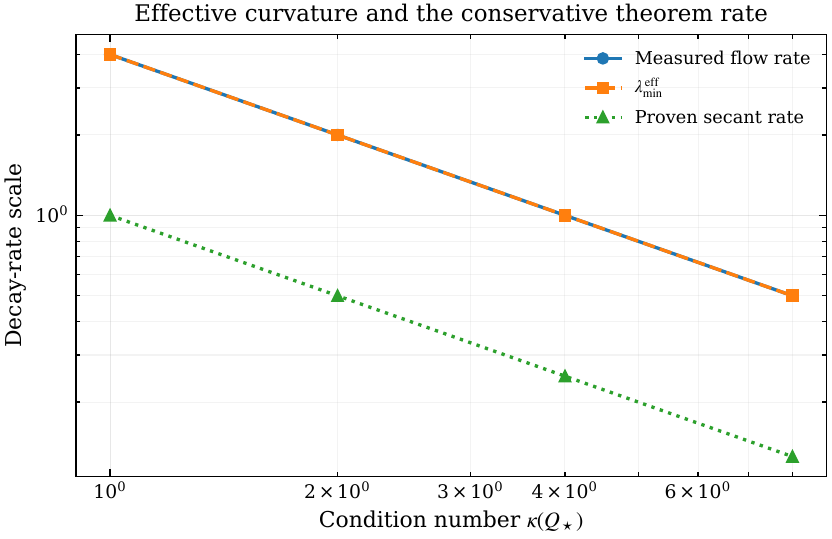}
        \caption{Effective, measured, and secant-guaranteed rates.}
        \label{fig:exp4-rate-comparison}
    \end{subfigure}
    \hfill
    \begin{subfigure}[t]{0.32\textwidth}
        \centering
        \includegraphics[
            width=\linewidth
        ]{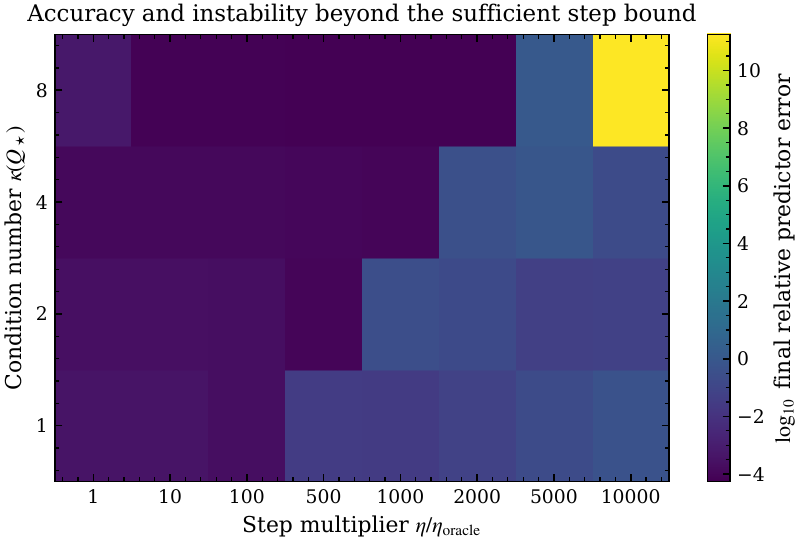}
        \caption{Empirical stability over condition numbers and step sizes.}
        \label{fig:exp4-stability-map}
    \end{subfigure}
    \caption{
        Conditioning and discrete stability.
        Poorer conditioning decreases the effective curvature and slows the
        local flow.  The oracle step-size bound identifies a safe regime but
        is substantially smaller than the empirical stability threshold.
    }
    \label{fig:experiment4-conditioning}
\end{figure}

Accordingly, the oracle step size should be interpreted as a sufficient
guarantee rather than an accurate characterization of the stability
boundary.  The experiment supports the qualitative dependence of the
rate and stable step size on the signal spectrum, while showing that
the explicit constants can be sharpened.

\subsection{Bregman selection under commuting measurements}
\label{subsec:experiment-bregman-selection}

We now turn to the implicit-bias results of Section~4.  In the invariant
joint-spectral coordinates, the continuous factor gradient flow is
\begin{equation}
    \dot q_a
    =
    -
    \frac{4q_a}{d_a}
    \left[
        \frac{1}{n}
        B^\top(Bq-y)
    \right]_a,
    \qquad
    q_a(0)=\varepsilon^2.
    \label{eq:experimental-reduced-gradient-flow}
\end{equation}
The flow is integrated in logarithmic coordinates
\[
    z_a=\log q_a,
\]
which preserves strict positivity numerically.

For each
\[
    \varepsilon
    \in
    \{0.5,0.25,0.1,0.05\},
\]
we independently solve the convex problem
\begin{equation}
    q_\varepsilon
    =
    \arg\min_{\substack{q\geq0\\Bq=y}}
    D_h(q,\varepsilon^2\mathbf 1)
    \label{eq:experimental-bregman-projection}
\end{equation}
using the dual optimality equations.  We then compare the optimizer
with the terminal point of the continuous flow.

To test the effect of redundant measurements, we also augment the
measurement system by linearly dependent equations.  This modification
leaves the feasible set unchanged but changes the least-squares
objective and therefore the transient vector field.

\begin{table}[H]
    \centering
    \caption{
        Agreement between the continuous flow and the independently computed
        Bregman projection.  ``Base'' and ``augmented'' refer to measurement
        systems with the same feasible set.
    }
    \label{tab:bregman-selection-results}
    \resizebox{\textwidth}{!}{
    \begin{tabular}{ccccc}
        \toprule
        \(\varepsilon\)
        &
        Difference between projections
        &
        Base flow to projection
        &
        Augmented flow to projection
        &
        Difference between flow limits
        \\
        \midrule
        \(0.50\)
        &
        \(1.61\times10^{-15}\)
        &
        \(1.42\times10^{-10}\)
        &
        \(2.19\times10^{-10}\)
        &
        \(2.61\times10^{-10}\)
        \\
        \(0.25\)
        &
        \(2.52\times10^{-15}\)
        &
        \(1.57\times10^{-10}\)
        &
        \(2.33\times10^{-10}\)
        &
        \(2.72\times10^{-10}\)
        \\
        \(0.10\)
        &
        \(2.15\times10^{-15}\)
        &
        \(1.63\times10^{-10}\)
        &
        \(2.41\times10^{-10}\)
        &
        \(2.89\times10^{-10}\)
        \\
        \(0.05\)
        &
        \(2.76\times10^{-15}\)
        &
        \(1.62\times10^{-10}\)
        &
        \(2.41\times10^{-10}\)
        &
        \(2.93\times10^{-10}\)
        \\
        \bottomrule
    \end{tabular}
    }
\end{table}

The two independently computed Bregman projections agree to
approximately \(10^{-15}\), and both gradient flows converge to that
projection with errors of order \(10^{-10}\).  The redundant
measurements nevertheless change the transient path and time scale.
For example, at \(\varepsilon=0.1\), the stopping times are
approximately
\[
    t_{\mathrm{base}}
    =
    29.72,
    \qquad
    t_{\mathrm{augmented}}
    =
    51.59.
\]

We also evaluate the Lyapunov identity associated with
\cref{thm:global-convergence-bregman-projection}:
\begin{equation}
    D_h(q^\dagger,q(t))
    =
    D_h(q^\dagger,q(0))
    -
    \int_0^t
    \frac{
        \|Bq(s)-y\|_2^2
    }{
        n
    }
    \,ds.
    \label{eq:experimental-bregman-lyapunov}
\end{equation}
The maximum discrepancy between the two sides of
\cref{eq:experimental-bregman-lyapunov} is approximately \(10^{-6}\).

\begin{figure}[H]
    \centering
    \begin{subfigure}[t]{0.32\textwidth}
        \centering
        \includegraphics[
            width=\linewidth
        ]{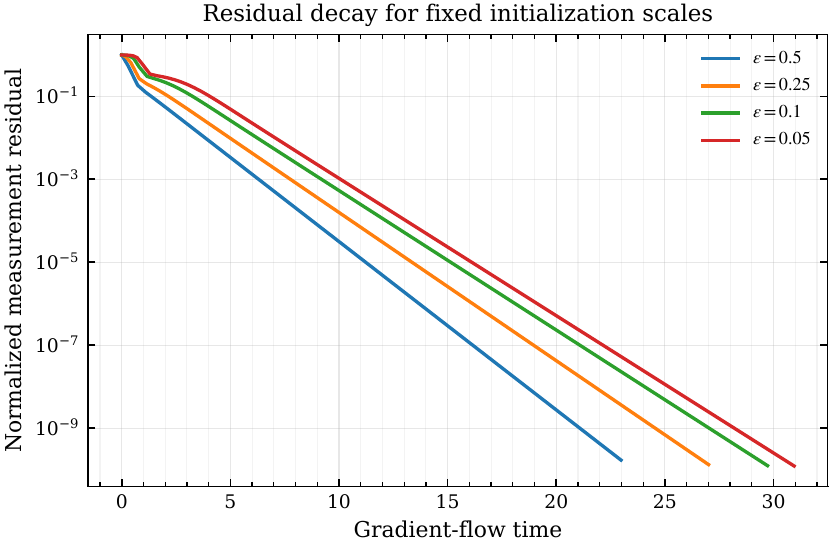}
        \caption{Residual decay for several initialization scales.}
        \label{fig:exp5-residual-decay}
    \end{subfigure}
    \hfill
    \begin{subfigure}[t]{0.32\textwidth}
        \centering
        \includegraphics[
            width=\linewidth
        ]{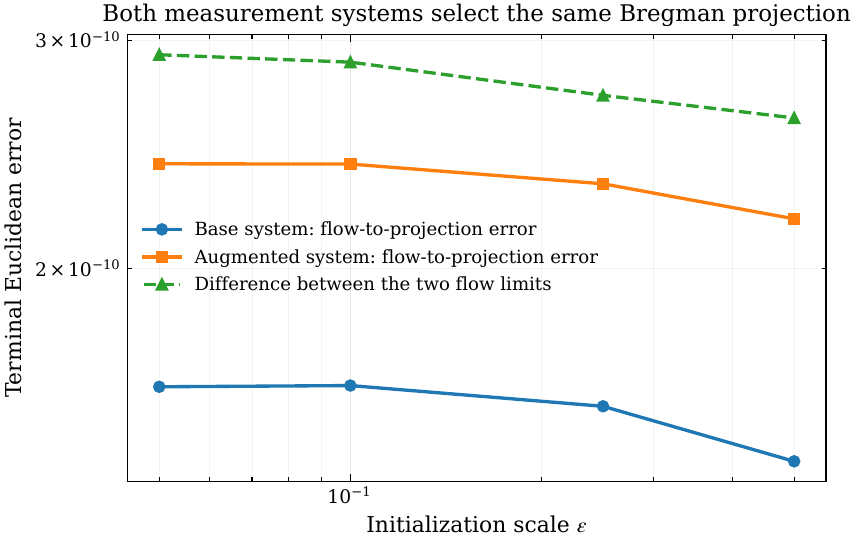}
        \caption{Agreement of flow limits and Bregman projections.}
        \label{fig:exp5-projection-agreement}
    \end{subfigure}
    \hfill
    \begin{subfigure}[t]{0.32\textwidth}
        \centering
        \includegraphics[
            width=\linewidth
        ]{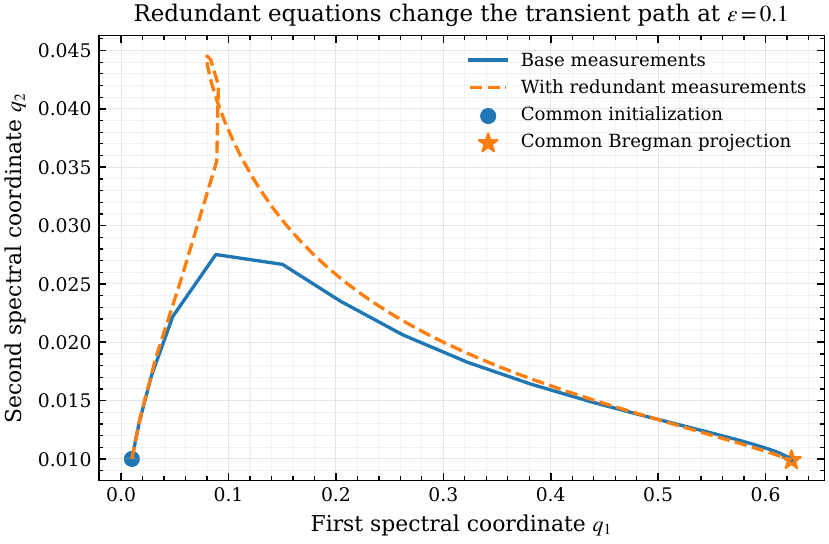}
        \caption{Different transient paths with the same selected endpoint.}
        \label{fig:exp5-redundant-paths}
    \end{subfigure}
    \caption{
        Fixed-initialization Bregman selection.
        The flow converges to the independently computed Bregman projection.
        Redundant measurements preserve the feasible set and selected point
        but alter the transient dynamics.
    }
    \label{fig:experiment5-bregman-selection}
\end{figure}

These results give strong numerical support to
\cref{thm:global-convergence-bregman-projection} in the strictly
feasible commuting regime.  They also illustrate why redundant
measurements cannot simply be discarded when discussing convergence
time: the selected feasible point depends only on the constraint set,
whereas the transient dynamics depend on the particular least-squares
representation.

\subsection{Minimum trace, entropic tie-breaking, and finite-step error}
\label{subsec:experiment-minimum-trace-entropy}

The final experiment tests the small-initialization and finite-step
selection results.  We use four joint spectral blocks with
multiplicities
\[
    d=(1,1,2,3),
\]
and the reduced measurement system
\begin{equation}
    B
    =
    \begin{pmatrix}
        1 & 1 & 1 & 0\\
        0 & 0 & 1 & 1
    \end{pmatrix},
    \qquad
    y
    =
    \begin{pmatrix}
        1\\
        0.3
    \end{pmatrix}.
    \label{eq:experiment6-measurement-system}
\end{equation}
Writing \(c=0.3\), every feasible point can be parameterized as
\[
    q_4=t,
    \qquad
    q_3=c-t,
    \qquad
    q_1+q_2=1-c+t.
\]
The weighted trace is
\begin{equation}
    \tau(q)
    =
    \sum_{a=1}^4d_aq_a
    =
    1+c+2t.
    \label{eq:experiment6-weighted-trace}
\end{equation}
Consequently, the minimum-trace face is
\begin{equation}
    \mathcal S_{\mathrm{tr}}
    =
    \left\{
        q\geq0:
        q_4=0,\;
        q_3=0.3,\;
        q_1+q_2=0.7
    \right\},
    \label{eq:experiment6-minimum-trace-face}
\end{equation}
and
\[
    \tau_\star
    =
    1.3.
\]
This face is non-singleton.  The unique weighted entropic minimizer on
it is
\begin{equation}
    q_{\mathrm{ent}}
    =
    (0.35,0.35,0.3,0).
    \label{eq:experiment6-entropic-point}
\end{equation}

The minimum-trace value is certified globally over the full PSD cone,
not only over the reduced diagonal algebra.  Indeed, the dual vector
\[
    \lambda=(1,1)
\]
satisfies
\[
    d-B^\top\lambda
    =
    (0,0,0,2)
    \geq0,
    \qquad
    y^\top\lambda
    =
    1.3.
\]
For the corresponding full semidefinite program, the joint block
eigenvalues of
\(\sum_i\lambda_iA_i\) are
\[
    1,
    \qquad
    1,
    \qquad
    1,
    \qquad
    \frac13,
\]
and hence
\[
    \sum_i\lambda_iA_i
    \preceq
    I.
\]
The primal value \(1.3\) and dual value \(1.3\) coincide.

\paragraph{Small-initialization limit.}

For isotropic initialization
\[
    q(0)=\varepsilon^2\mathbf1,
\]
we compute the exact Bregman projection for twelve values of
\(\varepsilon\) between \(0.7\) and \(0.005\).  The maximum feasibility
residual is
\[
    1.11\times10^{-16}.
\]
At the smallest initialization,
\[
    \varepsilon=0.005,
\]
we obtain
\[
    \tau(q_\varepsilon)-\tau_\star
    =
    1.086\times10^{-3}.
\]

The explicit bound from \cref{thm:minimum-trace-limit} has the form
\begin{equation}
    \tau(q_\varepsilon)-\tau_\star
    \leq
    \frac{
        C
    }{
        \log(1/\varepsilon^2)
    },
    \qquad
    C=4.243.
    \label{eq:experiment6-trace-envelope}
\end{equation}
Across all tested initialization scales,
\[
    \max_\varepsilon
    \log(1/\varepsilon^2)
    \bigl(
        \tau(q_\varepsilon)-\tau_\star
    \bigr)
    =
    0.4625
    <
    C.
\]
Thus every observed trace gap lies below the theoretical envelope.
The experiment does not claim that the
\(1/\log(1/\varepsilon^2)\) upper bound is asymptotically tight; in this
particular example the convergence may be faster.

\paragraph{Entropic tie-breaking.}

The distance to the entropic minimizer decreases monotonically with
\(\varepsilon\).  At \(\varepsilon=0.005\),
\[
    \|q_\varepsilon-q_{\mathrm{ent}}\|_2
    =
    8.58\times10^{-4}.
\]
A descriptive log--log regression over the smallest initialization
scales gives slope
\[
    1.306,
    \qquad
    R^2=0.99987.
\]
This fitted exponent is not asserted by
\cref{thm:entropic-tie-breaking}; the theoretically relevant
observation is the convergence to the unique entropy-selected point
inside the non-singleton minimum-trace face.

\paragraph{Finite-step selection error.}

Finally, we fix
\[
    \varepsilon=0.2
\]
and run the exact parameter gradient-descent recursion for
\[
    \eta
    \in
    \left\{
        \frac12,
        \frac14,
        \frac18,
        \ldots,
        \frac1{256}
    \right\}.
\]
All eight step sizes converge to strictly positive interpolating
limits.  The finite-step selection error is measured by
\[
    E_{\mathrm{step}}(\eta)
    :=
    \|q_{\varepsilon,\eta}-q_\varepsilon\|_2.
\]
The log--log fit gives
\begin{equation}
    E_{\mathrm{step}}(\eta)
    \propto
    \eta^{0.99094},
    \qquad
    R^2=0.999981.
    \label{eq:experiment6-discrete-fit}
\end{equation}
This agrees with the first-order estimate of
\cref{thm:first-order-discrete-selection-error},
\[
    \|q_{\varepsilon,\eta}-q_\varepsilon\|_2
    =
    O(\eta)
\]
for fixed \(\varepsilon\).

\begin{figure}[H]
    \centering
    \begin{subfigure}[t]{0.32\textwidth}
        \centering
        \includegraphics[
            width=\linewidth
        ]{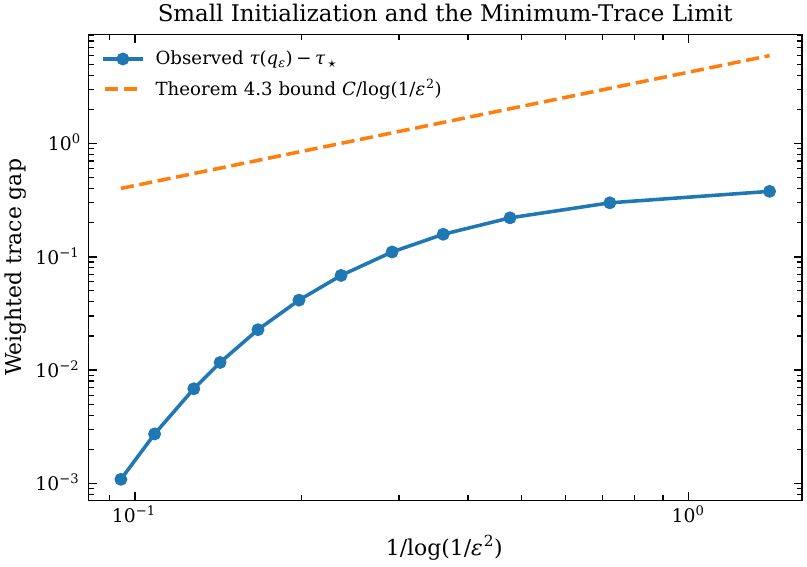}
        \caption{Trace gap and the theoretical upper envelope.}
        \label{fig:exp6-trace-gap}
    \end{subfigure}
    \hfill
    \begin{subfigure}[t]{0.32\textwidth}
        \centering
        \includegraphics[
            width=\linewidth
        ]{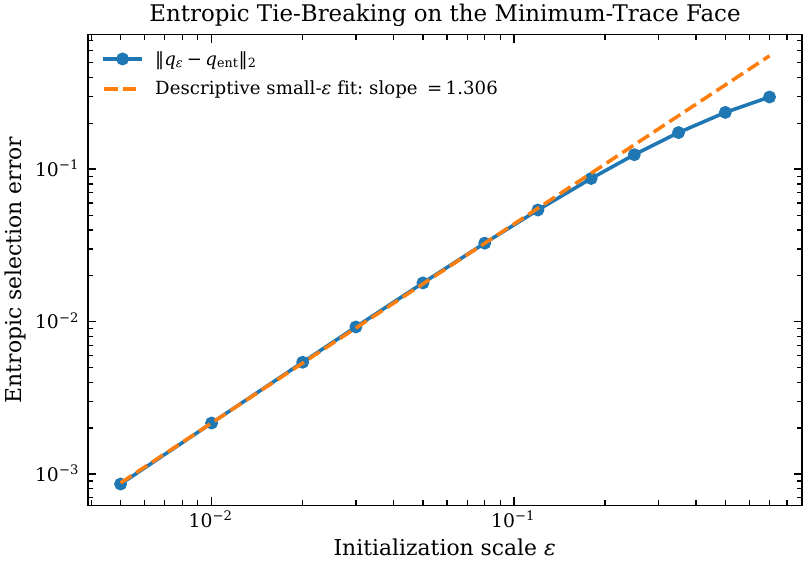}
        \caption{Convergence to the entropic minimizer.}
        \label{fig:exp6-entropy}
    \end{subfigure}
    \hfill
    \begin{subfigure}[t]{0.32\textwidth}
        \centering
        \includegraphics[
            width=\linewidth
        ]{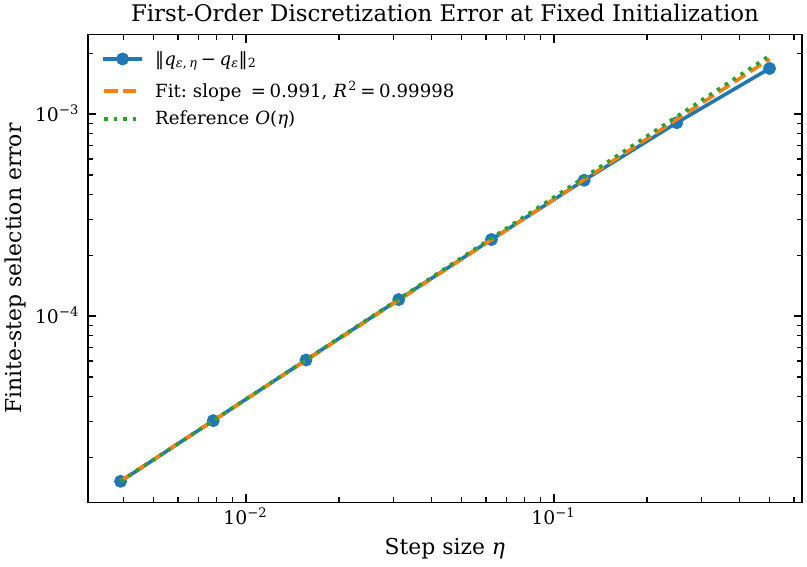}
        \caption{First-order finite-step selection error.}
        \label{fig:exp6-finite-step}
    \end{subfigure}
    \caption{
        Minimum-trace, entropy, and finite-step selection.
        The continuous Bregman projections approach the globally certified
        minimum-trace face, converge within that face to the unique weighted
        entropic minimizer, and differ from the finite-step limits by an
        empirically first-order error in \(\eta\).
    }
    \label{fig:experiment6-selection}
\end{figure}

\begin{figure}[H]
    \centering
    \begin{subfigure}[t]{0.49\textwidth}
        \centering
        \includegraphics[
            width=\linewidth
        ]{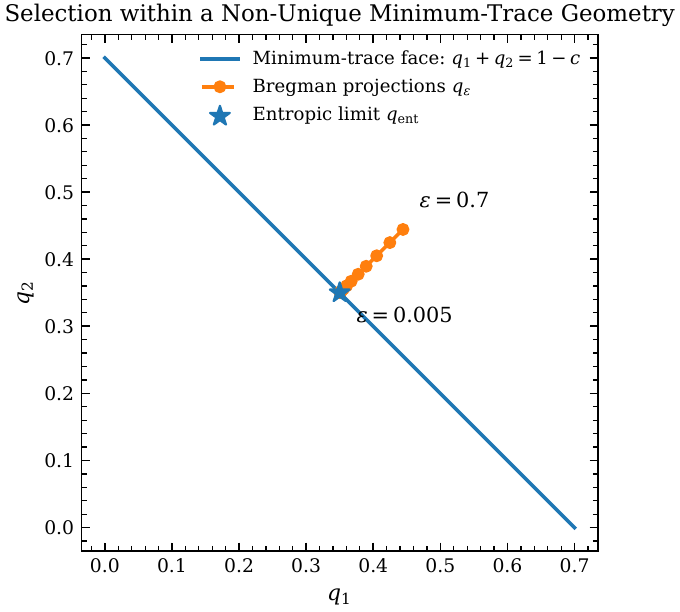}
        \caption{
            The non-singleton minimum-trace face and its entropy-selected
            point.
        }
        \label{fig:exp6-minimum-trace-face}
    \end{subfigure}
    \hfill
    \begin{subfigure}[t]{0.49\textwidth}
        \centering
        \includegraphics[
            width=\linewidth
        ]{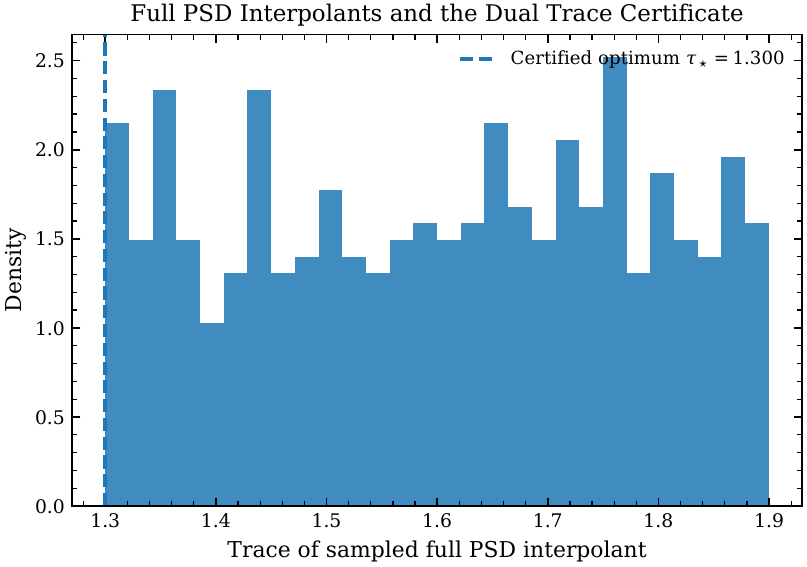}
        \caption{
            Traces of sampled full PSD interpolants and the certified optimum.
        }
        \label{fig:exp6-full-psd}
    \end{subfigure}
    \caption{
        Geometry and global certification of the minimum-trace example.
        Random full PSD interpolants have trace no smaller than the certified
        value \(\tau_\star=1.3\).  The certification itself follows from the
        explicit semidefinite dual certificate, rather than from random
        sampling.
    }
    \label{fig:experiment6-minimum-trace-geometry}
\end{figure}

The order of limits is important.  The experiment verifies the
fixed-\(\varepsilon\) convergence
\[
    q_{\varepsilon,\eta}
    \longrightarrow
    q_\varepsilon
    \qquad
    \text{as }
    \eta\downarrow0,
\]
followed by
\[
    q_\varepsilon
    \longrightarrow
    q_{\mathrm{ent}}
    \qquad
    \text{as }
    \varepsilon\downarrow0.
\]
No numerical or theoretical claim is made for a fixed positive
\(\eta\) followed directly by \(\varepsilon\downarrow0\), since the
constants in the finite-step estimate may deteriorate with the
initialization scale.

\subsection{Summary of the empirical evidence}
\label{subsec:experimental-summary}

The six experiments support the main mechanisms of the theory at
different levels of strength.

First, the orthogonal invariance and predictor-recursion experiments
verify exact identities to floating-point precision.  Second, the
effective Hessian predicts the local slow-mode convergence rate with
relative errors below \(2.4\times10^{-4}\).  Third, the spectral
initializer and factor dynamics recover the target at sample sizes far
below the current sufficient bound, demonstrating both the practical
relevance of the mechanism and the conservativeness of the explicit
finite-sample constants.  Fourth, the condition-number experiments
confirm the dependence of the local rate on the smallest nonzero signal
eigenvalue, while showing that both the secant rate and the oracle
step-size bound are conservative.  Finally, the commuting experiments
show that the gradient-flow limit agrees with the exact Bregman
projection, that small initialization approaches a globally
minimum-trace PSD interpolant, that nonuniqueness is resolved by a
weighted entropic criterion inside the invariant spectral algebra, and
that the finite-step selection error is first order in the step size.

Accordingly, the experiments are consistent with all qualitative
theoretical mechanisms established in Sections~\ref{sec:quotient-dynamics}--\ref{sec:commuting-trace-bias}.  They should not,
however, be interpreted as showing that every explicit bound is tight:
the basin radius, end-to-end sample complexity, secant rate, and oracle
step-size threshold are all substantially more conservative than their
empirical counterparts.

\section{Conclusion}
\label{sec:conclusion}

We studied positive quadratic networks through the exact
factorization
\[
f_U(x)
=
x^\top UU^\top x,
\]
with the aim of connecting parameter redundancy to concrete statements
about training dynamics, intrinsic curvature, recovery, and implicit
selection. The central structural fact is that, on the
full-column-rank stratum, the fiber of the map
\[
U\longmapsto UU^\top
\]
is exactly the orthogonal orbit
\[
\{UR:R\in O(r)\}.
\]
The corresponding smooth predictor space is therefore the quotient
\[
\mathbb{R}^{d\times r}_*/O(r),
\]
which is diffeomorphic to the manifold
$\mathcal{M}_r^+$ of rank-$r$ positive semidefinite matrices.
Rank-deficient factors lie on the singular boundary of this stratum
and require separate treatment.

This quotient description is directly tied to the algorithm used to
train the model. For every smooth objective of the form
\[
L(U)
=
\ell(UU^\top),
\]
the ordinary Euclidean factor gradient is horizontal with respect to
the quotient structure induced by the Frobenius metric. Consequently,
Euclidean factor gradient flow projects exactly to
\[
\dot Q
=
-2\bigl(
\nabla\ell(Q)Q
+
Q\nabla\ell(Q)
\bigr).
\]
The resulting quotient flow is therefore not an auxiliary optimization
procedure imposed on predictor space; it is the exact predictor
evolution generated by Euclidean training in factor coordinates. At a
nonzero step size, factor gradient descent induces the exact congruence
recursion
\[
Q_{k+1}
=
\bigl(
I_d-2\eta\nabla\ell(Q_k)
\bigr)
Q_k
\bigl(
I_d-2\eta\nabla\ell(Q_k)
\bigr).
\]
Its quadratic correction explicitly records the discrepancy between
the discrete predictor update and the continuous quotient flow.

For quadratic regression, the quotient geometry yields an effective
notion of curvature that removes the redundancy of the factorization
without removing genuine predictor non-identifiability. At an
interpolating predictor $Q_\star$, the quotient Hessian satisfies
\[
\operatorname{Hess}_g
\ell_n(Q_\star)[H,H]
=
\frac1n
\|\mathcal{A}_n(H)\|_2^2,
\qquad
H\in T_{Q_\star}\mathcal{M}_r^+.
\]
Its nullspace is therefore
\[
T_{Q_\star}\mathcal{M}_r^+
\cap
\ker\mathcal{A}_n.
\]
Directions that remain degenerate after quotienting correspond to
sample-level non-identifiability rather than infinitesimal orthogonal
gauge motion. The generalized eigenvalues of this Hessian, defined
relative to the quotient metric, describe the local intrinsic modes of
the factor trajectory. In particular, the smallest effective
eigenvalue governs the slowest mode of the linearized quotient flow.

Under Gaussian rank-one measurements, we computed the population
normal operator, proved a uniform deviation bound for the empirical
operator on the full ambient space $\mathbb{S}^d$, constructed a
moment-based spectral initializer, and established a deterministic
local secant inequality. These ingredients yield local exponential
convergence of ordinary factor gradient flow and local linear
convergence of factor gradient descent under a sufficiently small
oracle step size.

The resulting end-to-end Gaussian recovery theorem is explicit but
deliberately conservative. It relies on full-space second-moment
control rather than a restricted concentration theorem on a tangent
space or low-rank secant set. Accordingly, its sufficient sample
requirement is not claimed to scale with the intrinsic dimension of
$\mathcal{M}_r^+$ or to be statistically optimal. The numerical
experiments reinforce this distinction: the smallest effective
eigenvalue accurately tracks the observed slow local training mode,
whereas successful initialization, recovery, and step-size stability
occur far beyond the regions certified by the present constants.

In the underdetermined commuting regime, the predictor dynamics admit
an additional variational interpretation. Joint spectral coordinates
reduce the factor gradient flow to an entropy mirror flow. For every
strictly positive initialization, the trajectory converges to the
Bregman projection determined by that initialization:
\[
q_\infty
=
\arg\min_{\substack{q\ge0,\,Bq=y}}
D_h(q,q(0)).
\]
Thus finite initialization produces an explicit, initialization-dependent interpolation principle rather than merely an unspecified
preference for low complexity.

For isotropic initialization
\[
q(0)
=
\varepsilon^2\mathbf{1},
\]
the dominant term of the Bregman objective is the weighted trace. As
$\varepsilon\downarrow0$, the selected predictors approach the
minimum-trace solution set. Under the commuting assumptions, this
minimum trace agrees with the minimum over the full positive
semidefinite interpolation set. If the minimum-trace solution is
nonunique, the lower-order entropy term selects a unique point on the
minimum-trace face within the invariant joint spectral algebra. The
resulting hierarchy is
\[
\text{finite positive initialization}
\Longrightarrow
\text{Bregman projection},
\]
\[
\text{isotropic small initialization}
\Longrightarrow
\text{minimum trace},
\]
\[
\text{nonunique invariant minimum-trace face}
\Longrightarrow
\text{weighted entropy tie-breaking}.
\]
For fixed positive initialization, finite-step factor gradient descent
selects an interpolant whose distance from the continuous Bregman
projection is $O(\eta)$. The constants in this comparison may depend
on the initialization scale, so the result first sends
$\eta\downarrow0$ with the initialization fixed.

The scope of these conclusions is intentionally precise. The smooth
quotient construction applies only to the full-column-rank stratum and
does not regularize singularities associated with rank loss. The
Gaussian recovery theorem is rank matched, local after a
data-dependent spectral initialization, based on a conservative
full-space deviation event, and uses an oracle learning-rate condition
for discrete gradient descent. The explicit Bregman and entropy
characterizations require a commuting measurement algebra. We do not
claim an entropy tie-breaking law for arbitrary noncommuting positive
semidefinite interpolants, a near-optimal Gaussian sample complexity,
or equality between the small-initialization limits of gradient flow
and gradient descent at a fixed nonzero step size.

Several directions remain open. A first challenge is to replace the
full-space Gaussian operator estimate by concentration on the fixed
tangent space or on a suitable low-rank secant neighborhood. Such a
result could substantially sharpen the recovery basin and sample
complexity while preserving the exact connection between effective
curvature and the actual factor dynamics. A second challenge is to
understand which parts of the selection theory survive beyond the
commuting setting. The exact quotient dynamics and effective-Hessian
construction remain available for general positive semidefinite
measurements, but the scalar mirror-flow reduction generally does not.
Developing an intrinsic variational description of these
noncommutative predictor dynamics would connect the local quotient
geometry established here with a broader theory of implicit
regularization in factorized matrix models and neural networks.


\acks{This research received no external funding.}


\newpage

{\noindent \em Remainder omitted in this sample. See http://www.jmlr.org/papers/ for full paper.}

\vskip 0.2in
\bibliography{sample}

\end{document}